\documentclass[twoside,11pt]{article}

\usepackage[abbrvbib, preprint]{jmlr2e}

\hypersetup{hidelinks}

\usepackage{imports}
\usepackage{commands}

\usepackage{lastpage}
\ShortHeadings{Consensus on Feature Attributions in the Rashomon Set}{Laberge, Pequignot, Mathieu, Khomh and Marchand}
\firstpageno{1}

\begin{document}

\title{Partial Order in Chaos: Consensus on Feature Attributions in the Rashomon Set.}

\author{
    \name $\!\!$Gabriel Laberge\normalfont{\textsuperscript{1}}\email gabriel.laberge@polymtl.ca \\
    \name Yann Pequignot\normalfont{\textsuperscript{2}} \email yann.pequignot@iid.ulaval.ca \\
    \name Alexandre Mathieu\normalfont{\textsuperscript{2}} \email alexandre.mathieu.7@ulaval.ca \\
    \name Foutse Khomh\normalfont{\textsuperscript{1}} \email foutse.khomh@polymtl.ca \\
    \name Mario Marchand\normalfont{\textsuperscript{2}} \email mario.marchand@ift.ulaval.ca \\
    \addr \normalfont{\textsuperscript{1}}Génie Informatique et Génie Logiciel, Polytechnique Montréal\\
    \normalfont{\textsuperscript{2}}Institut Intelligence et Données, Université de Laval à Québec\\
    \newline
    \textbf{Published at : \url{https://www.jmlr.org/papers/v24/23-0149.html}}
}

% \editor{Pradeep Ravikumar}

\maketitle

\begin{abstract}%
Post-hoc global/local feature attribution methods are progressively being employed to 
understand the decisions of complex machine learning models. Yet, because of limited amounts of data, 
it is possible to obtain a diversity of models with good empirical performance but that 
provide very different explanations for the same prediction, making it hard to derive insight from them.
In this work, instead of aiming at reducing the under-specification of model explanations, we fully 
embrace it and extract logical statements about feature attributions that are consistent across all 
models with good empirical performance (\ie all models in the Rashomon Set).
We show that $\textbf{partial}$ orders of local/global feature importance arise from this methodology 
enabling more nuanced interpretations by allowing pairs of features to be incomparable when there is no consensus on 
their relative importance. We prove that every relation among features present in these partial orders also holds in 
the rankings provided by existing approaches. Finally, we present three use cases employing hypothesis spaces with 
tractable Rashomon Sets (Additive models, Kernel Ridge, and Random Forests) and show that partial orders 
allow one to extract consistent local and global interpretations of models despite their under-specification.
\end{abstract}


\section{Introduction}\label{sec:introduction}
The Machine Learning (ML) framework has proven to be an essential tool in many data-intensive 
domains such as software engineering, medicine, and cybersecurity 
\citep{esteves2020understanding,kaieski2020application,salih2021survey}. 
However, the lack of interpretability of complex models is still an important hurdle to their applicability.
For this reason, various model-agnostic techniques such as LIME \citep{ribeiro2016should}, 
SHAP \citep{lundberg2017unified}, and Integrated/Expected Gradient (IG/EG) 
\citep{sundararajan17a,erion2021improving} have recently been developed to provide explanations of model 
decisions in the form of local feature attributions. These attributions are meant to indicate the 
contribution (positive/negative/null) of individual features toward a model prediction, and their 
magnitudes (positive/null) can be used to rank features in order of importance. As researchers and 
practitioners have started to apply these model-agnostic explanations to real-world settings, it has 
become apparent that they are subject to variability. First, given a fixed model, re-running the explainer 
can yield different local feature attributions \citep{visani2020statistical,slack2021reliable,zhou2021s}. 
Second, retraining the model can induce different local explanations for the same decisions 
\citep{fel2021good,shaikhina2021effects,schulz2021uncertainty}. This phenomenon, known as under-specification, 
arises when one employs a rich hypothesis space containing various models that all fit the data while 
having very different behaviors \citep{d2020underspecification}.

In this work, we focus on uncertainty induced by the model under-specification, while controlling the 
variability arising from the explainer. Current literature addresses this uncertainty by aggregating 
local explanations from an ensemble of independently trained models. The aggregation is either 
conducted by averaging the models \citep{shaikhina2021effects}, or averaging the local feature 
importance ranks \citep{schulz2021uncertainty}. We find that, although these methods provide a single 
local feature attribution to explain all models, it is unclear what statements practitioners 
are allowed to make with confidence using said explanation.

Our characterization of explanations uncertainty departs from the current ones by focusing on 
\emph{statements} about local/global feature attribution. Our motto in this context of model 
under-specification is: \textit{only consider statements on which all models with good performance agree}
Concretely, we are going to work with the set of all models with an empirical loss at most $\epsilon$, 
or equivalently, with all models in the Rashomon Set \citep{fisher2019all}.
At a fixed tolerance $\epsilon$, local/global feature attribution statements on which there 
is a consensus in the Rashomon Set form \textbf{partial} orders, instead of the total orders typically used 
to rank features. Partial orders have the advantage to enable safer interpretation by allowing two features to be 
incomparable, which occurs when two models in the Rashomon Set disagree on their relative importance. In such cases, we abstain 
from claiming that one feature is more important than the other and let practitioners study both 
features and decide to modify whichever feature is most actionable. Here is a brief summary of the contributions 
of this work:

\begin{enumerate}
    \item We identify local/global feature attribution \emph{statements} on which there is a perfect 
    consensus across all models with an empirical loss at most $\epsilon$ (\ie all models in the Rashomon Set). 
    These statements result in partial orders, which differ from the total orders commonly used to visualize 
    feature attributions. Our methodology currently supports the Rashomon Sets of Additive Regression, 
    Kernel Ridge Regression, and Random Forests.
    \item We prove that if feature $i$ is locally more important than feature $j$ according to our 
    partial orders, then the same relation holds in the total rankings proposed by 
    \citet{shaikhina2021effects,schulz2021uncertainty}.
    This property establishes that our approach based on partial order is conservative over these total ranking 
    approaches in the sense that it differs from them only by dismissing some relations among features that are 
    deemed uncertain. This is a desirable property given the lack of ground truth in explainability, 
    which restricts quantitative comparisons between competing techniques.
    \item We finally present empirical evidence on three open-source datasets that our 
    partial orders are indeed more cautious than total orders, while still conveying important 
    information about the predictions. Each use-case employs a different class of models to better highlight the 
    versatility of our framework.
\end{enumerate}

The rest of the paper is structured as follows: \textbf{Section \ref{sec:background}}
introduces Machine Learning notation, local/global feature attributions, and the problem of model 
under-specification, \textbf{Section \ref{sec:motivation}} presents a toy-example that serves
as the motivation behind our method, \textbf{Section \ref{sec:method}}
discusses our methodology for asserting consensus in the Rashomon Set, while 
\textbf{Sections \ref{sec:additive}, \ref{sec:kernel}, \& \ref{sec:random_forest}}
apply the methodology to Additive models, Kernel Ridge Regression, and Random Forests respectively. 
Finally, \textbf{Section \ref{sec:discussion}} discusses the results and \textbf{Section \ref{sec:conclusion}} concludes the paper.

\section{Background \& Related Work}\label{sec:background}

\subsection{Machine Learning Notation}\label{sec:background:ML}

In supervised Machine Learning (ML) settings, we work with an input space 
$\X\subseteq \R^d$, a target space $\Y$, an hypothesis space $\Hypo:\X\rightarrow \Y'$, and a 
loss function $\ell:\Y'\times\Y\rightarrow \R_+$. We shall refer to each individual 
function $h\in \Hypo$ as a model or a hypothesis. For parametric hypothesis spaces 
$\Hypo=\{h_{\bm{\theta}}: \bm{\theta} \in \R^p\}$, each realization of the parameters $\bm{\theta}$
is a different model/hypothesis. We suppose there exists a distribution $\D$ over $\X\times \Y$ 
from which examples from a dataset $S=\dataset\sim \D^N$ are sampled iid. 
The ultimate goal of supervised ML is to find a model $h^\star\in\argmin_{h\in \Hypo}\poploss{h}$, 
with minimal population loss $\poploss{h}:=\E_{(\myvec{x}, y)\sim \D}[\ell(h(\myvec{x}), y)]$.
However, since the data-generating distribution $\D$ is unknown, we cannot compute
the population loss $\poploss{h}$ and must resort to studying the empirical loss on the dataset $S$
\begin{equation}
    \emploss{S}{h}:=\frac{1}{N}\sum_{i=1}^N\ell(h(\myvec{x}^{(i)}), y^{(i)}),
    \label{eq:emploss}
\end{equation}
which can be minimized over $\Hypo$ to get an estimate $h_S\in\argmin_{h\in \Hypo}\emploss{S}{h}$ of $h^\star$.
In this work, we study the hypothesis spaces $\mathcal{H}$ of Additive Splines 
\citep[Chapter 5]{hastie2009elements}, Kernel Ridge Regression 
\citep[Chapters 6 \& 11]{mohri2018foundations} and Random Forests \citep{breiman2001random}. 
The two loss functions $\ell$ that are considered are the squared loss $\ell(y',y)=(y'-y)^2$ for a 
continuous target $\Y\subseteq \R$ and the $0\minus\!1$ loss $\ell(y',y)=\mathbbm{1}(y'\neq y)$ 
for a binary target $\Y=\{0,1\}$.

\subsection{Feature Attribution}\label{sec:background:feature_attribution}

The ML paradigm has been successful in tackling tasks where traditional programming methods fail. 
Still, the lack of transparency of some state-of-the-art models such as Random Forests and 
Multi-Layered Perceptrons prohibits their wide-spread application \citep{arrieta2020explainable}.
To meet this novel challenge, the community of eXplainable Artificial Intelligence (XAI) has recently 
been growing with the ambition of \emph{explaining} black box models. In this paper, by \emph{explaining}, 
we mean asking a \emph{contrastive question} about the model and then \emph{answering} said question.

A \emph{contrastive question} takes the form : 
why is the model output $h(\myvec{x})$ so high/low compared to a baseline value? The baseline value is commonly chosen to be the average model 
output $\E_{\myvec{z}\sim\background}[h(\myvec{z})]$ over a distribution $\background$ called the background. At the heart of any contrastive 
question is a quantity called the Gap
\begin{equation}
    G(h, \myvec{x}) := h(\myvec{x}) - \E_{\myvec{z}\sim\background}[h(\myvec{z})].
    \label{eq:gap}
\end{equation}
Therefore, asking a contrastive question amounts to measuring a Gap $G(h, \myvec{x})\neq0$ and 
wondering why it is strongly positive or negative. Examples of contrastive questions include: 
\begin{enumerate}
    \item Why is individual $\bm{x}$ predicted to have a higher-than-average risk of heart 
    disease? Here, the Gap is positive and the background $\background$ is the empirical distribution 
    over the whole dataset.
    \item Why is house $\bm{x}$ predicted to have a lower price than house $\bm{z}$? In that case, the Gap 
    is negative and the background $\background$ is the Dirac measure $\delta_{\bm{z}}$.
\end{enumerate}

Now, to \emph{answer} a contrastive question, we need mathematical tools to probe the model and 
extract information from it. Examples of such techniques are local feature attributions, which are 
vector-valued functionals $\myvec{\phi}:\Hypo\times\X \rightarrow \R^d$ whose vector output represents 
the contribution of each feature towards the Gap
\begin{equation}
    \sum_{i=1}^d\phi_i(h,\bm{x}) = G(h, \myvec{x}).
    \label{eq:additive}
\end{equation}
Since the feature attributions sum up to the Gap, a large positive attribution for feature $i$ is 
interpreted as stating that the input component $x_i$ increased the model output relative to the baseline.
The amplitude of the score $|\phi_i(h, \bm{x})|$ is called the local feature importance and it is 
often used to rank features. That is, we interpret $|\phi_i(h, \bm{x})|<|\phi_j(h, \bm{x})|$ as stating 
that feature $i$ is locally less important than feature $j$ for explaining the Gap.
Returning to the contrastive question on house prices, a negative attribution with maximal local
importance may be given to the size $x_i=\texttt{small}$ of house $\bm{x}$. This 
would mean that the small size of the house is the main factor driving its price down relative to house $\bm{z}$.
In this work, we are only going to consider local feature attributions that are linear w.r.t the model:
\begin{equation}
    \myvec{\phi}(h_1 + \alpha h_2,\bm{x}) = \myvec{\phi}(h_1,\bm{x}) + \alpha \myvec{\phi}(h_2,\bm{x}),
    \label{eq:linearity}
\end{equation}
for any hypotheses $h_1, h_2 \in \mathcal{H}$, and $\alpha\in\R$.
The principal reason for this restriction is that it will render the optimization problems described in 
\textbf{Section \ref{sec:method}} tractable. We now present two linear local feature attributions methods 
that have previously been used to answer contrastive questions about black boxes: SHAP 
\citep{lundberg2017unified}, and Expected Gradient (EG) \citep{erion2021improving}. 

\subsubsection{Shapley Values}\label{sec:background:feature_attribution:shap}

The Shapley values are a fundamental concept from cooperative game theory \citep{shapley1953value}.
Letting $[d]=\{1,2,\ldots,d\}$ be the set of all $d$ features, and given a subset $P\subseteq [d]$ 
of features, we define the replace function $\myvec{r}_{P}:\R^d\times \R^d\rightarrow \R^d$ as
\begin{equation}
    r_{P}(\myvec{z}, \myvec{x})_i = 
        \begin{cases}
        x_i & \text{if  } i \in P\\
        z_i & \text{otherwise}.
    \end{cases}
\end{equation}
Moreover, let $\pi$ be a permutation of $[d]$, $\pi(i)$ be the position of the feature $i$ in $\pi$, and 
$\pi_{:i}=\{j\in[d]\,:\,\pi(j)<\pi(i)\}$. The Shapley values, as defined in the library SHAP 
\citep{lundberg2017unified}, are the average marginal contributions of specifying the $i$th feature from 
the background distribution across all coalitions
\begin{equation}
    \phi^\text{SHAP}_i(h,\bm{x}) := \E_
    {\substack{
    \pi \sim \Omega\\
    \myvec{z}\sim \background}}\left[ h(\,\myvec{r}_{\pi_{:i}\cup \{i\}}(\myvec{z}, \myvec{x})\,) - 
    h(\,\myvec{r}_{\pi_{:i}}(\myvec{z}, \myvec{x})\,)\right],
    \label{eq:shap}
\end{equation}
where $\Omega$ is the uniform distribution over all $d!$ permutations of the $[d]$. Because they involve 
an expectation over all permutations, the Shapley values scale poorly w.r.t the number of features, 
although a method called TreeSHAP was recently developed to reduce the complexity to polynomial assuming 
the model being explained is an ensemble of decision trees \citep{lundberg2020local,laberge2022understanding}.

\subsubsection{Integrated/Expected Gradient}\label{sec:background:feature_attribution:eg}

The Integrated/Expected Gradient (IG/EG) originates from a different background: cost-sharing in economics. It 
is also known as the Aumann-Shapley value and has been previously used to compute saliency maps of Convolutional 
Neural Networks \citep{sundararajan17a,erion2021improving}. The general definition of EG is
\begin{equation}
    \phi^\text{EG}_i(h,\bm{x})
    := \E_{\substack{
    \myvec{z} \sim \background,\\
    t \sim U(0, 1)}}
    \left[(x_i - z_i) \frac{\partial h}{\partial x_i}\bigg|_
    {t\myvec{x} + (1-t)\myvec{z}}\right]
    \label{eq:ig}.
\end{equation}

The main idea of this approach is to average the gradient along linear paths between reference inputs 
sampled from the background and the input $\myvec{x}$ of interest. When the background distribution degenerates 
to a single atom at input $\bm{z}$ ($\background=\delta_{\bm{z}}$), the Expected Gradient falls back 
the so-called Integrated Gradient.

\subsubsection{Global Feature Attribution}\label{sec:background:feature_attribution:GFI}

As a complement to local feature attributions, global feature attributions are vector-valued functionals
$\bm{\Phi}:\Hypo \rightarrow \R_+^d$ that aim to highlight which features are globally most used by the model. 
Unlike local explanations, these functionals $\bm{\Phi}$ are not specific to a given input, and the values of
the attributions are restricted to be positive. Hence, we will often refer to them as global feature 
importance. A straightforward way to extract global insight from local feature 
attributions is to average their absolute value across the data
\begin{equation}
    \Phi^{[1]}_i(h) := 
    \E_{\myvec{x}\sim \D}[\,|\phi_i(h, \myvec{x})|\,]
\end{equation}
which is the by-default scheme in the Python libraries \texttt{SHAP} \citep{lundberg2017unified} 
and \texttt{InterpretML} \citep{nori2019interpretml}.
Another way to combine local attributions into global ones is to average their squared amplitude
\begin{equation}
    \Phi^{[2]}_i(h) := 
    \E_{\myvec{x}\sim \D}[\,\phi_i(h, \myvec{x})^2\,].
\end{equation}
Although this functional has not yet been proposed, it is a natural measure of 
importance for linear models $h(\myvec{x})=\omega_0 + \sum_{i=1}^d \omega_i x_i$ whose local feature 
attributions (using $\background=\D)$ are $\phi_i(h, \myvec{x}) = \omega_i(x_i - \E_{\myvec{z}\sim\D}[z_i])$ \citep{lundberg2017unified}. 
In that case, the global importance $\Phi^{[2]}_i = \omega_i^2\mathbb{V}_{\myvec{z}\sim\D}(z_i)$ 
correspond to the standardized coefficients. We will also see in \textbf{Section \ref{sec:additive}} 
that $\bm{\Phi}^{[2]}$ presents computational advantages over $\bm{\Phi}^{[1]}$ in the case of Additive Regression.

This work focuses on SHAP and EG local feature attributions and their global counterpart 
$\bm{\Phi}^{[1]}$ or $\bm{\Phi}^{[2]}$. Still, but there exist many more post-hoc methods for 
local/global feature attributions. For instance, LIME \citep{ribeiro2016should} 
computes local feature attributions by training a linear model to mimic the behavior of $h$ 
around $\myvec{x}$. This local explainer was not used because it does not respect Equation 
\ref{eq:additive}. Moreover, Permutation Feature Importance \citep{breiman2001random} 
and SAGE \citep{covert2020understanding} extract global feature importance by perturbing 
a feature and reporting the impact on model performance. Studying these two global importance
techniques is part of our future work.

\subsection{Under-Specification and Rashomon Set}\label{sec:background:feature_attribution:rashomon}

The Rashomon Effect \citep{breiman2001statistical}, also known as model 
under-specification \citep{d2020underspecification} or model multiplicity 
\citep{marx2020predictive} refers to the observation that there often exists 
a large diversity of models that fit the data well. This is especially true when
one is employing a hypothesis space with a large capacity. Formally, model under-specification 
can be characterized via the Rashomon Set \citep{fisher2019all}
\begin{definition}[Rashomon Set]
    Given a hypothesis space $\mathcal{H}$, a loss function $\ell$, a data set $S$, 
    and a tolerance threshold $\epsilon >0$, the Rashomon set is defined as
    \begin{equation}
        \Rashomon := \big\{ h \in \mathcal{H} : \emploss{S}{h} \leq \epsilon \big\},
    \end{equation}
    where we leave the dependence in $S$ and $\ell$ implicit from the context.
\end{definition}
Although Rashomon Sets have an appealing and simple interpretation, their computation
is intractable unless $|\Hypo|$ is small or unless $\Hypo$ is the set of linear hypotheses 
fitted with squared loss. Hence, in general settings, the Rashomon Sets have to be estimated, 
which can be done by sampling models and keeping the ones with satisfactory performance 
\citep{dong2019variable,semenova2022existence}. However, this method can be time-consuming 
and requires extensive memory to store thousands of models. 
Alternatively, by relaxing the notion of ``model'' to include all possible feature 
selections (\ie $\Hypo=\cup_{D\subseteq [d]}\Hypo_D$ where $\Hypo_D$ only relies on features in $D$), 
forward selection strategies can enumerate good models more efficiently \citep{kissel2021forward}.

Other approaches work implicitly with the Rashomon Set by solving optimization
problems over $\Hypo$ under the constraint that $\emploss{S}{h} \leq \epsilon$.
In doing so, one can explore the different characteristics of models in the Rashomon
Set without ever needing to represent the set explicitly. Such optimization problems have been 
studied to characterize the under-specification of model predictions 
\citep{marx2020predictive,coker2021theory,hsu2022rashomon}, and global feature importance 
\citep{fisher2019all}. However, to the best of our knowledge, this is the first work that explores 
the range of possible local feature attributions $\bm{\phi}(h,\bm{x})$ across all models 
from the Rashomon Set.

\subsection{Under-Specification of Feature Attributions.}\label{sec:background:feature_attribution:uxai}

In the words of Leo \cite{breiman2001statistical}
``\textit{The multiplicity problem and its effect on conclusions drawn from models needs serious attention}.''
Indeed, since models are under-specified, so are their interpretations via local/global
feature attributions. In practice, this translates to situations where a large set of
independently trained models all yield different local explanations for the same Gap, 
or different rankings of global feature importance. If our goal is to not just understand 
one hypothesis $h$, but also to provide interpretations that are robust to the inherent
under-specification of the ML pipeline, then contradicting explanations are problematic.
Previous work tackles this uncertainty by aggregating the feature attributions of multiple 
independently trained models. They both consider an ensemble $E=\{h_k\}_{k=1}^M$ of $M$ models trained 
independently via a stochastic learning algorithm $h_k\sim \mathcal{A}(S)$. The 
local feature attributions of each of these models are computed 
$\{\myvec{\phi}(h_k, \myvec{x})\}_{k=1}^M$ and aggregated. Uncertainty scores are provided in 
tandem with the aggregated attributions as means to convey how ``confident'' the local attributions are.

For instance, \citet{shaikhina2021effects} aggregate local feature attributions by explaining 
the average model and the uncertainty scores are the variances of feature attributions among models. 
That is, they define the average model $h_E = \frac{1}{M}\sum_{i=1}^M h_k$ and compute its corresponding 
feature attributions $\myvec{\phi}(h_E, \myvec{x})$, which the authors show to be equivalent
to averaging the local feature attributions of each individual model when attribution is a 
linear functional. The uncertainty score for the attribution of feature $i$ is the variance 
$\frac{1}{M}\sum_{k=1}^M (\,\phi_i(h_k, \myvec{x}) - \phi_i(h_E, \myvec{x})\,)^2$.

In a similar effort,  \citet{schulz2021uncertainty} obtain aggregated local explanations 
by averaging the ranks of the feature importance across models 
$\frac{1}{M}\sum_{k=1}^M \mathbf{r}[\,|\bm{\phi}(h_k, \myvec{x})|\,]$,
where $\textbf{r}:\R^d_+ \rightarrow [d]$ is the rank function that maps each component of a vector to 
its rank among the other components. For the uncertainty score for feature $i$ they suggest using the 
ordinal consensus metric, which takes values between $0$ and $1$ and measures the consistency 
between the rankings. As we shall see in \textbf{Section \ref{sec:motivation}}, both of these 
approaches share the same limitations: it is unclear what statements we can/cannot 
make with confidence when analyzing the resulting local feature attributions. Indeed, 
they both end up providing a total order of local feature importance which suggests that 
every feature is either more important or less important than any other feature, irrespective of 
the explanation uncertainty. Moreover, the uncertainty scores shown in tandem with the explanations 
do not easily translate to confidence scores about statements of the form 
\enquote{feature $i$ is locally more important than feature $j$}. Finally, they do not consider 
the whole Rashomon Set but rather employ ensembles of $M$ independently trained models, 
which may underestimate the true under-specification of the ML task.

\section{Motivation}\label{sec:motivation}

\begin{figure*}[t]
    \centering
    \begin{subfigure}[b]{0.47\textwidth}
        \includegraphics[width=\linewidth]
        {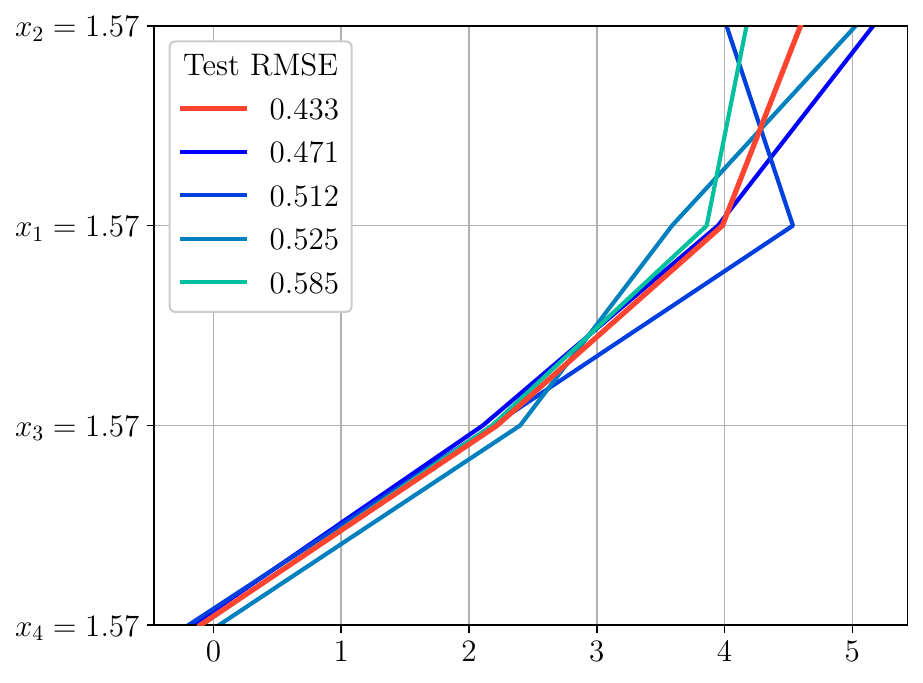}
        %\caption{}
    \end{subfigure}
    %\begin{subfigure}[b]{0.175\textwidth}
    %    \includegraphics[width=\linewidth]
    %    {Images/PO/Motivation/total_5.pdf}
        %\caption{}
    %\end{subfigure}
    %\begin{subfigure}[b]{0.31\textwidth}
    %    \includegraphics[width=\linewidth]
    %    {Images/PO/Toy/toy_ood_bis.pdf}
        %\caption{}
    %\end{subfigure}
    \begin{subfigure}[b]{0.34\textwidth}
%        \centering
        \vfill
       \includegraphics[width=\linewidth]
        {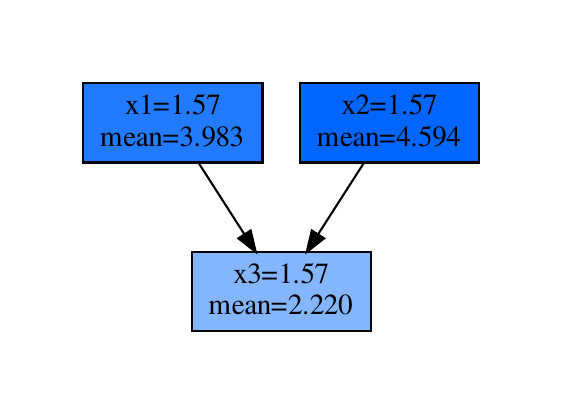}
        \vfill
        \vspace{.75cm}
     \end{subfigure}
    \caption{ Left: local feature attributions for the average model $h_E$ 
    (orange line) and each individual model (blue lines).
    Right: Partial order of local feature importance. There is a directed path from 
    feature $x_i$ to feature $x_j$ if \textbf{all good models} agree that feature $x_i$ 
    is more important than $x_j$.
    }
    \label{fig:toy_result}
\end{figure*}

We illustrate the limitations of current methods and motivate our own with a toy regression problem. 
We sampled 1000 $4$-dimensional points $\bm{x}\sim \mathcal{N}(\bm{0},\bm{\Sigma})$ where
$\bm{\Sigma}$ is identity, except for $\Sigma_{1,2}=\Sigma_{2,1}=0.75$, labelled them via 
$y:=f(\bm{x})+\Delta$, with $f(\myvec{x})=\minus 8 \cos(x_1-x_2)\cos(x_1+x_2)+1.5x_3$ 
($x_4$ is a dummy variable) and $\Delta$ is Gaussian noise with standard deviation $\sigma=0.1$.
We then independently trained five Multi-Layered Perceptrons (MLP) with layerwidths$=50, 20, 10$ and 
ReLU activations. All models ended up having test set Root-Mean-Squared-Error (RMSE) between $0.47$ 
and $0.62$, while the target had a standard deviation of $4.91$. After conducting paired
Student-$t$ tests between the model with RMSE $0.47$ and the four others, we concluded that the one with
error $0.62$ was significantly worst and should be discarded. The other three models did not have 
a significantly worst test RMSE and so we kept them.

We analyzed the predictions of the four remaining models at the input 
$\bm{x}=(\frac{\pi}{2},\frac{\pi}{2},\frac{\pi}{2},\frac{\pi}{2})$ which ranged from $9.05$ to $10.05$ 
(the ground truth being $f(\myvec{x})=0.75\pi+8 \approx 10.36$). Specifying the 
background distribution $\background$ to be the whole training set, we computed the output baselines 
$\E_{\bm{z}\sim\background}[\,h_k(\bm{z})\,]$ which ranged from $\minus1.00$ to $\minus1.07$ 
across the four models. Therefore, for all four models, the prediction Gap $G(h_k, \bm{x})$ 
was positive meaning that running SHAP or EG on all models would answer the same contrastive question: 
why is the prediction at $\bm{x}$ so much higher-than-average? To provide insight into why the Gaps at 
$\bm{x}$ are positive, Figure~\ref{fig:toy_result}~(Left) presents the SHAP local feature attributions 
for all four models as blue lines. We see that the various MLPs lead to different 
interpretations. To make sense of these, we used the two state-of-the-art methods for local feature 
attribution aggregation.

Following \citet{shaikhina2021effects}, we average the predictions of our four models, leading to 
a single predictor $h_E$ with a test RMSE of $0.43$. The resulting SHAP feature attribution is shown as 
an orange line in Figure~\ref{fig:toy_result}~(Left). The total order of local feature importance 
for this average model is represented in the first column of Table \ref{tab:comparison_methods}. 
In particular, this explanation suggests that $x_2$ is more important than $x_1$, which given our 
knowledge of the symmetry of the ground truth seems somewhat spurious. Indeed, since the underlying 
data-generating distribution, the target function $f$, and the point $\bm{x}$ to explain are all symmetric 
w.r.t $x_1$ and $x_2$, an ideal explanation would certainly not support that $x_2$ 
is more important than $x_1$. The uncertainty of the local feature attribution is characterized 
via the variance across the five models, see the second column of Table \ref{tab:comparison_methods}.
We note that variance is higher for the attributions of features $x_1$ and $x_2$,  suggesting that 
their contribution toward the output is more uncertain. Still, it is unclear what variance values are 
low/high enough to label attributions as trustworthy/untrustworthy. Moreover, despite their higher 
variance, features $x_1$ and $x_2$ are locally more important than features $x_3$ and $x_4$ for 
all models. Thus, the variance can lead to an overly pessimistic picture of the insights one can gather 
from feature attributions of multiple models

%Thus, we argue that variance provides a pessimistic picture of the insights one can gather from
%feature attributions of multiple models.

Following \citet{schulz2021uncertainty}, we averaged the ranks of the SHAP local feature 
importance, see the third column of Table \ref{tab:comparison_methods}. This method also suggests that 
$x_2$ is locally more important than $x_1$, which is again spurious. Using the Ordinal 
Consensus as an uncertainty metric (the fourth column of Table \ref{tab:comparison_methods}) 
suggests that all feature importance ranks are confident. Indeed, both $x_1$ and $x_2$ have an Ordinal 
Consensus of $0.83$ seeing that there is only a single model for which the ranks of these two features 
are switched. Nonetheless, looking at Figure~\ref{fig:toy_result}~(Left), the model that contradicts 
all others has a test RMSE of $0.512$, which is the second best of the whole ensemble. Simply put, 
this model offers a different but still valid perspective on the data. However, its opinion is 
``washed out'' by the other three models in the computation of the Ordinal Consensus. 
Hence, we argue that the Ordinal Consensus offers a view of uncertainty that is too optimistic.

\begin{table}[t]
    \centering
    \begin{tabular}{ccccc}
    \toprule
    Feature & Attribution $h_E$ & Variance & Mean rank &  Ordinal Consensus  \\
    \midrule
     $x_2=1.57$ &  4.59 &  0.50 &        2.75 &           0.83 \\
     $x_1=1.57$ &  3.98 &  0.35 &        2.25 &           0.83 \\
     $x_3=1.57$ &  2.22 &  0.10 &        1.0 &           1.00 \\
     $x_4=1.57$ & -0.10 &  0.09 &        0.0 &           1.00 \\
    \bottomrule
    \end{tabular}
    \caption{Aggregated feature attributions and uncertainty scores following previous methods.}
    \label{tab:comparison_methods}
\end{table}

As we have just highlighted, the methods of \citet{shaikhina2021effects} and 
\citet{schulz2021uncertainty} share the same limitations: 

\begin{itemize}
    \item It is unclear what statements one can/cannot make using these frameworks.
    For instance, is $x_2$ really more important than $x_1$ for explaining the gap?
    Both approaches return a total order of local feature importance, which suggests
    one statement of relative importance for every pair of features \ie feature $i$ is
    locally less/more important than feature $j$. As we have seen, the uncertainty metrics 
    provided in tandem with the total orders (Variance or Ordinal Consensus) do not help to 
    decide what statements on relative importance are trustworthy.
    \item It is unclear what is the impact of model performances on the insights provided by these 
    two methods. For instance, the second-best model in the ensemble contradicts all others 
    regarding the relative importance of $x_1$ and $x_2$. However, its opinions are diluted 
    when aggregating all explanations.
\end{itemize}

In light of those takeaways, we decide to focus our method directly on statements about relative feature 
importance, and whether or not all good models agree on them. For instance, how can we decide 
if feature $x_2$ is locally more important than $x_1$? As noted earlier, one model considers, 
contrary to the other four, that $x_1$ is more important than $x_2$. Given that this model is as good 
as any other, we can simply decide to \textbf{abstain} from claiming any relation of importance between 
$x_1$ and $x_2$. In this case, abstention seems indeed a cautious position given the symmetry of the 
ground truth. Following this logic, for every other pair of features, we check if all four models 
agree on their relative importance. For instance, all four models agree that $x_1$ is more important 
than $x_3$. We decide to record this consensus as a trustworthy statement and we represent it with an 
arrow from $x_1$ to $x_3$ in Figure~\ref{fig:toy_result}~(Right). Furthermore, we observe that 
while all four models agree that $x_1$, $x_2$ and $x_3$ have a positive attribution, this is not 
the case for $x_4$ (our dummy variable). Based on this observation, we decide to keep only the 
variables for which all models agree on the sign and exclude $x_4$ from our final explanation.

All relations of importance among pairs of features for which there is consensus among the four models 
actually form a \emph{partial order}, a generalization of total orderings which can be conveniently 
represented using a Directed Acyclic Graph called a Hasse diagram. The partial order of 
Figure~\ref{fig:toy_result}~(Right) summarizes our explanation. Note that the partial order 
suggests that the only relative importance statements we can make are that features $x_1$ 
and $x_2$ are locally more important than $x_3$. These two statements are also supported 
by the total orders of \citet{shaikhina2021effects} and \citet{schulz2021uncertainty}, a fact 
that always holds as discussed in \textbf{Section \ref{sec:method:prior_work}}.

\section{Methodology}\label{sec:method}

\subsection{Consensus on Statements about Feature Attributions}\label{sec:method:consensus}

\subsubsection{Local}\label{sec:method:consensus:local}

Having introduced a basic motivation for considering the consensus among diverse models 
with good performance, we now present a formal description of the approach. First and foremost, 
our theory focuses on statements $s:\Hypo\times \X \rightarrow \{0, 1\}$ about local 
feature attributions. Given a performance threshold $\epsilon>0$, end-users will only be presented 
statements on which there is a perfect consensus for all models in the Rashomon Set 
\begin{equation}
    \forall h\in \Rashomon\quad s(h, \myvec{x})=1.
\end{equation}
We now present various statements about local feature attributions.

\begin{definition}[\textbf{Positive (Negative) Gap}]
     We say that the gap $G(h, \myvec{x})$ is positive (resp. negative) according to $h$ if $G(h,\myvec{x})>0$ (resp. $G(h, \myvec{x})<0$). Formally, the statements take the
     form $s(h, \bm{x})=\mathbbm{1}[G(h,\myvec{x})>0]$ and $s(h, \bm{x})=\mathbbm{1}[G(h,\myvec{x})<0]$.
\end{definition}
Before running SHAP or EG, it is primordial to understand the sign of the gap 
as it is the basis behind the contrastive question we attempt to answer. There may exist instances 
$\myvec{x}^{(i)}$ in the data where there is no consensus on the sign of the gap. Therefore, we let
\begin{equation}
    \text{SG}(\epsilon):=\big\{i \in [N] : \forall h_1,h_2 \in \Rashomon \,\,\, \texttt{sign}[G(h_1, \myvec{x}^{(i)})]=\texttt{sign}[ G(h_2, \myvec{x}^{(i)})]\big\},
\end{equation}
be the sets of data instances on which a contrastive question makes sense. If two models disagree on the sign of 
the Gap, then it is useless to run SHAP or EG on them since these techniques would not end up answering the same
contrastive question. If a contrastive question has been formulated without ambiguity, we can run SHAP or 
EG and analyze the local feature attributions.
\begin{definition}[\textbf{Positive (Negative) Attribution}]
     We say that feature $i$ has positive (resp. negative) attribution according to $h$ if 
     $\phi_i(h, \myvec{x})>0$ (resp. $\phi_i(h, \myvec{x})<0$). More formally, the statements are
     $s(h, \bm{x})=\mathbbm{1}[\phi_i(h,\myvec{x})>0]$ and 
     $s(h, \bm{x})=\mathbbm{1}[\phi_i(h,\myvec{x})<0]$.
\end{definition}
We can now define the sets
\begin{equation}
    \text{SA}(\epsilon, \myvec{x}):=\big\{i \in [d] : \forall h_1,h_2 \in \Rashomon \,\,\, \texttt{sign}[\phi_i(h_1, \myvec{x})]=\texttt{sign}[ \phi_i(h_2, \myvec{x})]\big\},
\end{equation}
which store the features whose attribution has a consistent sign across all good models. After 
identifying the sign of the local feature attributions, it makes sense to order them 
according to their magnitude.
% Note that if a feature is not in this set, this means that there are two models in the Rashomon Set which lead to feature attributions of opposite signs. %We shall also employ the
%set $S(\epsilon):=S_-(\epsilon) \cup S_+(\epsilon)$ of all features with a
%well-defined sign up to tolerance $\epsilon$.
%When asking a contrastive question with positive (resp. negative) gaps, it is useful to only focus on features that have positive/negative attributions according to the ensemble, which we store in sets $S_+$ and $S_-$. 
%\begin{definition}[\textbf{Negligible Importance}]
%    We say that feature $i$ has a negligible
%    importance according to $h$ if $|\phi_i(h, \myvec{x})| \leq \tau$.
%    Formally, the statement is $s(h):=\mathbbm{1}[|\phi_i(h, \myvec{x})| \leq \tau]$
%\end{definition}
%Negligible features can be put in a set
%\begin{equation}
%    N(\epsilon) = \{i\in [d] : \forall h \in \mathcal{R}(\epsilon) \,\,|\phi_i(h, \myvec{x})|\leq \tau\}.
%\end{equation}
%in $S_+$ and $S_-$ are reported and excluded from further explanations and we are left with sets $S_{+,\tau}$ and $S_{-,\tau}$. 
%$S_\tau(\epsilon) = S(\epsilon) \setminus N(\epsilon)$
%The reason we remove negligible features is that the sets $S(\epsilon)$ can be very large depending on the data which may make it harder to interpret results. By letting the threshold $\tau$ be a task-dependent adjustable parameter, practitioners can customize it to obtain explanations that fit their needs. We next look for consensus on relative importance.   
\begin{definition}[\textbf{Local Relative Importance}] We say that feature $i$ is locally 
less important than  $j$ (or equivalently $j$ is locally more important than $i$) according to 
$h$ if $|\phi_i(h, \myvec{x})| \leq |\phi_j(h, \myvec{x})|$. Formally, the statements take the form 
$s(h, \bm{x}):= \mathbbm{1}[|\phi_i(h, \myvec{x})| \leq |\phi_j(h, \myvec{x})|]$.
\label{def:partial_order}
\end{definition}
Note that model consensus on local relative importance leads to a partial order 
$\preceq_{\epsilon, \myvec{x}}$ on $ \text{SA}(\epsilon, \myvec{x})$ defined by:
\begin{equation}
    i \preceq_{\epsilon, \myvec{x}} j \iff 
    \forall\,h \in \Rashomon\  \,\,\,|\phi_i(h, \myvec{x})| \leq |\phi_j(h, \myvec{x})|\, ,
    \label{eq:po}
\end{equation}
$\forall i,j\in \text{SA}(\epsilon, \myvec{x})$. By requiring a perfect consensus on the Rashomon Set, 
we guarantee that the order relations will be transitive. Partial orders differ from the common total 
orders by allowing some pairs of features to be incomparable when there exist two models with conflicting 
evidence on relative importance.

% \subsection{Asserting Consensus on Statements}\label{sec:method:optim}

Recall that asserting the consensus on a statement over the Rashomon Set  (\ie verifying that 
$\forall h_1,h_2\in \Rashomon, \,\,s(h_1,\bm{x})\!=\!s(h_2,\bm{x})\!=\!1$) can require checking that 
uncountably many hypotheses $h$ satisfy that statement. Fortunately, for the specific statements that 
are of interest to us, this can be rephrased as an optimization problem.

\begin{definition}[Local Feature Attribution Consensus]
    Given a tolerance level $\epsilon>0$, a Rashomon Set $\Rashomon$, and a 
    local feature attribution $\bm{\phi}:\Hypo\times \X \rightarrow \R^d$, consensus on statements are asserted via
    the following optimization problems.
    \begin{enumerate}
        \item \textbf{Positive (Negative) Gap : } There is
        consensus that the gap $G(h, \myvec{x})$ is positive (resp. negative) if $\inf_{h\in \Rashomon}G(h,\myvec{x})>0$ (resp. $\sup_{h\in \Rashomon} G(h,\myvec{x})<0$).
        \item \textbf{Positive (Negative) Attribution :} There is consensus that 
        feature $i$ has a positive (resp. negative) attribution if
        $\inf_{h\in \Rashomon}\phi_i(h, \myvec{x}) >0$ 
        (resp. $\sup_{h\in \Rashomon}\phi_i(h, \myvec{x}) <0$).
        %\item \textbf{Negligible Importance :} feature $i$ is said to have a negligible
        %importance according to $\mathcal{R}(\epsilon)$ 
        %if $-\tau\leq \min_{h\in \mathcal{R}(\epsilon)}\phi_i(h, \myvec{x})$
        %and $\max_{h\in \mathcal{R}(\epsilon)}\phi_i(h, \myvec{x}) \leq\tau$.
        \item \textbf{Local Relative Importance :} Let there be a consensus that the 
        attribution of features $i$ and $j$ have signs $s_i$ and $s_j$. Under this assumption, the 
        local feature importance becomes $|\phi_i(h, \myvec{x})| = s_i
        \phi_i(h, \myvec{x})$ for any $h\in \Rashomon$, and similarly for feature $j$. Consequently, 
        there is a consensus that $i$ is locally less important than $j$ if 
        $$\sup_{h\in\Rashomon}s_i
        \phi_i(h, \myvec{x})-s_j\phi_j(h, \myvec{x}) 
        \leq 0.$$
    \end{enumerate}
    \label{def:computing_local_consensus_optim}
\end{definition}
 
These optimization problems may potentially be intractable depending on the hypothesis set
$\Hypo$ and loss functions $\ell$. Nonetheless, we will see that they can be solved exactly 
and efficiently for Additive Regression, Kernel Ridge Regression, and Random Forests.

\subsubsection{Global}\label{sec:method:consensus:global}

We can also consider global model statements $s:\Hypo\rightarrow \{0, 1\}$, which are no longer
specific to any input $\myvec{x}$, and assert a consensus over them. When interpreting models
globally, there is no need to define the notions of Gap or even sign of the attribution. 
Indeed, since global feature importance are already positive, we only need to study statements 
of relative importance.
\begin{definition}[\textbf{Global Relative Importance}] 
    We say that feature $i$ is globally less important than  $j$ (or equivalently, $j$ is globally 
    more important than $i$) according to $h$ if $\Phi_i(h) \leq \Phi_j(h)$. 
    Formally, the statements take the form $s(h):= \mathbbm{1}[\Phi_i(h) \leq \Phi_j(h)]$.
    \label{def:partial_order_global}
\end{definition}
Model consensus on global relative importance defines a partial order 
$\preceq_{\epsilon}$ on $[d]$:
\begin{equation}
    i \preceq_\epsilon j \iff 
    \forall\,h \in \Rashomon\  \,\,\,\Phi_i(h) \leq \Phi_j(h).
    \label{eq:po_global}
\end{equation}
As with local feature attributions, consensus assertion over the Rashomon Set can be rephrased as
an optimization problem.

\begin{definition}[Global Feature Importance Consensus]
    Given a tolerance level $\epsilon>0$, a Rashomon Set $\Rashomon$, and a Global
    Feature Importance $\bm{\Phi}:\Hypo\rightarrow \R^d$, there is a consensus 
    that $i$ is globally less important than $j$ if and only if 
        $$\sup_{h\in\Rashomon}\Phi_i(h)-\Phi_j(h) \leq 0.$$
    \label{def:computing_global_consensus_optim}
\end{definition}

% The next element of the theory is the critical threshold $\epsilon_\text{crit}$
% at which model consensus on a statement $s$ ceases to hold
% \begin{equation}
%     \epsilon_\text{crit}(s) := \sup \{\epsilon \in \R^+ \\
%     \,\,|\,\, s(h_1)=s(h_2)\,\,\,\forall h_1,h_2 \in \mathcal{R}(\epsilon, S, \Hypo)\}.
%     \label{eq:critical_eps}
% \end{equation}
% We make the dependence of $\epsilon_\text{crit}$ on the dataset explicit
% $S$ to accentuate the fact that the critical threshold is a random
% variable. Indeed, if the dataset $S$ is resampled, then the Rashomon Set
% $\mathcal{R}(\epsilon, S, \Hypo)$ changes, and so does the maximal value
% of tolerance at which a consensus ceases to hold.
% Statements that never hold for all models at any tolerance $\epsilon$, will be assigned a critical threshold of $0$.

\subsection{Recommendations for Error Tolerance}\label{sec:method:epsilon}

It remains to address the specification of the error tolerance $\epsilon$. 
This is a critical choice because the tolerance controls the size of the Rashomon Set and therefore the number 
of statements on which consensus is attained. Assuming that $h_S$ is unique, when the tolerance error 
is set to its minimum value we explain a single model $h_S$ and we have total orders of local/global 
feature importance. As we increase $\epsilon$, contradicting explanations will arise and the
total orders will become partial orders. The number of statements present in these partial orders will 
diminish and eventually become null for a sufficiently high $\epsilon$. Thus, varying the error tolerance 
influences \emph{how many} statements about the empirical loss minimizer we abstain from making.

But why would we ever want to abstain from making certain statements supported by $h_S$? Isn't it 
the model that is going to be deployed anyway? The risk is that some explanations of
$h_S$ might be contradicted by another model with \enquote{slightly worst empirical loss}. 
When this occurs, we argue that the explanations of $h_S$ are not trustworthy and we advocate for abstention.
Determining the right notion of \enquote{slightly worst empirical loss} is a difficult problem. 
Here we suggest two approaches 1) one based on statistical guarantees 2) a heuristic 
based on relative error increases.

\subsubsection{Capture Bounds}\label{sec:method:epsilon:capture}

Assume we can find $\epsilon_{\text{max}}$ such that any model with a larger empirical loss 
can be shown to be suboptimal in terms of population loss $\poploss{h}$.
More precisely, with probability $1-\delta$, $\emploss{S}{h}> \epsilon_{\text{max}}$
implies that $\poploss{h}> \poploss{h^\star}$.
Then it is not relevant to set $\epsilon>\epsilon_{\text{max}}$ since the 
Rashomon Set would include models that are likely suboptimal. Assuming $h^\star$ is unique,
%\Mario{This is not generally true. But you do not need to make this hypothesis}\Gabriel{If I don't make this 
%assumption, then having an empirical risk greater than epsilon does not imply that h is sub-optimal in true risk.
%Indeed, one h* could be in the Rashomon set while another h* is not.}
%\Mario{However, it seems to me that the capture bound given by Equation~\ref{eq:cb} is valid simultaneously for all $h^*$ and all $h_S$. I have now emphasized that below.},
this $\epsilon_{\text{max}}$ is the smallest value that respects
\begin{equation}
    \probdata[\emploss{S}{h^\star}\leq\epsilon_{\text{max}}]=
    \probdata[h^\star\in \mathcal{R}(\Hypo, \epsilon_{\text{max}})] > 1-\delta.
    \label{eq:bound_h_star}
\end{equation}
We shall refer to such statistical guarantees as ``Capture Bounds'' since they guarantee 
that the Rashomon Set will ``capture'' the best-in-class model.
By setting $\epsilon=\epsilon_{\text{max}}$, with high probability, any statement on which there is a 
consensus on the Rashomon Set will also hold for the unknown $h^\star$. That is, we explain the 
best model without knowing which one it is. We now present three capture bounds

First, if $\Hypo$ is finite and small (\eg $|\Hypo|\leq 100$), we recommend using Model Set 
Selection \citep{kissel2021forward}. We define the subset $E\subseteq \Hypo$ of all models 
that are not significantly worse than the empirical risk minimizer $h_S$ according 
to a statistical test \eg paired Student-$t$ tests with significance $1-\delta$. 
% More specifically, for a fixed $h\in \Hypo$, we run the Hypothesis Test
% \begin{equation}
%     \begin{aligned}
%         &H_0 : \poploss{h}\leq \poploss{h'} \quad\forall\,\, h'\in \Hypo\\
%         &H_1 : \poploss{h}> \poploss{h_S} \,\,\,\,\text{ for some } h'\in \Hypo
%     \end{aligned}
% \end{equation}
% and define the set $E=\{h\in \Hypo : H_0 \,\,\,\text{was not rejected for}\,\, h\}$ that contains 
% $h^\star$ with probability $1-\delta$. 
Setting  $\epsilon_{\text{max}}=\max\{\emploss{S}{h}\}_{h\in E}$ guarantees that 
Equation \ref{eq:bound_h_star} holds. This capture bound was previously applied to the ensemble of 
five MLPs from \textbf{Section \ref{sec:motivation}}.

Second, if strong assumptions can be made on how the target was generated, then the
following capture bound can be used.

\begin{proposition}
    Under the assumption that the data were generated by the optimal model 
    $h^\star$ plus iid zero-mean Gaussian noise
    \begin{equation}
        y = h^\star(\myvec{x}) + \Delta,\quad\text{where }\,\,\,
        \Delta\sim\mathcal{N}(0, \sigma^2),
        \label{eq:hypo_model_noise}
    \end{equation}
    and using the squared loss $\ell(y',y)=(y' - y)^2$, we have that
    \begin{equation}
        \probdata[\emploss{S}{h^\star}> \epsilon_{\text{max}}] = 
        1-F_{\chi^2_N}\bigg(\frac{N}{\sigma^2}\epsilon_{\text{max}} \bigg),
    \end{equation}
    where $F_{\chi^2_N}$ is the CDF of a chi-2 random variable with $N$ degrees of freedom. 
    The proof is provided in \textbf{Appendix \ref{app:proofs:statistical}}.
    \label{prop:capture_bound_linear}
\end{proposition}

Solving $\delta := 1-F_{\chi^2_N}(\frac{N}{\sigma^2}\epsilon_{\text{max}} )$ for $\epsilon_{\text{max}} $
yields the desired tolerance. If the residuals $\Delta$ follow another law than Gaussian, 
one could replace the $\chi^2_N$ CDF by the CDF of the distribution of $1/N\,\sum_{i=1}^N (\Delta^{(i)})^2$.
The assumption that the data was generated by $h^\star$ plus symmetric noise is very strong, but it
is ubiquitous in Statistics and Linear Regression (See for instance \citep[Section 3.2]{hastie2009elements} 
and \citep[Section 13.5]{wasserman2004all}). Therefore, we think this capture bound 
is \emph{at-least} worth investigating in any regression problem.
\newpage
Third, we suggest this capture bound if a good reference hypothesis can be chosen
apriori \ie before seeing the dataset $S$ on which the empirical loss is computed.

\vspace{-0.25cm}
    \begin{proposition}
    Let  $\ell$ be the $0\minus1$ loss, $S\sim \mathcal{D}^N$ 
    be a dataset, $h_\text{ref}\in\Hypo$ be a reference model that is independent of $S$, and $h^\star$ be a 
    best in-class hypothesis, for any $\epsilon'\in \R^+$, we have
    \begin{equation}
        \probdata[\emploss{S}{h^\star}\geq \epsilon'+\emploss{S}{h_\text{ref}}] \leq \exp\bigg\{ -\frac{N \epsilon'^2}{2}\bigg\}.
    \end{equation}
    The proof is provided in \textbf{Appendix \ref{app:proofs:statistical}}
    \label{prop:capture_bound_classif}
\end{proposition}
Solving $\delta:=\exp\big\{ -\frac{N(\epsilon_{\text{max}} - \emploss{S}{h_\text{ref}})^2}{2}\big\}$ 
for $\epsilon_{\text{max}}$ yields the error tolerance.

\subsubsection{Relative Increase Heuristic}\label{sec:method:epsilon:heuristic}

Capture bounds rely on very strong assumptions and therefore cannot be used out-of-the-box 
for all problems. When they are inapplicable, we recommend the heuristic
\begin{equation}
    \epsilon = (1+\epsilon_\text{rel})\times\emploss{S}{h_S},
\end{equation}
for a $\epsilon_\text{rel}$ typically fixed to $5\%$ \citep{dong2019variable,coker2021theory}, although smaller
values could be used. Setting $\epsilon$ based on this heuristic does not provide any statistical guarantee.
Consequently, any alternative model $h'\in \Rashomon$ that is highlighted by the practitioner
should be compared to $h_S$ using a paired Student-$t$ test on fresh data. For example, if a model 
in the Rashomon Set is found to contradict $h_S$ on a statement of interest, then one should assert that 
the test error of this alternative model is not significantly worse than that of the empirical loss minimizer.
% Using this approach, the statistical guarantees are obtained \emph{after} choosing $\epsilon$, unlike Capture
% Bounds whose statistical guarantees apply \emph{before} choosing $\epsilon$.

\subsubsection{Sensitivity Analysis}

When any experimental choice is made, it is important to conduct a sensitivity
analysis regarding this choice. Otherwise, the results could be disregarded for being 
arbitrary or manipulable. Setting the tolerance $\epsilon$ is one such critical choice and 
therefore we must provide evidence that our conclusions are not too sensitive to the 
specific value of $\epsilon$ employed. To measure such sensitivity, we propose to compute 
the normalized cardinality of the local partial orders 
\begin{equation}
    |\preceq_{\epsilon, \myvec{x}^{(i)}}| := 
    \bigg(\frac{1}{2}d(d+1)\bigg)^{-1} 
    \mathbbm{1}[\myvec{x}^{(i)}\in \text{SG}(\epsilon)]\times 
    |\{(j, k)\in \text{SA}(\epsilon, \myvec{x}^{(i)})^2 : 
    j \preceq_{\epsilon, \myvec{x}^{(i)}}k\}|.
    \label{eq:cardinality}
\end{equation}
For any example $\myvec{x}^{(i)}$, this measure goes from $0$ (when the gap does not have a 
consistent sign) to $1$ (when we have a total order among the $d$ features).
This quantity returns the ratio of statements highlighted by the local
partial order to the total number of possible statements 
$\frac{1}{2}d(d+1)=\frac{1}{2}d(d-1)\,(\text{local relative importance})\,
+d \,(\text{attribution sign}$). 
Given the cardinality measure, a sensitivity analysis on $\epsilon$
would involve asserting the stability of the histograms 
$\{\,|\preceq_{\epsilon, \myvec{x}^{(i)}}|\,\}_{i=1}^N$ for small perturbations of $\epsilon$.

\subsection{Relation To Prior Work}\label{sec:method:prior_work}

Prior methods for characterizing the effect of model uncertainty on local feature attributions 
have mainly focused on explaining an ensemble of models $E=\{h_k\}_{k=1}^M$ trained with the same stochastic 
learning algorithm $h_k\sim \mathcal{A}(S)$ \citep{shaikhina2021effects,schulz2021uncertainty}. We go a step 
further by studying the feature attributions of all models in the Rashomon Set. For this reason, it may not 
be immediately clear how our method compares to prior work. The following proposition shows that what we propose is a 
more conservative alternative to both existing methods.

\begin{proposition}
    Let $\bm{\phi}(\cdot, \bm{x})$ be a linear local feature attribution functional, 
    and $E=\{h_k\}_{k=1}^M$ be an ensemble of $M$ models from $\Hypo$ trained with the same 
    stochastic learning algorithm $h_k\sim \mathcal{A}(S)$. Said local feature attribution and 
    ensemble will be employed in the methods of \citep{shaikhina2021effects, schulz2021uncertainty}. 
    Moreover, let $\epsilon \geq \max \{\emploss{S}{h_k}\}_{k=1}^M$ be an error tolerance, and let 
    $\preceq_{\epsilon, \myvec{x}}$ be the consensus order relation on $\text{SA}(\epsilon, \myvec{x})$ 
    (cf. Equation \ref{eq:po}). If the relation $i\preceq_{\epsilon, \myvec{x}}j$ holds,
    we have that $i$ is locally less important than $j$ in the two total orders of prior work
    \citep{shaikhina2021effects, schulz2021uncertainty}.
    \label{prop:relation_to_prior}
\end{proposition}

This proposition is key as it implies that our framework will not provide users with statements that are not supported
by existing approaches. In a way, all we do is abstain from making statements whose uncertainty is highest.
We think this is an important property to have because, unlike model predictions, there are no ground truths for 
feature attributions. For example, a practitioner can apply multiple aggregation mechanisms to model predictions
(Arithmetic Mean, Geometric Mean, Majority Vote etc.) and compare the resulting test set performances using the 
target $y$ as ground truth. However, when aggregating feature attributions using different schemes, there is no 
metric for what feature importance ranking is the best, or closest to ground truth. This is one of the major 
challenges currently faced by the explainability community. Still, since our framework only highlights 
statements supported by existing approaches, we avoid the need for quantitative comparisons.

The following \textbf{Section \ref{sec:additive}, \ref{sec:kernel}, \& \ref{sec:random_forest}} each presents a 
practical application of our framework using a different hypothesis space and dataset. The code to reproduce 
these experiments is available online\footnote{\url{https://github.com/gablabc/Partial_Order_in_Chaos}}.

\section{Application to Additive Regression}\label{sec:additive}
\subsection{Rashomon Set}\label{sec:additive:rashomon}
Additive models have the form $h(\myvec{x}) := \omega_0 + \sum_{j=1}^d h_j(x_j)$ where each function 
$h_j$ only depends on the feature $x_j$. Since the output is the sum of $d$ functions $h_j$, 
the attribution of each individual feature is readily available which is why these models are 
advertised as transparent. To fit an additive model, one must choose a class of hypotheses for
each univariate function $h_j$. A first method is to represent each of the
functions non-parametrically via a sum of univariate decision trees. This scheme is
what is currently done in the \texttt{ExplainableBoostingMachine} of
the \texttt{InterpretML}  Python library \citep{nori2019interpretml} for instance.
The parametric alternative is to define a basis $\{h_{jk}\}_{k=1}^{M_j}$ along each
dimension $j$ (for example using Splines) and represent the additive model using linear combinations 
of these basis functions \citep[Chapter 5]{hastie2009elements}
\begin{equation}
    h_{\bm{\omega}}(\myvec{x}) := \omega_0 + \sum_{j=1}^d \underbrace{\sum_{k=1}^{M_j}
    \omega_{jk} h_{jk}(x_j)}_{h_j(x_j)}=\bm{\omega}^T \bm{h}(\bm{x})
    \label{eq:basis_additive}
\end{equation}
where $$\bm{\omega}:=[\omega_0, 
\underbrace{\omega_{11}, \omega_{12},\ldots,\omega_{1,M_1}}_{\text{feature 1}}, 
\underbrace{\omega_{21}, \omega_{22}, \ldots,\omega_{2,M_2}}_{\text{feature 2}}, 
\ldots, 
\underbrace{\omega_{d1}, \omega_{d2}, \ldots, \omega_{d,M_d}}
_{\text{feature $d$}}]^T,$$
and 
$$\bm{h}(\bm{x}):=[1, 
\underbrace{h_{11}(\myvec{x}), h_{12}(\myvec{x}),
\ldots,h_{1M_1}(\myvec{x})}_{\text{feature 1}}, 
\ldots,
\underbrace{h_{d1}(\myvec{x}), h_{d2}(\myvec{x}),
\ldots,h_{dM_d}(\myvec{x})}_{\text{feature d}}]^T.$$

By letting $\bm{H}$ be the $N\times(1+\sum_{j=1}^dM_j)$ matrix whose
$i$th row is $\bm{h}(\bm{x}^{(i)})^T$,
the empirical loss minimizer for the squared loss takes the familiar form
\begin{equation}
    \leastsq = (\bm{H}^T\bm{H})^{-1}\bm{H}^T\bm{y}.
    \label{eq:sol_additive}
\end{equation}

\begin{definition}[Rashomon Set for Parametric Additive Regression]
    Let $\Hypo$ be the set of Parametric Additive Regression models (cf Equation 
    \ref{eq:basis_additive}), $\ell$ be the squared loss, $S$ be a dataset of
    size $N$, and 
    $\leastsq=\argmin_{h\in \Hypo}\emploss{S}{h}$ be the least-square estimate. If one
    uses the performance threshold $\epsilon \geq \emploss{S}{\leastsq}$,
    then the Rashomon set $\Rashomon$ consists of all parameters $\bm{\omega}$ s.t.
    \begin{equation}
        (\bm{\omega}-\leastsq)^T\frac{\bm{H}^T\bm{H}}{N}(\bm{\omega}-\leastsq) 
        \leq \epsilon - \emploss{S}{\leastsq}.
    \end{equation}
    We see that the Rashomon Set is an ellipsoid in parameter space. Moreover, if we let
    $\epsilon < \emploss{S}{\leastsq}$, then the Rashomon Set is empty.
\end{definition}

This result is a simple generalization of the Rashomon Set of Ridge Regression derived in 
\cite{semenova2022existence} to Parametric Additive Regression.

\subsection{Asserting Model Consensus}\label{sec:additive:consensus}

\subsubsection{Local Feature Attribution}\label{sec:additive:consensus:LFA}

In addition to having an analytical expression of their Rashomon Set,
additive models also have a clear notion of local feature attribution. 
For instance, running SHAP and EG on an additive model while taking the whole dataset $S$ 
as the background yields the same result
\vspace{-20pt}
\begin{equation}
    \begin{aligned}
        \phi_j^\text{SHAP}(h, \myvec{x}) &= \phi_j^\text{EG}(h, \myvec{x})= h_j(x_j) - \frac{1}{N} \sum_{i=1}^N h_j(x_j^{(i)})\\
        &=\sum_{k=1}^{M_j} \omega_{jk} \bigg(h_{jk}(x_j) - 
        \frac{1}{N}\sum_{i=1}^N h_{jk}(x_j^{(i)})\bigg)
        =\sum_{k=1}^{M_j} \omega_{jk} \overline{h}_{jk}(x_j)
        =\bm{\omega}^T_j\, \overline{\bm{h}}_j(\myvec{x}) ,
        \label{eq:basis_shap_additive}
    \end{aligned}
\end{equation}
which is a linear function of the weights $\bm{\omega}$. We have
seen previously in \textbf{Definition \ref{def:computing_local_consensus_optim}} that asserting the 
consensus on local feature attribution statements amounts to optimization problems that are 
linear with respect to the attributions $\bm{\phi}$. Therefore, asserting a consensus on the 
Rashomon Set of Parametric Additive models requires maximizing/minimizing a linear function on an ellipsoid
\begin{equation}
    \begin{aligned}
        \minormax_{\bm{\omega}}\quad & \bm{a}^T \bm{\omega}\\
        \textrm{with} \quad & (\bm{\omega}-\leastsq)^T\bm{A}(\bm{\omega}-\leastsq)
        \leq \epsilon - \emploss{S}{\leastsq},
    \end{aligned}
    \label{eq:optim_linear}
\end{equation}
with $\bm{A}:=\frac{\bm{H}^T\bm{H}}{N}$ and assuming $\epsilon \geq \emploss{S}{\leastsq}$. The value of $\bm{a}$ depends 
on the type of statement and the instance $\bm{x}^{(i)}$ being explained:
\begin{itemize}
    \item \textbf{Positive (Negative) Gap}\vspace{5pt} \newline
    $\bm{a}=[0, 
    \overline{h}_{11}(\myvec{x}^{(i)}),
    \ldots,\overline{h}_{1M_1}(\myvec{x}^{(i)}), 
    \ldots,
    \overline{h}_{d1}(\myvec{x}^{(i)}),
    \ldots,\overline{h}_{dM_d}(\myvec{x}^{(i)})]$
    \item 
    \textbf{Positive (Negative) Attribution of Feature $j$}
    \vspace{5pt} \newline
    $\bm{a}=[0,\ldots, 
    \overline{h}_{j1}(\myvec{x}^{(i)}), \overline{h}_{j2}(\myvec{x}^{(i)}),
    \ldots,\overline{h}_{jM_j}(\myvec{x}^{(i)}), 
    \ldots, 0]$
    \item \textbf{Local Relative Importance of Features $i$ and $j$}
    \vspace{5pt} \newline
    $\bm{a}=[0, 
    \ldots, s_i \overline{h}_{i1}(\myvec{x}^{(i)}), \ldots, s_i \overline{h}_{iM_i}(\myvec{x}^{(i)}), 0, \ldots, 0, 
    \minus s_j \overline{h}_{j1}(\myvec{x}^{(i)}),
    \ldots,\minus s_j\overline{h}_{jM_j}(\myvec{x}^{(i)}),\ldots, 0]$
\end{itemize}

These optimization problems have an analytical solution that can be computed
rapidly using the Cholesky decomposition $\bm{A}= \bm{C}\bm{C}^T$. 
The optimal values of Equation \ref{eq:optim_linear} are 
\begin{equation}
    \pm\sqrt{\epsilon - \emploss{S}{\leastsq}}\,\,\|\bm{a}'\|+\bm{a}^T\leastsq,
    \label{eq:sol_optim_linear}
\end{equation}
where $\bm{a}'= \bm{C}^{-1} \bm{a}$ see \textbf{Appendix \ref{app:optim:ellipsoid:linear}} for more details. 
This result is a generalization of Theorem 4 from \citet{coker2021theory} to Additive models and arbitrary 
linear functionals of the weights $\bm{a}^T\bm{\omega}$. We deduce from Equation \ref{eq:sol_optim_linear} that 
the minimum and maximum values of any linear functional evaluated on the Rashomon Set are a deviation of 
$\sqrt{\epsilon - \emploss{S}{\leastsq}}\|\bm{a}'\|$ from $\bm{a}^T\leastsq$
the value of the functional evaluated on the least-square. Since the deviation is an explicit function of the 
tolerance $\epsilon$, a consensus on local feature attribution statements can be efficiently asserted at any 
tolerance level.

% This deviation scales with the square root of the tolerance $\epsilon$. Hence if a statement initially holds for the least-square
% estimate, we increase the tolerance $\epsilon$ which increases the deviation $\sqrt{\epsilon - \emploss{S}{\leastsq}}\|\bm{a}'\|$ and
% identify $\epsilon_\text{crit}(s, S)$ as the smallest $\epsilon$ for where the consensus breaks.

\subsubsection{Global Feature Importance}\label{sec:additive:consensus:GFI}

Now investigating global feature importance, we observe that the functional
$\Phi_j^{[2]}$ is a quadratic form of the weights
\begin{equation}
    \begin{aligned}
        \Phi_j^{[2]}(h) 
        &:=\frac{1}{N}\sum_{i=1}^N\,\phi_j(h, \myvec{x}^{(i)})^2
        %=\sum_{i=1}^N \big(\bm{\omega}^T_j \,\overline{\bm{h}}_j(\myvec{x}^{(i)})\big)^2
        = \frac{1}{N}\sum_{i=1}^N \bm{\omega}^T_j \,\overline{\bm{h}}_j(\myvec{x}^{(i)})\,
        \overline{\bm{h}}_j(\myvec{x}^{(i)})^T\, \bm{\omega}_j \\
        &= \bm{\omega}_j^T \bigg(\frac{1}{N}\sum_{i=1}^N 
        \overline{\bm{h}}_j(\myvec{x}^{(i)})\,
        \overline{\bm{h}}_j(\myvec{x}^{(i)})^T\bigg) \bm{\omega}_j = 
        \bm{\omega}_j^T \bm{B}_j \bm{\omega}_j.
        \label{eq:basis_gfi_additive}
    \end{aligned}
\end{equation}
Therefore, asserting a consensus on global relative importance statements in the 
Rashomon Set of Additive Regression 
(solving \textbf{Definition \ref{def:computing_global_consensus_optim}}) 
requires optimizing a quadratic form over an ellipsoid
\begin{equation}
    \begin{aligned}
        \minormax_{\bm{\omega}}\quad & 
        \bm{\omega}_i^T \bm{B}_i \bm{\omega}_i-\bm{\omega}_j^T \bm{B}_j \bm{\omega}_j\\
        \textrm{with} \quad & 
        (\bm{\omega}-\leastsq)^T\bm{A}(\bm{\omega}-\leastsq)
        \leq \epsilon - \emploss{S}{\leastsq},
    \end{aligned}
    \label{eq:optim_qpqc_centered}
\end{equation}
which is known as the Trust-Region-Subproblem (TRS). 
Impressively, by Corollary 7.2.2 of \cite[Section 7.2]{conn2000trust} this problem has 
necessary optimality conditions for the global optimum, even when the quadratic form is non-convex. 
We describe our TRS solver in \textbf{Appendix \ref{app:optim:ellipsoid:quadratic}}. 
We end by noting that \citet{fisher2019all} previously defined the Model Class Reliance
as the interval $\big[\min_{h\in\Rashomon}\Phi_i(h), \max_{h\in\Rashomon}\Phi_i(h)\big]$ which they
computed for Ridge Regression by solving a TRS. However, our framework is more general
because we can also assert a consensus on relative importance relations \ie all 
good models agree that $i$ is globally less important than $j$.

\subsection{House Price Prediction}\label{sec:additive:experiments}

The Kaggle-Houses\footnote{\footnotesize\url{https://www.kaggle.com/c/house-prices-advanced-regression-techniques}} 
dataset consists of predicting the logarithm of the selling price of 2919 houses based on 79 numerical and 
categorical features. The training set $S$ contains the first 1460 houses which are labeled, while the 
test set regroups the remaining 1459 houses whose selling prices are hidden by Kaggle. The only way to measure
test performance is to submit predictions on the Kaggle Website.

For simplicity, we only selected numerical features and removed time-related features
since we are only interested in the physical properties of the houses. Moreover, features that were 
perfectly collinear with others were ignored since they would render the matrix $\bm{H}^T\bm{H}$ singular.
We were left with 19 numerical features which were non-redundant, although some had a very high Spearman 
correlation : \texttt{GarageArea/GarageCars},
\texttt{BsmtPercFin/BsmtFullBath}, and \texttt{BedroomAbvGrd/TotRmsAbvGrd}. We decided to keep correlated features to see how 
they impact model underspecification.

\begin{figure*}[t]
    \begin{subfigure}[b]{0.52\textwidth}
       \includegraphics[width=\linewidth]
        {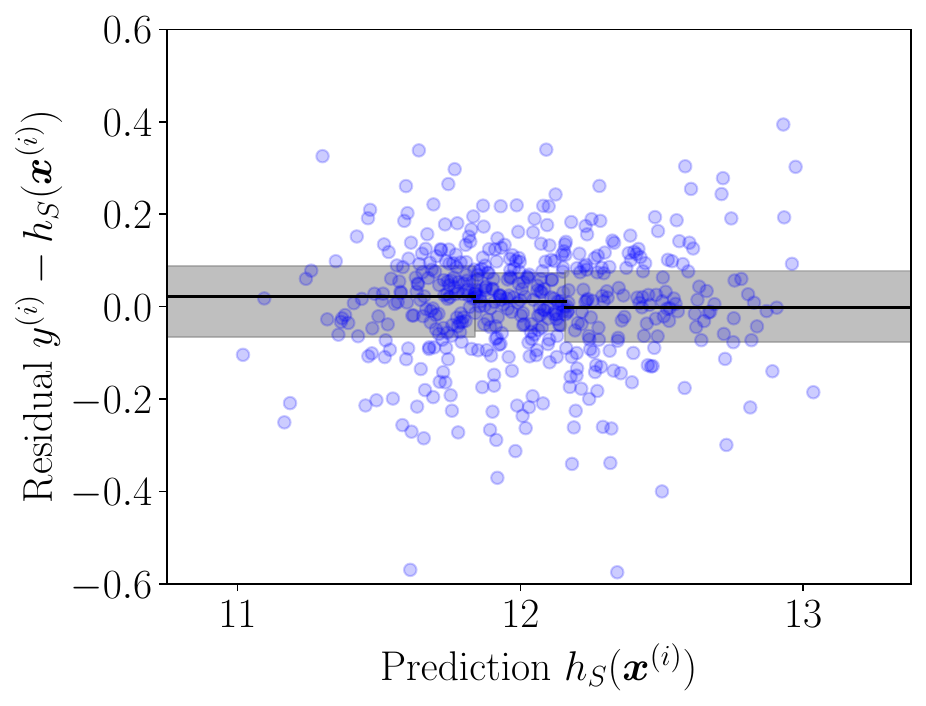}
    \end{subfigure}
    \hfill
    \begin{subfigure}[b]{0.48\textwidth}
       \includegraphics[width=\linewidth]
        {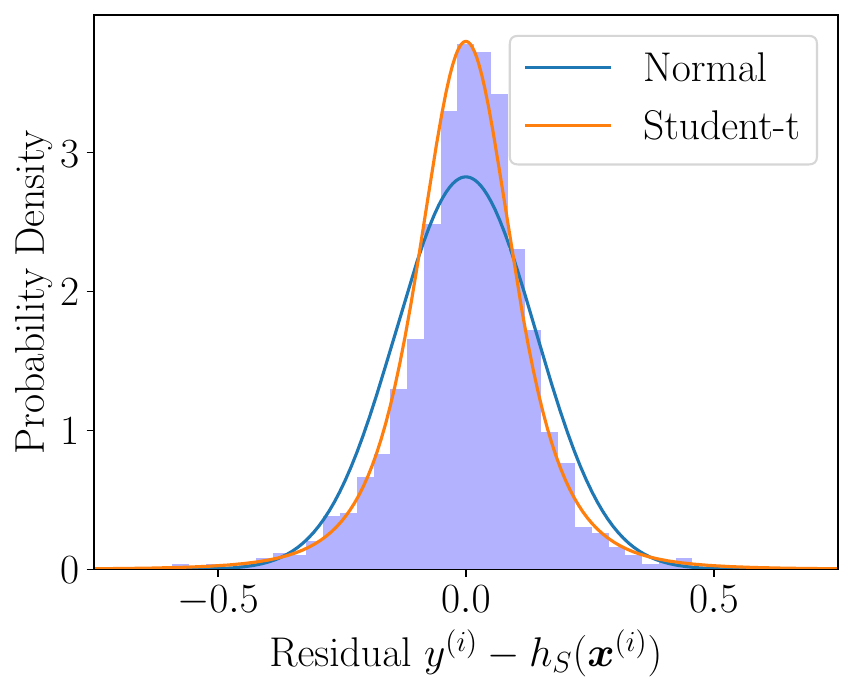}
    \end{subfigure}
    \caption{
    Residuals Analysis of $h_S$. (Left) Residual as a function of the prediction to 
    assess homogeneity. The horizontal lines represent the $25^\text{th}$, $50^\text{th}$, 
    and $75^\text{th}$ percentiles for three different prediction bins. (Right) Histogram 
    of the residuals and fitted densities.
    }
    \label{fig:kaggle_residual}
\end{figure*}

% when feature pairs had a high Spearman correlation 
% $(>\!0.55)$, only one feature was kept. Correlated features were removed because we suspect they would aggravate 
% the under-specification of feature attributions since competing models would rely on different subsets of 
% correlated features. Moreover, we removed the two time-related features \texttt{YearSold} and 
% \texttt{YearRemodAdd}, since we are only interested in the physical properties of the houses. We were 
% finally left with 15 features in total. 

% We trained additive regression models with splines provided by the
% \texttt{SplineTransformer} of the Scikit-Learn Python library \citep{scikit-learn}. 
Additive Regression requires deciding which $h_j$ to parametrize with spline bases and which to parametrize
as linear functions of the input $h_j(x_j)=\omega_j x_j$. For each feature, we fitted the target with a 
depth-3 decision tree using only that feature as input and selected the $k$ features with the lowest RMSE 
for spline parametrization. We tuned the hyperparameter $k$, the polynomial degree 
of the splines, the number of knots, and their positions via five-fold cross-validation.
The resulting least-square models had a train RMSE of 0.141.
As references for test performance, predicting the average training set target yields a
RMSE of 0.426 on Kaggle while Gradient Boosting\footnote{\url{https://www.kaggle.com/code/eesuck/xgboost-regressor}} 
leads to an error of 0.127. In the case of Additive Regression, we got a test error of 0.150.

To quantify the under-specification of our hypothesis class, we computed the Rashomon Set
of all good models on the training set. We could not use the test set since labels are not available.
To fix a reasonable value of tolerance $\epsilon$, we investigated whether the assumptions 
behind the capture bound of \textbf{Proposition \ref{prop:capture_bound_linear}} were reasonable 
on this dataset. That is, could the labels have been provided by the best-in-class $h^\star$ plus iid noise
$\Delta$? We first assumed that $h_S$ and $h^\star$ make similar enough predictions on training data to 
view the residuals $\{y^{(i)} - h_S(\myvec{x}^{(i)})\}_{i=1}^N$ as noise samples
$\{\Delta^{(i)}\}_{i=1}^N$. Figure \ref{fig:kaggle_residual}~(Left) supports that the residuals are homogeneous
but Figure \ref{fig:kaggle_residual}~(Right) reveals they are not Gaussian and are better modeled with a 
Student-$t$. Supported by these observations, we modeled the noise $\Delta$ with a Student-$t$ 
distribution fitted on the residuals. Afterward, we approximated the distribution of 
$\emploss{S}{h^\star}=\frac{1}{N}\sum_{i=1}(\Delta^{(i)})^2$ with the empirical distribution resulting 
from sampling $\{\Delta^{(i)}\}_{i=1}^N\sim t_\nu^N$ a total of $2\!\times\! 10^5$ times. 
Taking the $95^\text{th}$ percentile of this empirical distribution yielded the tolerance 
$\epsilon_{\text{max}}=0.1444$. Under our assumptions, by fixing 
$\epsilon=\epsilon_{\text{max}}=0.1444$ we have an 
approximate $95\%$ chance that the Rashomon Set will include the best-in-class model.

\begin{figure*}[t]
    \begin{subfigure}[b]{0.5\textwidth}
       \includegraphics[width=\linewidth]
        {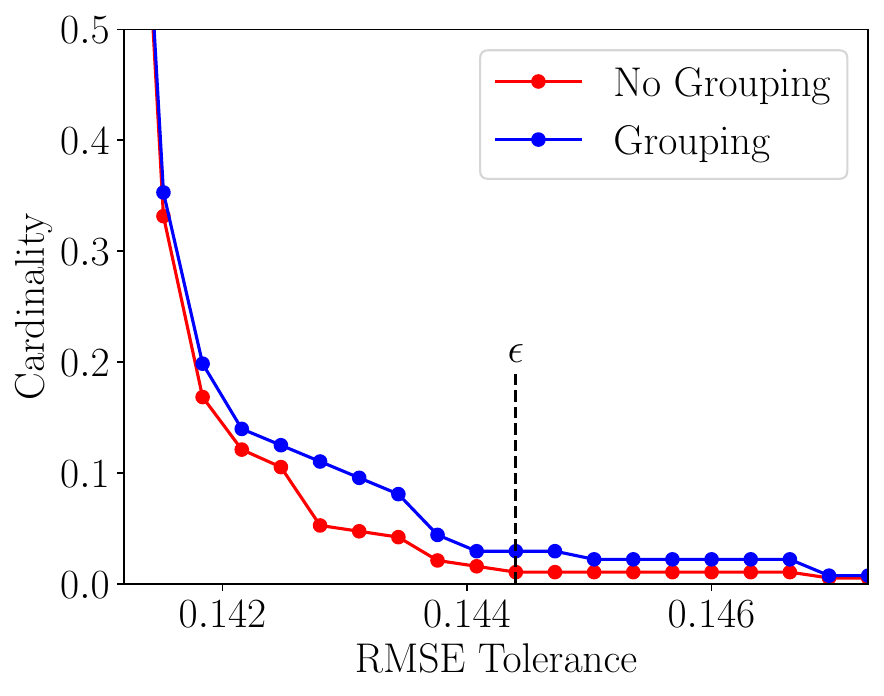}
    \end{subfigure}
    \hfill
    \begin{subfigure}[b]{0.42\textwidth}
       \includegraphics[width=\linewidth]
        {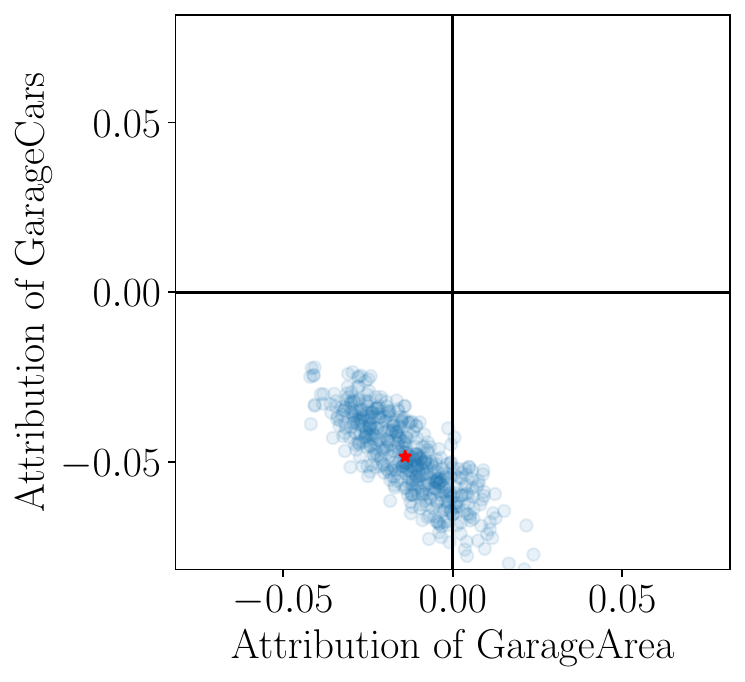}
    \end{subfigure}
    \caption{(Left) Sensitivity Analysis regarding the choice of $\epsilon$.
    The median partial-order cardinalities are shown as a function of the 
    tolerance on training RMSE. The two curves represent whether or not we group 
    correlated features together. (Right) Local Feature Attributions of models sampled
    from the Rashomon Set boundary. We observe a trade-off between local attributions of
    correlated features.}
    \label{fig:kaggle_epsilon_tune}
\end{figure*}

\subsubsection{Local Feature Attribution}

\begin{figure*}[t]
    \centering
    \begin{subfigure}[c]{0.49\textwidth}
        \includegraphics[width=\linewidth]
        {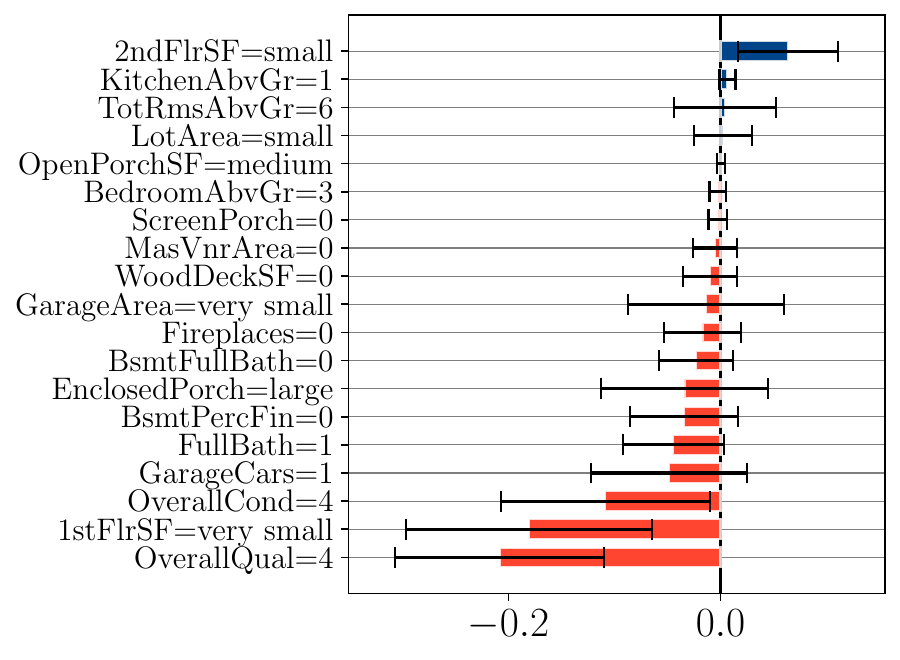}
        \label{fig:kaggle_attrib_instance_2}
    \end{subfigure}
    \hfill
    \begin{subfigure}[c]{0.49\textwidth}
        \raisebox{0.5cm}{
        \hspace{0.25cm}
            \includegraphics[width=0.87\linewidth]
            {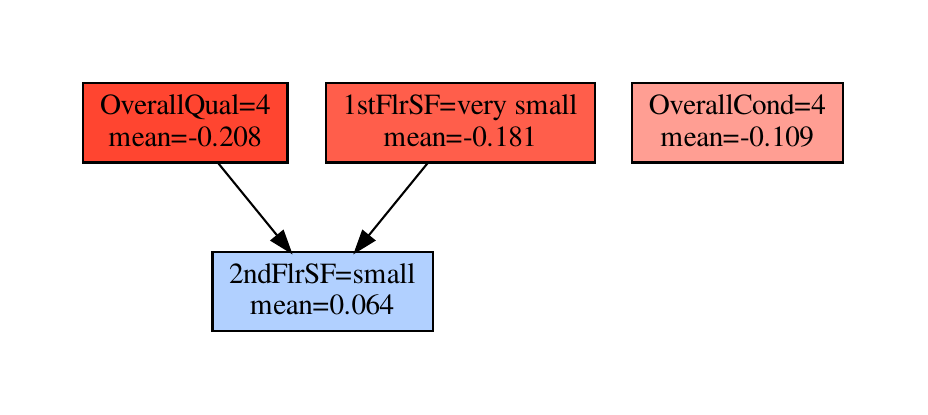}
        }
    \end{subfigure}\vspace{-0.5cm}
    \begin{subfigure}[c]{0.49\textwidth}
        \includegraphics[width=\linewidth]
        {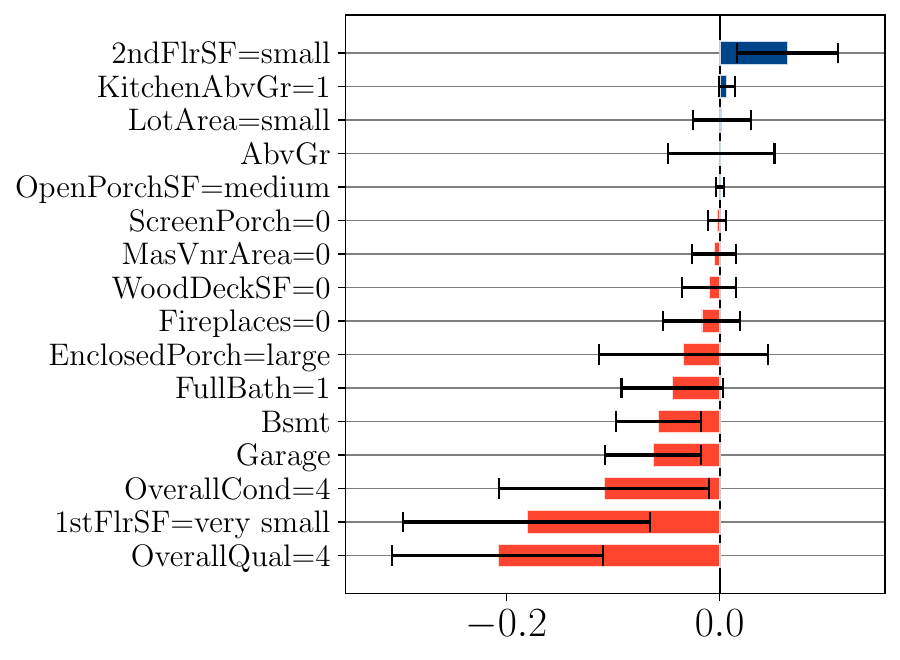}
        \label{fig:kaggle_attrib_instance_2}
    \end{subfigure}
    \begin{subfigure}[c]{0.49\textwidth}
        \raisebox{0.5cm}{
        \includegraphics[width=\linewidth]
        {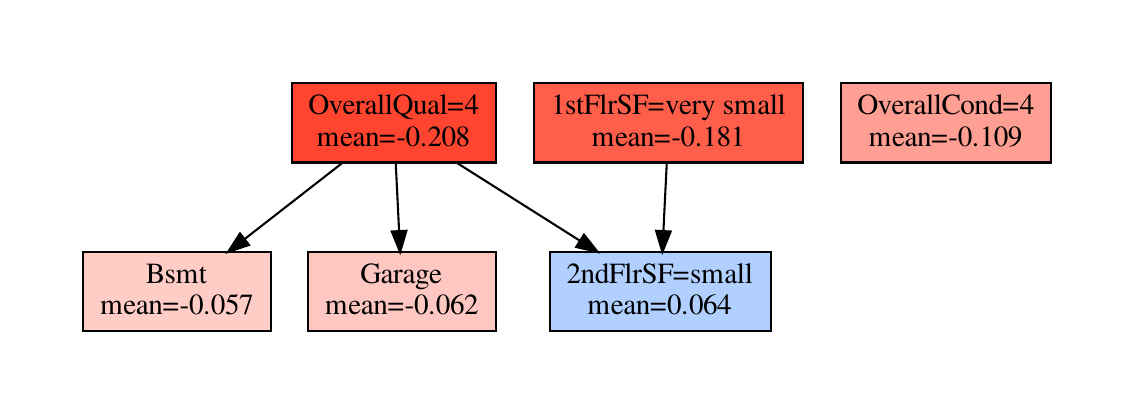}
        }
    \end{subfigure}
    \vspace{-.5cm}
    \caption{Local feature attributions of a house with a below-average price. 
    (Top) Without grouping. (Bottom) With grouping.}
    \label{fig:kaggle_explain_instances}
\end{figure*}

Local feature attributions were computed on all houses in the training set using 
Equation \ref{eq:basis_shap_additive}. To conduct a sensitivity analysis of our 
local explanations with respect to the choice of $\epsilon$, we computed the partial order 
cardinalities (cf. Equation \ref{eq:cardinality}) at several tolerance values, see 
the red curve in Figure \ref{fig:kaggle_epsilon_tune}(Left). We observe that the 
cardinalities are stable with respect to small perturbations of $\epsilon$. However, the cardinalities are 
rather small, which we suspect is partly due to feature correlations. To test this hypothesis, we sampled
models from the Rashomon Set boundary and compared their local attributions for correlated features, see
Figure \ref{fig:kaggle_epsilon_tune}(Right). We observe a trade-off: the more models rely on one feature, the
less they rely on the other. To deal with this under-specification, we propose to
\emph{group} correlated features $i$ and $j$ and consider their \emph{joint} local attribution
\begin{equation}
    \phi_{ij}(h, \bm{x}) := \phi_i(h, \bm{x}) + \phi_j(h, \bm{x}).
\end{equation}
instead of their separate local attribution. Therefore, we group \texttt{GarageArea/GarageCars} 
into \texttt{Garage}, \texttt{BsmtPercFin/BsmtFullBath} into \texttt{Bsmt} and 
\texttt{BedroomAbvGrd/TotRmsAbvGrd} into \texttt{AbvGrd}. Doing so, one obtains partial orders with 
higher cardinalities as evidenced by the red curve in Figure \ref{fig:kaggle_epsilon_tune}(Left), 
suggesting that grouping correlated features can reduce explanation under-specification. In the sequel, we will
present local/global feature attributions with and without grouping.

We explained the predictions on the house with the fifth-smallest selling price: 40K USD. Said predictions
ranged from 70K to 100K in the Rashomon Sets of both Scenarios and there was a consensus that the gap was negative. 
Figure \ref{fig:kaggle_explain_instances} shows the local feature attribution on this instance and the partial 
orders that summarize all the statements good models agree on. We observe that features \texttt{OverallQual=4}
(quality of materials and finish of the house from a scale of 1 to 10) and 
\texttt{1stFlrSF=very small} have maximal importance when explaining the drop in price relative to the 
mean. These statements are robust to the choice of model within the Rashomon Set. 
\texttt{OverallQual=4} also has maximal importance but, because it is incomparable 
to any other feature, we find it safer to simply ignore it. Moreover, we note that there are no 
garage-related and basement-related features in the Hasse diagram without Grouping.
As illustrated in Figure \ref{fig:kaggle_explain_instances}~(Top-Left), the attributions of 
highly correlated features such as \texttt{GarageArea/GarageCars} and \texttt{BsmtPercFin/BsmtFullBath} 
do not have a consistent sign. This is because competing models can rely on one feature or the other, 
which prohibits a consensus on which feature leads to a decrease in selling price. By considering the
joint local attribution of correlated features, the attributions of the groups
\texttt{Garage} and \texttt{Bsmt} become consistently negative, see Figure 
\ref{fig:kaggle_explain_instances}~(Bottom-Left). Hence, our framework allows us to get consistent
model interpretations in spite of the presence of strong feature correlations.

Finally, we note that, at 
tolerance $\epsilon=0.1444$, the sign of the gaps is well-defined for $68\%$ of the houses. 
For about one-third of houses, it does not make sense to ask the contrastive question: 
\emph{Why is this house price higher/lower than average?}.
We discuss how to deal with those houses in \textbf{Appendix \ref{app:gap}}.

\newpage
\subsubsection{Global Feature Importance}

We end this section by presenting global feature importance in Figure \ref{fig:kaggle_explain_global}. 
For simplicity, we only include in the Hasse diagrams the features whose global 
importance is non-null across the whole Rashomon set. Such features appear to 
be necessary in the sense that every model in the Rashomon Set relies on them.
As seen previously, the partial order without Grouping does not contain features related 
to the basement and the garage. We believe that this can be again 
attributed to strong feature correlations. By grouping correlated features,
we discover that the joint effects of \texttt{Garage} and \texttt{Bsmt} are important for all
good models.

As a final observation, all models agree that \texttt{1stFlrSF} is more important than 
\texttt{OverallCond} despite the fact that their min-max intervals of global importance intersect. 
This means that looking at min-max intervals of global feature importance (\ie the Model Class Reliance 
\citep{fisher2019all}) does not provide the full picture of the Rashomon Set.

\begin{figure*}[t]
     \centering
     \begin{subfigure}[c]{0.5\textwidth}
         \includegraphics[width=\linewidth]
         {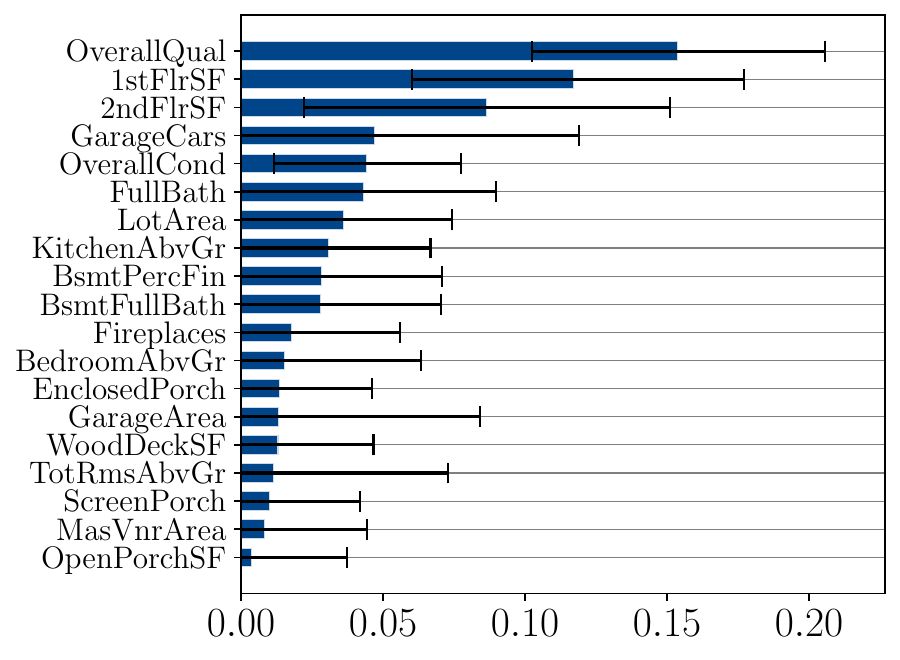}
     \end{subfigure}
     \hfill
     \begin{subfigure}[c]{0.49\textwidth}
     \raisebox{0.5cm}{
         \includegraphics[width=0.9\linewidth]
         {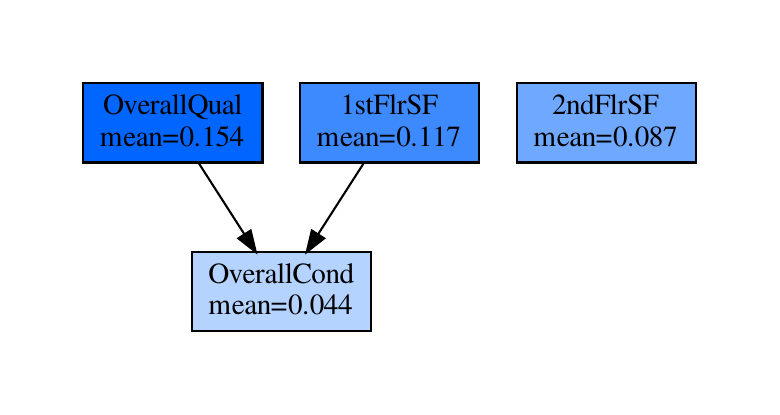}
     }
     \end{subfigure}
     \begin{subfigure}[c]{0.5\textwidth}
         \includegraphics[width=\linewidth]
         {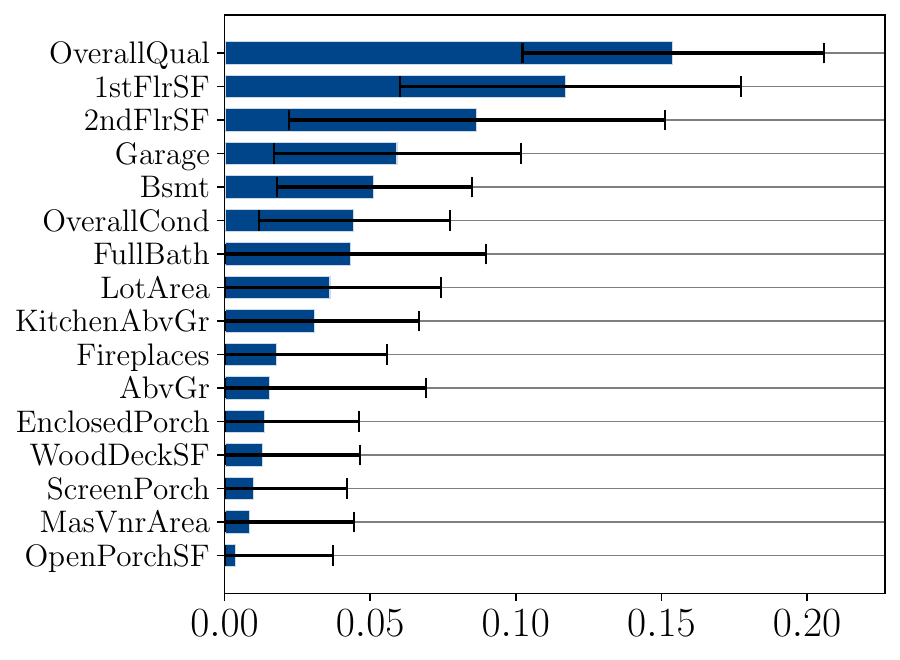}
     \end{subfigure}
     \begin{subfigure}[c]{0.49\textwidth}
     \raisebox{0.5cm}{
         \includegraphics[width=0.85\linewidth]
         {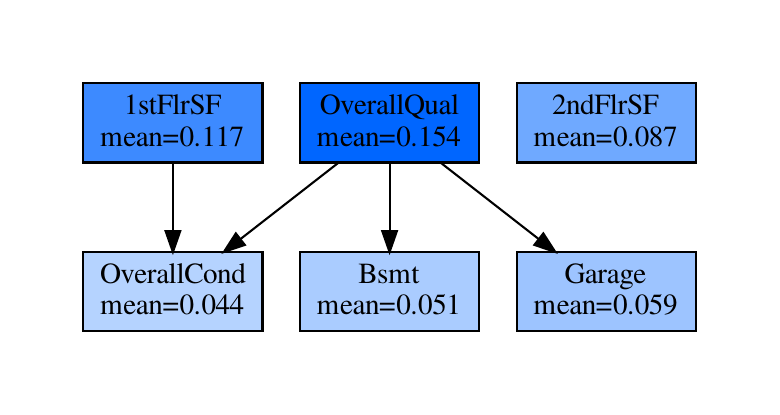}
     }
     \end{subfigure}
     \vspace{-.5cm}
     
    \caption{Global Feature Importance of the Kaggle-Houses dataset. 
    (Top) Without grouping. (Bottom) With grouping.}
    \label{fig:kaggle_explain_global}
\end{figure*}
\section{Application to Kernel Ridge Regression}\label{sec:kernel}
\subsection{Rashomon Set}
Let $k:\X\times \X\rightarrow \R_+$ be a symmetric positive definite kernel. Then such a kernel induces a 
functional space called a Reproducing-Kernel-Hilbert-Space (RKHS), which is actually the completion of the 
Pre-Hilbert space \citep{mohri2018foundations}
\begin{equation}
    \Hypo_k := \bigg\{h_{\bm{\alpha}}(\bm{x})=\sum_{j=1}^R \alpha_j k(\bm{x}, \bm{r}^{(j)}) \,\,\,\text{for}\,\,\,R\in \mathbb{N},\bm{\alpha}\in \R^R,\bm{r}^{(j)}\in \X\bigg\}
\end{equation}
endowed with the scalar product
\begin{equation}
    \langle \,k(\cdot, \bm{r}^{(i)})\,,\,k(\cdot, \bm{r}^{(j)})\,\rangle_{\Hypo_k} 
    := k(\bm{r}^{(i)}, \bm{r}^{(j)}),
\end{equation}
from which the terminology ``Reproducing-Kernel'' arises. The space $\Hypo_k$ is infinite-dimensional 
since it requires specifying any integer $R$ and any $R$ reference inputs $\bm{r}^{(j)}$. For simplicity, 
we shall fix the $R$ reference inputs in advance and store them in a dictionary $D:=\{\bm{r}^{(j)}\}_{j=1}^R$. 
We will then use the finite-dimensional approximation
\begin{equation}
    \Hypo^D_k := \bigg\{h_{\bm{\alpha}}(\bm{x})=\sum_{j=1}^R \alpha_j k(\bm{x}, \bm{r}^{(j)}) \,\,\,\text{for}\,\,\,\bm{\alpha}\in \R^R\bigg\}
\end{equation}
s.t. $\Hypo^D_k\subset \Hypo_k$ as was done in \citep{fisher2019all}. Since these spaces are still 
considerably expressive, it is common to apply regularization when learning models from them. 
From the Rashomon perspective, this implies studying the Rashomon Set
\begin{equation}
    \mathcal{R}(\Hypo^D_k,\epsilon) := 
        \bigg\{ h_{\bm{\alpha}} \in \mathcal{H}^D_k : 
       \emploss{D}{h_{\bm{\alpha}}} + \lambda \|h_{\bm{\alpha}}\|^2 \leq \epsilon \bigg\},
    \label{eq:regularized_rashomon_set}
\end{equation}
where $\lambda>0$ is a regularization hyper-parameter that is fine-tuned by cross-validation and 
$\|h_{\bm{\alpha}}\|^2:=\langle h_{\bm{\alpha}},h_{\bm{\alpha}}\rangle_{\Hypo_k}=\sum_{i,j=1}^R \alpha_i \alpha_j 
k(\bm{r}^{(i)}, \bm{r}^{(j)})$
is the functional norm induced by the scalar product on $\Hypo_k$.
We let $\bm{K}\in \R^{R\times R}$ be the symmetric positive semi-definite matrix of kernel 
evaluations on the dictionary $\bm{K}[i, j]=k(\bm{r}^{(i)}, \bm{r}^{(j)})$. The regularized least-square solution is
\begin{equation}
    \kleastsq = (\bm{K} + \lambda R \bm{I})^{-1}\bm{y}.
    \label{eq:sol_kernel}
\end{equation}

Given this notation, we can present the Rashomon Set of Kernel Ridge Regression.

\begin{definition}[Rashomon Set for Kernel Ridge Regression]
    Let $\Hypo^D_k$ be the space induced by the kernel $k$ and dictionary $D$, $\ell$ be the squared loss, $\lambda> 0$ be a regularization hyper-parameter, and $\kleastsq$ be the solution of the regularized least-square. If one
    uses the performance threshold $\epsilon \geq \emploss{D}{h_{\kleastsq}}+ \lambda \|h_{\kleastsq}\|^2$,
    then the Rashomon set $\mathcal{R}(\Hypo^D_k,\epsilon)$ consists of all models 
    $h_{\bm{\alpha}}$ s.t.
    \begin{equation}
        (\bm{\alpha}-\kleastsq)^T(\bm{K}/R + \lambda \bm{I})\bm{K}(\bm{\alpha}-\kleastsq)
        \leq \epsilon - \emploss{D}{h_{\kleastsq}}-\lambda \|h_{\kleastsq}\|^2.
    \end{equation}
    We see that the Rashomon Set is an ellipsoid in $\R^R$.
\end{definition}

The proof is mutatis mutandis like the proof for Ridge Regression in \cite{semenova2022existence} 
but with Kernel Ridge instead.

\subsection{Asserting Model Consensus}

Unlike the previous section, the model $h_{\bm{\alpha}}$ is no longer additive, and hence there is no universal 
way to assign a score $\phi_i$ to each input feature when explaining a gap in model predictions. Hence, we must 
rely on either SHAP or Integrated Gradient, which are two principled approaches for computing said scores. 
Because the exponential burden of Shapley values has not yet been solved for kernel methods, SHAP was not 
used and we instead employed the Integrated Gradient with a single baseline input $\bm{z}$. Henceforth, 
assuming the kernel is continuous and has continuous partial derivatives ($k\in \mathbb{C}^1(\X\times\X))$, 
we compute the IG as follows.
\begin{equation}
    \begin{aligned}
        \phi^\text{IG}_i(h_{\bm{\alpha}}, \bm{x}, \bm{z}) 
        &:= (x_i-z_i)\int_0^1\frac{\partial h_{\bm{\alpha}}}{\partial x_i}\bigg|_{t\myvec{x} + (1-t)\myvec{z}}dt\\
        &=\sum_{j=1}^R \alpha_j 
        \underbrace{\bigg[(x_i-z_i)\int_0^1\frac{\partial k(\cdot, \bm{r}^{(j)})}{\partial x_i}\bigg|_
        {t\myvec{x} + (1-t)\myvec{z}}dt \bigg]}_{\phi_{ij}}
        =\sum_{j=1}^R \alpha_i \phi_{ij},
    \end{aligned}
    \label{eq:kernel_IG}
\end{equation}
which is  a linear function of the coefficients $\bm{\alpha}$. Consequently, asserting a consensus on IG feature
attributions will again amount to optimizing a linear function over an ellipsoid so we can leverage results from 
the previous section. The only additional step required for Kernel Ridge is to pre-compute the path integrals 
$\phi_{ij}$ with common quadrature methods.

Similarly to Additive Models in \textbf{Section \ref{sec:additive:consensus:GFI}}, one can combine 
local feature attributions into global feature importance $\bm{\Phi}^{[2]}$ which are a quadratic form 
of the $\bm{\alpha}$ coefficients. Again, asserting a consensus over the Rashomon Set would require 
solving a TRS.

\subsection{Criminal Recidivism Prediction}\label{sec:experiments:subsec:COMPAS}

COMPAS is a proprietary model currently employed in the United States to predict the risk of recidivism from 
individuals that were recently arrested. These risks are encoded as integers going from 1 (low-risk) to 10 
(high-risk). The use of this automated tool in the justice system is driven by the promise of providing 
objective information to judges based on empirical data, thus circumventing human biases. Still, the strong 
reliance of models on historical data means they can reproduce/perpetuate past injustices. To test such claims,
ProPublica has collected several thousands of COMPAS scores from 2013-2014 in the Florida Broward County 
\citep{compas}. In the resulting article, several pairs of Caucasian and African-American defendants are 
presented along with their COMPAS scores, the former often being lower than the latter despite the Caucasian 
defendant having a longer criminal history. These examples of pairs along with the subsequent analysis from 
the article seem to imply that the proprietary model depends on race. However, the methodology of ProPublica 
has been heavily criticized alongside the claim that COMPAS depends explicitly on race \citep{rudin2018age}.
Hence, there may exist alternative explanations besides race for the discrepancy between scores, so it is 
pertinent to study the local feature attributions of the whole Rashomon Set of reasonable models 
when predicting COMPAS scores.

To analyze the dependencies of risk scores on the various features, we repeated the experiments of 
\citet{fisher2019all} where a Kernel Ridge Regression model was fitted directly on the 1-10 scores 
from the ProPublica dataset. The same features were employed while adding two additional ones related 
to juvenile misdemeanors and felonies. The dataset was split in train and test sets with ratios of 
0.8 and 0.2. The training samples were used to define the dictionary of reference inputs $D$. 
We utilized the polynomial kernel $k(\bm{x},\bm{x'}) = (\gamma \langle \bm{x},\bm{x}'\rangle + 1)^p$ 
with degree $p=3$ and the Gaussian kernel $k(\bm{x},\bm{x'}) = \text{exp}(-\gamma \|\bm{x}-\bm{x}'\|^2)$. 
The kernel scale hyper-parameter $\gamma$ and the regularization factor $\lambda$ were fine-tuned with 
5-fold cross-validation on the training set, see the results for Gaussian Kernels in 
Figure \ref{fig:COMPAS_result}~(a). Similar results were obtained with Polynomial Kernels. The test set 
RMSE of the final model was 2.11 for Gaussian Kernels and 2.12 for Polynomial Kernels. We note that 
the performances are worse than \citet{fisher2019all} because, unlike them, we predict the recidivism 
risk scores and not the risk scores for \emph{violent} recidivism, which could be easier to predict. 
In this paper, we decided to study the recidivism risk scores instead since these are the ones that were 
actually discussed in the ProPublica article.

\begin{figure*}[t]
    \centering
    \begin{subfigure}[b]{0.48\textwidth}
        \hspace{1pt}
        \includegraphics[width=\linewidth]
        {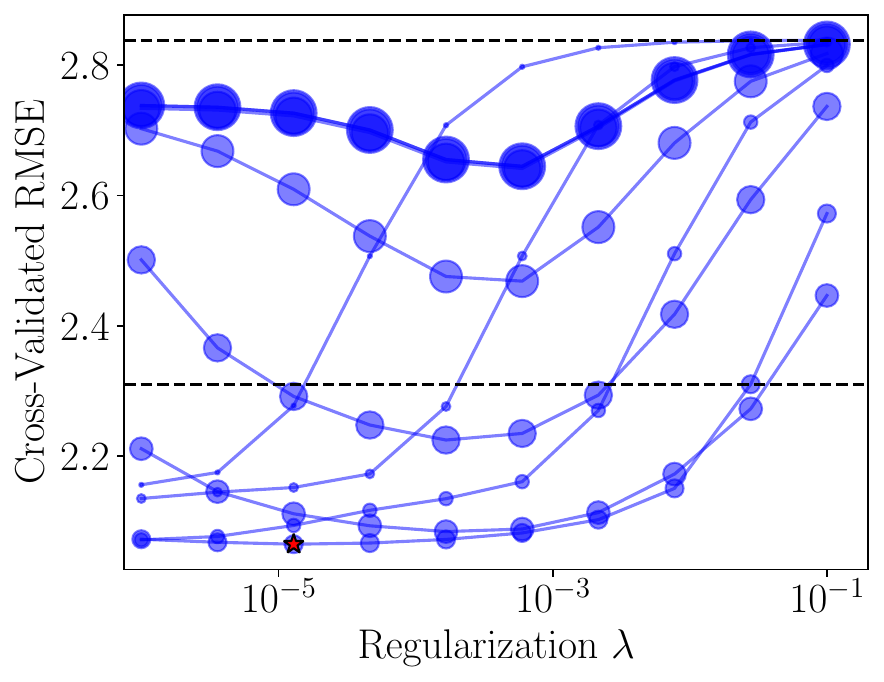}
       \caption{Tuning
       of $\gamma$ (dot sizes) and $\lambda$ with Gaussian kernels.
       The top horizontal line shows the error of the predictor returning the 
       mean, while the bottom line shows the error of a Random Forest 
       with default hyperparameters.}
    \end{subfigure}
    \hfill
    \begin{subfigure}[b]{0.505\textwidth}
    \hspace{5pt}
    \includegraphics[width=\linewidth]
    {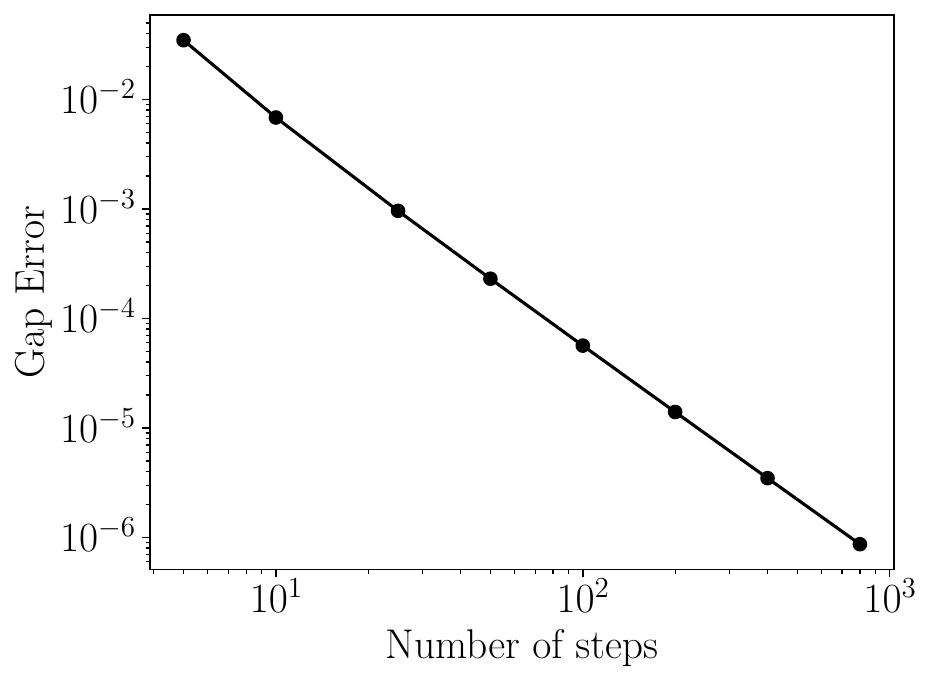}
        \caption{Convergence 
        of the Gap Error w.r.t the number of steps in the quadrature. 
        We observe a second-order convergence. Hence, augmenting the number of steps 
        by a factor 10 reduced the error by a factor 100.}
    \end{subfigure}
    \caption{}
    \label{fig:COMPAS_result}
\end{figure*}

\begin{table}[t]
    \centering
    \begin{tabular}{ccccccc}
    \toprule
    Input & Name & Score & Race & Age & Priors & Charge \\
    \midrule
     $\bm{x}$ & Robert Cannon & 6 & African-American & 22 & 0 & Misdemeanor \\
     $\bm{z}$ &  James Rivelli & 3& Caucasian & 54 & 3 & Felony \\
    \bottomrule
    \end{tabular}
    \caption{Comparison of the COMPAS scores of two individuals.}
    \label{tab:comparison_COMPAS}
\end{table}

After fitting the models, we identified a pair of Caucasian/African-American individuals who were highlighted 
in the ProPublica piece and applied our explainability framework on them. More specifically, we compared 
Robert Cannon and James Rivelli, see Table \ref{tab:comparison_COMPAS}. James Rivelli is a 54-year-old 
Caucasian man who was arrested for shoplifting. Despite having a criminal record with three priors, 
he was assigned a low COMPAS score. In contrast, Robert Cannon, a 22-year-old African-American charged 
with petit theft, was assigned a high risk of recidivism. Letting Robert be the input of $\bm{x}$ and 
James be the input $\bm{z}$, we observe that the differences in scores are also present for the 
Kernel Ridge models: $h_{\kleastsq}(\bm{x})=4.9$ and $h_{\kleastsq}(\bm{z})=2.5$ for Gaussian Kernels, 
and $h_{\kleastsq}(\bm{x})=4.9$ and $h_{\kleastsq}(\bm{z})=2.4$ for Polynomial Kernels. Therefore, 
we have a prediction gap $G(h_{\kleastsq}, \bm{x})=h_{\kleastsq}(\bm{x}) - h_{\kleastsq}(\bm{z})$ 
that is positive.

Given the historical racism in the United States, it is very tempting to look at these two individuals 
and say that Robert Cannon is predicted to have a higher risk ``because of his race''. Still, there 
may exist a diversity of alternative explanations for this discrepancy, which we can study by exploring 
the Rashomon Set of our Kernel Ridge models. The Integrated Gradient was employed using Robert as the 
input of interest $\bm{x}$ and James as the reference input $\bm{z}$ to obtain feature attributions. 
Since computing the IG feature attributions requires estimating the integrals of Equation \ref{eq:kernel_IG} 
with quadratures, we ended up with estimates $\widehat{\bm{\phi}}^\text{IG}(h_{\kleastsq}, \bm{x})$ of 
the real attributions $\bm{\phi}^\text{IG}(h_{\kleastsq}, \bm{x})$. We characterized the estimation error 
of this discretization by reporting the Gap Error
\begin{equation}
    \bigg|\sum_{i=1}^d \widehat{\phi}_i(h_{\kleastsq}, \bm{x}) 
    - G(h_{\kleastsq}, \bm{x})\bigg|,
    \label{eq:gap_error}
\end{equation}
and used it as a proxy of how well $\widehat{\bm{\phi}}(h_{\kleastsq}, \bm{x})$ approximates 
$\bm{\phi}(h_{\kleastsq}, \bm{x})$. By simplicity, the Trapezoid quadrature was implemented, 
see Figure \ref{fig:COMPAS_result}~(b) for the convergence of the Gap Error as the number of steps 
in the quadrature increases. We note that, as expected, the quadrature converges to the second order. 
For the remainder of the analysis, we have employed quadratures with 1000 steps.

Now, can we use a capture bound to set the tolerance $\epsilon$? Unfortunately, the empirical loss 
$\emploss{D}{h_{\bm{\alpha}}} + \lambda \|h_{\bm{\alpha}}\|^2$ involves regularization so we cannot 
guarantee that the Rashomon Set (cf. Equation \ref{eq:regularized_rashomon_set}) contains $h^\star$ 
unless we make a strong (unverifiable) smoothness assumption $\|h^\star\|^2\leq B$.
Without knowledge of $B$, we instead resort to a relative increase heuristic 
$\epsilon = 1.01\times \big[\emploss{D}{h_{\kleastsq}}+\lambda\|h_{\kleastsq}\|^2\big]\approx 4.23$ 
(an increase of $\epsilon_\text{rel}=1\%$ of the minimum objective value $4.19$).
Unlike \textbf{Sections \ref{sec:additive:experiments} \& \ref{sec:random_forest:experiments}}, 
we did not compute a sensitivity analysis w.r.t. changes in $\epsilon$. Rather, by setting it to 
a reasonably small value, we only wish to prove the existence of competing models that disagree 
on their explanation for the discrepancy between James and Robert.

\begin{figure}[t]
     \centering
     \begin{subfigure}[c]{0.58\textwidth}
         \includegraphics[width=\linewidth]
         {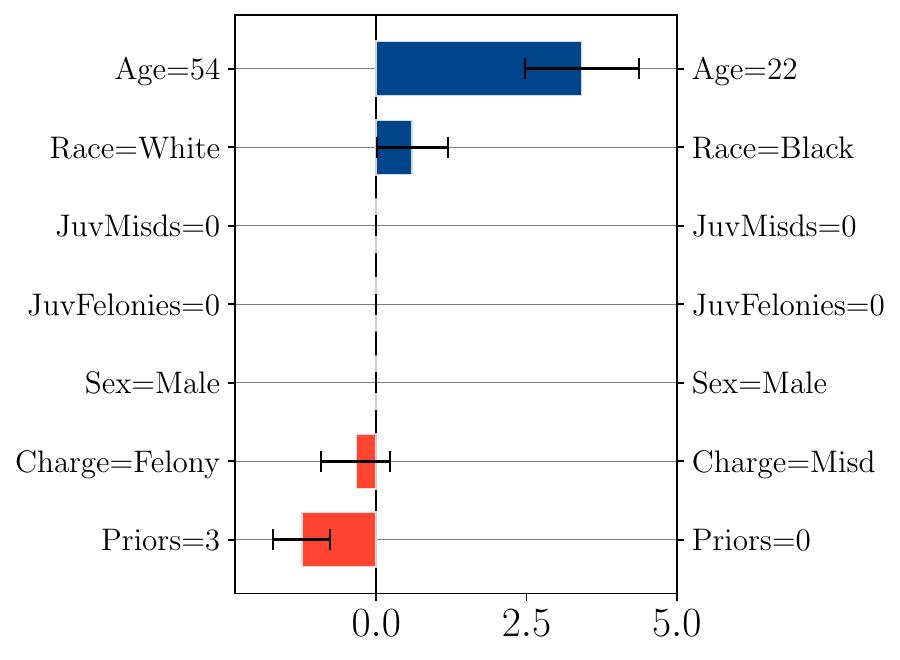}
     \end{subfigure}
     \hfill
     \begin{subfigure}[c]{0.41\textwidth}
     \raisebox{0.5cm}{
     \hspace{0.95cm}
         \includegraphics[width=0.65\linewidth]
         {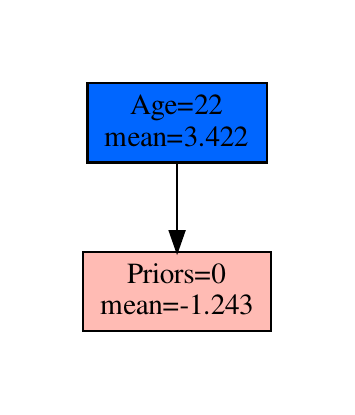}
     }
     \end{subfigure}
     \begin{subfigure}[c]{0.58\textwidth}
         \includegraphics[width=\linewidth]
         {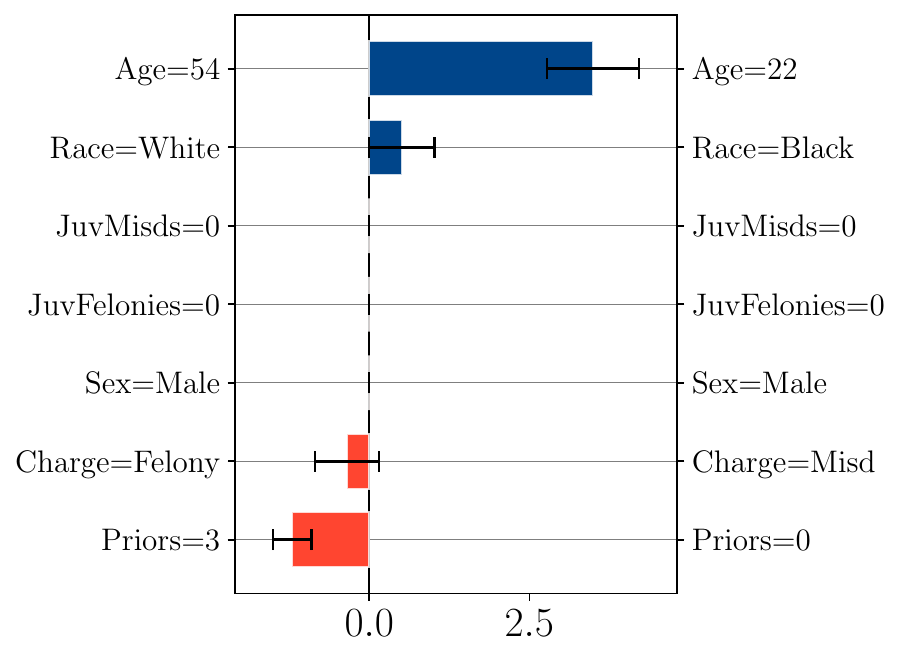}
     \end{subfigure}
     \hfill
     \begin{subfigure}[c]{0.41\textwidth}
     \raisebox{0.5cm}{
     \hspace{0.95cm}
         \includegraphics[width=0.65\linewidth]
         {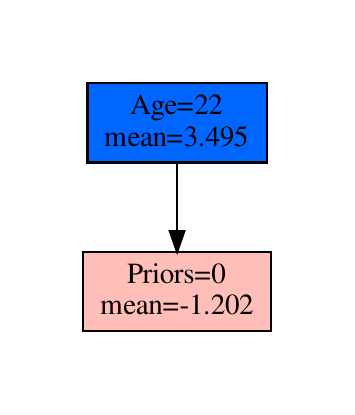}
     }
     \end{subfigure}
    \caption{
    Local feature attribution comparing Robert Cannon to James Rivelli.
    (Top) Gaussian Kernels. (Bottom) Polynomial Kernels. The features on the left of the bar 
    charts represent James while the values on the right represent Robert.}
    \label{fig:compas_explain_instances}
\end{figure}

Figure \ref{fig:compas_explain_instances} presents the local feature attributions across the Rashomon Sets
$\mathcal{R}(\Hypo^D_k,4.23)$ of Gaussian and Polynomial Kernels. Since the results are consistent 
across the two types of Kernels, we will only discuss Gaussian Kernels. Inspecting the top bar plot, 
we see that, according to the Integrated Gradient of the empirical loss minimizer, 
plotted as the blue/red bars, the features $\texttt{Age=22}$ and $\texttt{Race=Black}$ have 
positive attributions while the features $\texttt{Charge=Misdemeanor}$ and $\texttt{Prior=0}$
have negative attributions. This suggests that one of the possible explanations for the 
high risk of Robert relative to James is racial discrimination toward African-Americans. However, 
when we additionally consider the opinion of models with slightly worst performance on the training data, 
some of our previous statements on feature attribution cease to hold. Importantly, there exists a
competing model $h'\in\mathcal{R}(\Hypo^D_k,4.23)$ that yield a null attribution to the feature 
\texttt{Race=Black}, and whose test error is not significantly worse than $h_{\kleastsq}$ according 
to a paired Student-$t$ test with $\delta=0.05$. Therefore, there are reasonable explanations for 
the disparity between Robert and James, that do not rely on Robert being African-American.
Even when considering the whole Rashomon Set, there remain statements on which models reach consensus. 
Notably, the attribution of the feature \texttt{Age=22} remains positive and has maximum importance. 
% The only way for age to be a negligible factor in explaining 
% the difference in scores between Robert and James would be to employ models with worst empirical loss 
% than what we considered.

These observations are concordant with previous work of \citet{rudin2018age} which hypothesizes that 
COMPAS depends strongly on age and (at most) weakly on race. Nonetheless, our local analysis on James 
and Robert must not be taken as absolute facts about the proprietary model COMPAS. This is because we 
do not have access to the model and we are surrogating it with Kernel Ridge models fitted on 7 features. 
The original COMPAS model, on the contrary, takes 137 different factors into consideration to produce a 
score \citep{rudin2018age}. Our analysis is more of a proof of concept that our explainability framework 
can make sense of the local feature attributions of competing models and that it can highlight the 
diversity of explanations for the discrepancies between two individuals.

\newpage
\section{Application to Random Forests}\label{sec:random_forest}
\subsection{Rashomon Set}

A Random Forest (RF) is an ensemble of independently trained decision trees whose predictions are 
averaged to yield the final predictions \citep{breiman2001random}. To increase diversity, each tree is 
trained on a different bootstrap sample of the original dataset and each inner split is done among a 
random subset of features. We let $s$ represent the seed encoding all pseudo-random processes in the 
training of a single tree $t_s$. If $\mathcal{S}$ is a distribution over all possible seeds on a computer, 
the theoretical definition of a RF is
\begin{equation}
    h(\myvec{x}) = \E_{s\sim \mathcal{S}}[ \tree_s(\myvec{x})].
    \label{eq:rf}
\end{equation}
Given the finite representation of numbers on a computer, we can assume that the set of possible seeds 
is finite and of size $M$. Then, a reasonable choice of distribution over seeds is the uniform over 
$M$ seeds \ie $\mathcal{S} = U(\{1, 2\,\ldots,M\})$. In practice, the expectation $\E_{s\sim \mathcal{S}}$ 
has to be approximated using Monte-Carlo sampling. Given $m < M$, we subsample $m$ seeds uniformly 
at random $S \sim \mathcal{S}^m$, and return the sample average as our estimate of the RF
\begin{equation}
    h_S(\myvec{x}) =\frac{1}{m} \sum_{s\in S} \tree_s(\myvec{x}).
\end{equation}
By the weak law of large numbers, the estimated RF should converge to the true RF 
(cf. Equation \ref{eq:rf}) as $m$ increases. Since sampling $m$ seeds out of $M$ with/without replacement 
assigns a non-zero probability to any subset of $m$ seeds, we conceptualize the space of all 
\textbf{possible} RFs as the collection of all subsets of trees.
\begin{definition}
    Given a large set $\mathcal{T}=\{\tree_s\}_{s=1}^M$ of $M$ trees trained with $M$ seeds, 
    the set of all possible RFs of $m$ trees is
    \begin{equation}
        \Hypo_m := \bigg\{\frac{1}{m}\sum_{\tree \in T} \tree \,\,:\,\, T \subseteq \mathcal{T} \,\,\,\,\text{and}\,\,\,\, |T|= m\bigg\},
    \end{equation}
    i.e. all averages of subsets of $m$ trees from $\mathcal{T}$. Moreover, we define
    $\Hypo_{m:}:= \cup_{k=m}^M \Hypo_k$ as all RFs with least $m$ trees. We interpret 
    $\Hypo_{1:}$ as the set of all possible RF that can ever appear in practice on a given 
    dataset, regardless of the choice of $m$.
\end{definition}

\begin{figure}[t]
    \centering
    \includegraphics[width=\linewidth]
    {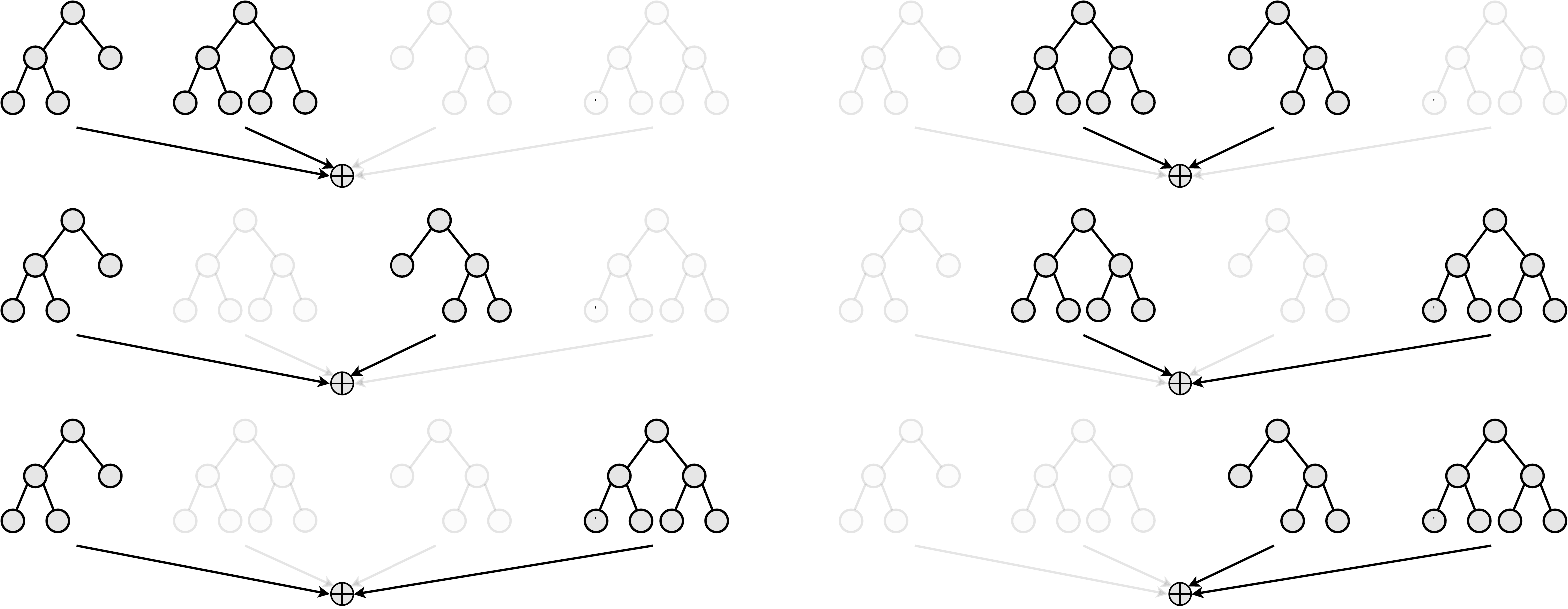}
    \caption{Example of the space $\Hypo_{2}$ representing all the 
    possible groupings of $2$ decision trees out of $M=4$.}
    \label{fig:rf_Hm}
\end{figure}

Figure \ref{fig:rf_Hm} illustrates an example of space $\Hypo_m$ which accentuates their combinatoric nature.
We also note the monotonic relation $m< m' \Rightarrow \Hypo_{m:} \supset \Hypo_{m':}\,\,$.
Since we interpret $\Hypo_{1:}$ as the set of all possible RFs that can ever appear in practice on a dataset, 
we aim to characterize its Rashomon Set $\mathcal{R}(\Hypo_{1:},\epsilon)$. Such a Rashomon Set cannot be 
explicitly represented because if its exponential size ($|\Hypo_{1:}|=2^M-1)$. Still, we will see that studying the 
space $\Hypo_{m:}$ for a carefully chosen $m$ can help us characterize a large subset of the Rashomon Set. The reason we 
want to work with hypotheses $\Hypo_{m:}$ is that they have a desirable property: optimizing a linear functional over 
them is tractable, as highlighted by the following proposition.
\begin{proposition}
    Let $\mathcal{T}:= \{\tree_s\}_{s=1}^M$ be a set of $M$ trees, $\Hypo_{m:}$ be the 
    set of all RFs with at least $m$ trees from $\mathcal{T}$, and
    $\phi:\Hypo_{m:}\rightarrow \R$ be a linear functional, then $\min_{h\in \Hypo_{m:}} \phi(h)$
    amounts to averaging the $m$ smallest values of 
    $\phi(\tree_s)$ for $s=1,2,\ldots M$.
\label{prop:rf_optim}
\end{proposition}
The proof of this proposition is presented in \textbf{Appendix \ref{app:proofs:random_forest}}.
Examples of linear functionals $\phi:\Hypo_{m:}\rightarrow \R$ include the model prediction at fixed input $h(\bm{x})$ 
and the SHAP feature attribution which we can compute efficiently with TreeSHAP \citep{lundberg2020local}. 

At this point, we assume that the desired tolerance on error $\epsilon$ has been fixed and so we wish 
to identify a value $m(\epsilon)$ that guarantees that 
$\Hypo_{m(\epsilon):}\subseteq \mathcal{R}(\Hypo_{1:},\epsilon)$, or equivalently, that $\max_{h\in\Hypo_{m(\epsilon):}}\emploss{S}{h}\leq \epsilon$.
This value $m(\epsilon)$ should be as small as possible so that the space $\Hypo_{m(\epsilon):}$ is as 
large as possible. With this goal in mind, we restrict ourselves to losses $\ell(y', y)$ that are 
monotonically increasing w.r.t $|y'-y|$. This includes the $0\,\minus1$ loss and the squared loss for 
example. Such losses are of interest because $\max_{y'\in \Y'} \ell(y', y)=\max\{\ell(\min_{y'\in \Y'} y', y)\,,\, \ell(\max_{y'\in \Y'} y', y)\}$ 
for any set $\Y'$, meaning that the worst loss on a point must be attained by either of the two most extreme predictions at that point.
Remembering that model predictions are linear functionals of the trees, 
\textbf{Proposition \ref{prop:rf_optim}} can be used to efficiently identify the min/max predictions 
at any input. Therefore, it makes sense to define the upper bound
\begin{equation}
\begin{aligned}
    \max_{h\in\Hypo_{m:}}\emploss{S}{h}
    \leq& \frac{1}{N}\sum_{i=1}^N 
    \max_{h \in\Hypo_{m:}}\ell(h(\myvec{x}^{(i)}), y^{(i)}),\\
    =&\frac{1}{N}\sum_{i=1}^N \max\bigg\{\ell\big(\min_{h\in \Hypo_{m:}} h(\myvec{x}^{(i)}),y^{(i)}\big)\,,\, 
    \ell\big(\max_{h\in \Hypo_{m:}} h(\myvec{x}^{(i)}),y^{(i)}\big)\bigg\}:= \epsilon^+(m),
\end{aligned}
\label{eq:epsilon_+}
\end{equation} 
which can be
computed efficiently at any $m\leq M$ in time $\mathcal{O}(NM\log M)$. Because of the scalability of this 
process w.r.t $M$, the total number of tree $M$ must be reasonable, but still large enough so that 
$\mathcal{T}=\{t_s\}_{s=1}^M$ is representative of all trees that would be produced with all possible 
seeds on a computer. We will see in the experiments of \textbf{Section \ref{sec:random_forest:experiments}} that 
setting $M=1000$ can be representative of all trees fitted on real-world data.

Now, given an absolute tolerance $\epsilon$ on the empirical loss, we search for the smallest 
number of trees $m$ we can keep while ensuring that $\epsilon^+(m)\leq \epsilon$
\begin{equation}
    m(\epsilon) := \min\{m\in [M] :
    \epsilon^+(m)\leq\epsilon\}.
\label{eq:choose_m}
\end{equation}

\begin{figure}[t]
    \centering
    \begin{tikzpicture}
\tikzset{>=latex}
% Axes
\draw[<->] (0,4.5) node[anchor=south east,rotate=90] {Empirical Loss}
-- (0,0) -- (9, 0) node[anchor=north west] {$m$};

\foreach \x in {1,...,8}
{
   \draw (\x, -0.1) node[anchor=north] {\x} -- (\x, 0.1);
   \node at (\x+0.25, -1.25) {$\Hypo_{\x:}\supset$};
}
 \node at (9, -1.25) {...};

% Function
\draw[red,line width=0.025cm] 
	(8,0.5) node[circle,draw,fill=red,scale=0.25] {} -- (8,0.75) -- 
	(7,0.75) node[circle,draw,fill=red,scale=0.25] {} -- (7,1) -- 
	(6,1) node[circle,draw,fill=red,scale=0.25] {} -- (6,1.5) -- 
	(5,1.5) node[circle,draw,fill=red,scale=0.25] {} -- (5,2.25) -- 
	(4,2.25) node[circle,draw,fill=red,scale=0.25] {} -- (4,3) -- 
	(3,3) node[circle,draw,fill=red,scale=0.25] {} -- (3,4.25) -- 
	(2,4.25) node[circle,draw,fill=red,scale=0.25] {} -- (2,4.5) -- 
	(1,4.5) node[circle,draw,fill=red,scale=0.25] {};
\node at (4,3.5) {$\textcolor{red}{\epsilon^+(m)}$};

% Intersection
\draw[dashed] (0,1.6) node[anchor=east] {$\epsilon$} -- (8,1.6);
\draw[dashed] (5,1.5) -- (5,0);
\node at (5, -0.75) {$m(\epsilon)$};

\end{tikzpicture}
    \caption{Choosing $m$ based on the error tolerance $\epsilon$.}
    \label{fig:epsilon_plus}
\end{figure}

The intuition behind the computation of $m(\epsilon)$ is presented in Figure~\ref{fig:epsilon_plus}.
Since setting $m=m(\epsilon)$ guarantees that $\max_{h\in\Hypo_{m:}}\emploss{S}{h}\leq \epsilon^+(m)\leq \epsilon$, 
we have $\Hypo_{m(\epsilon):}\subseteq \mathcal{R}(\Hypo_{1:},\epsilon)$. Hence, we are going to employ 
$\Hypo_{m(\epsilon):}$ as an under-estimate of the Rashomon Set over which we can efficiently optimize 
linear functionals such as model predictions or the SHAP local feature attributions.

We end this subsection by presenting in detail the computation of $\epsilon^+(m)$ on a toy example. 
We designed a regression task where the input follows a $\mathcal{N}(0,1)$ Gaussian and the output $y$ is a 
quadratic function $x^2$ plus some noise of amplitude $0.9$. A total of $M=1000$ different seed values were 
used to independently generate 1000 decision trees. Figure~\ref{fig:rf_toy}~(a) shows the upper bound 
$\epsilon^+(m)$ of any RF containing at least $m$ trees. Given a threshold on the RMSE of $\epsilon=1$, 
the smallest $m$ we can safely consider is $m(\epsilon)=691$. Hence, we suggest employing the set 
$\Hypo_{691:}$ as a subset of $\mathcal{R}(\Hypo_{1:}, 1)$. Figure~\ref{fig:rf_toy}~(b) presents the 
minimum and maximum predictions $\min_{h\in \Hypo_{691:}} h(x)$ and $\max_{h\in \Hypo_{691:}} h(x)$ 
at various values of $x$. We see that the min-max prediction intervals are wider in low-data density 
regions near the boundaries. This means that there is more disagreement among individual trees on these 
points. Such an observation makes sense because each tree is fitted on a bootstrap sample of the dataset 
and therefore some trees have never seen the boundary points.

\begin{figure}[t]
     \centering
     \begin{subfigure}[b]{0.515\textwidth}
        \includegraphics[width=\linewidth]
       {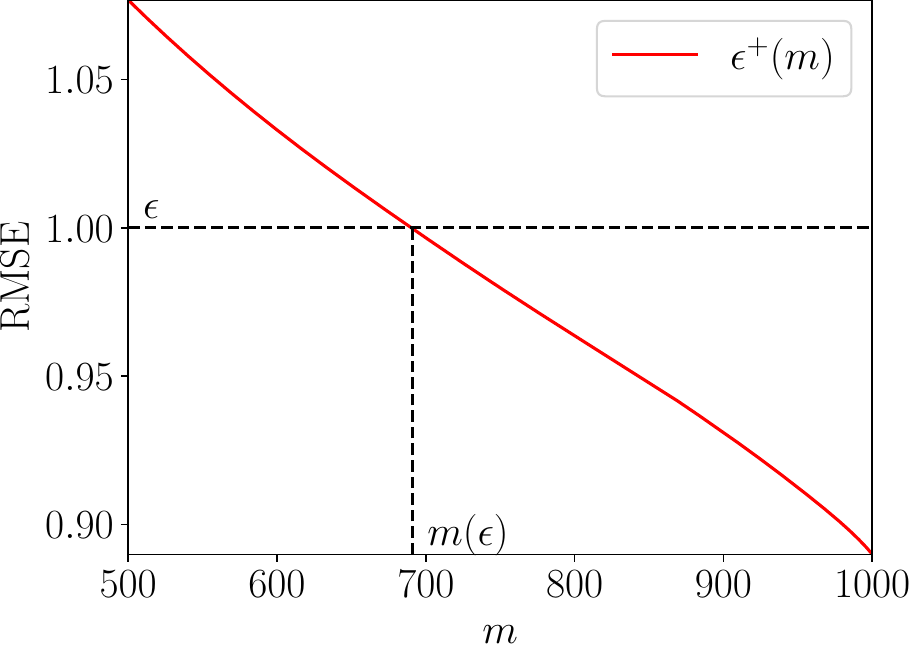}
       \caption{Performance bound $\epsilon^+(m)$. Given a RMSE tolerance of 
       $\epsilon=1$, the smallest $m$ we can safely consider is $m(\epsilon)=691$.}
    \end{subfigure}
    \hfill
    \begin{subfigure}[b]{0.475\textwidth}
        \includegraphics[width=\linewidth]
       {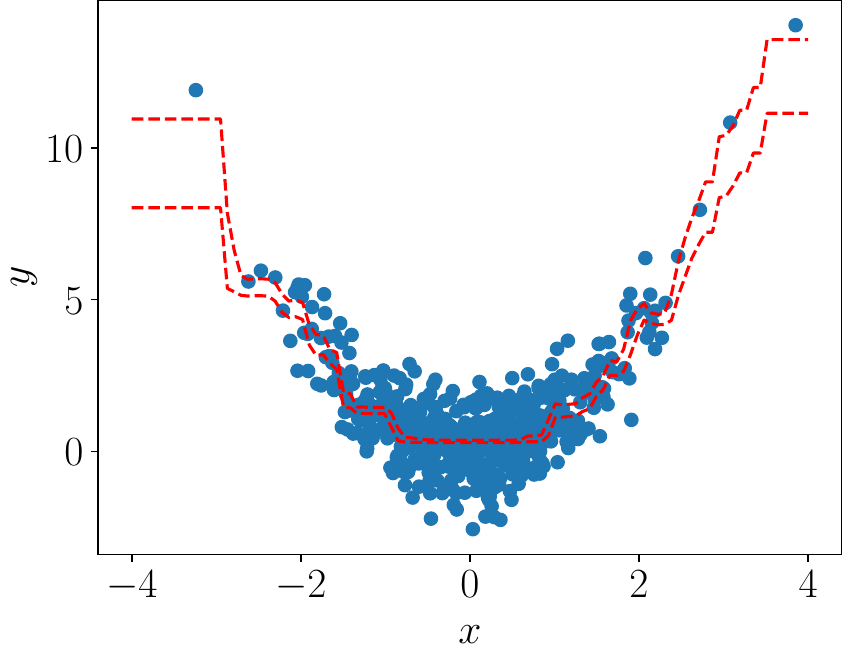}
        \caption{Toy regression data. The min-max predictions over the hypothesis 
        space $\Hypo_{691:}$ are shown as red lines.}
    \end{subfigure}
    \caption{}
    \label{fig:rf_toy}
\end{figure}

\subsection{Asserting Model Consensus}

\subsubsection{Local Feature Attribution}

We discuss how to assert model consensus on local feature attributions statements at any level of tolerance $\epsilon$.
Given an error tolerance $\epsilon$, we set $m$ to $m(\epsilon)$, and assert the consensus on $\Hypo_{m:}$ via optimization problems (cf. \textbf{Definition 
\ref{def:computing_local_consensus_optim}}) that we solve efficiently with \textbf{Proposition \ref{prop:rf_optim}}. For example, to 
compute $\min_{h\in \Hypo_{m:}} \phi_i(h, \myvec{x})$, we calculate the vector of feature attributions of all trees 
$[\phi_i(\tree_1, \myvec{x}),\phi_i(\tree_2, \myvec{x}),\ldots, \phi_i(\tree_M, \myvec{x})]^T$ with TreeSHAP, then we sort it and average 
its $m$ smallest values. The overall complexity of this procedure w.r.t $M$ is $\mathcal{O}(M \log M)$.

\subsubsection{Global Feature Importance}

Asserting model consensus on global feature importance statements is a lot more complicated 
since the functionals $\bm{\Phi}^{[1]},\bm{\Phi}^{[2]}$ are not linear w.r.t the model. 
Thus, we cannot leverage directly apply \textbf{Proposition \ref{prop:rf_optim}}. We refer to 
\textbf{Appendix \ref{app:optim:combinatorial}} for the full details of how we deal with 
global feature importance. In short, we employ the functional $\bm{\Phi}^{[1]}$ and create 
an ensemble $E$ containing
\begin{enumerate}
    \item Approximates of $\argminormax_h\Phi_i^{[1]}(h)$ for $1\leq i\leq d$.
    \item Approximates of $\argminormax_h\Phi_i^{[1]}(h)-\Phi_j^{[1]}(h)$ for $1\leq i<j\leq d$.
\end{enumerate}
After, we assert a consensus among all models in $E\subset\Hypo_{m(\epsilon):}$
leading to the partial order
\begin{equation}
 i \,\,\widehat{\preceq_\epsilon}\,\, j \iff 
    \forall\,h \in E\  \,\,\,\Phi_i(h) \leq \Phi_j(h).
\end{equation}
We consequently underestimate the diversity of our models, but the resulting partial order 
of global importance is guaranteed to be transitive.

\subsection{Income Prediction}\label{sec:random_forest:experiments}

The Adult-Income dataset available on the UCI
repository\footnote{\tiny\url{https://archive.ics.uci.edu/ml/datasets/adult}} contains the census data of 48,842 
individuals collected in 1994. It consists of a binary classification task with the goal of predicting whether or not a 
person makes more ($y=1$) or less ($y=0$) than 50k USD per year based on 14 attributes. Out of all these features, 
we removed \texttt{fnlwgt} because we do not fully understand what it represents and \texttt{native-country} because it 
is a categorical feature with very high cardinality. We were finally left with five numerical features and seven 
one-hot-encoded categorical ones. After encoding, we were left with a data matrix of 40 columns. The data was 
split into train and test sets with ratios 0.8 and 0.2 respectively. The training set was used to obtain the set 
$\mathcal{T}$ of $M$ iid trees.
For the model, we utilized Scikit-Learn's \texttt{RandomForestClassifiers} whose hyperparameters were tuned with a 100 
steps random search and 5-fold cross-validation. Then, we trained $M=1000$ trees in order to generate a set 
$\mathcal{T}$. The training was actually repeated 5 times so that we ended up with 5 distinct sets
of 1000 trees $\mathcal{T}_i$ with $i=1,2,\ldots,5$. We do not expect practitioners to fit several sets $\mathcal{T}_i$ 
when applying our methodology. This was done to verify our assumption that $\mathcal{T}$ is 
representative of all trees trained with bootstrapped data and random splits.

After obtaining large collections of trees, we estimated the Rashomon Set containing all RFs
that perform well on the test set. The loss employed was the 0-1 loss meaning the Rashomon Set 
contains all models with a Misclassification Rate below some threshold $\epsilon$. The tolerance
$\epsilon$ was set via the capture bound of \textbf{Proposition \ref{prop:capture_bound_classif}} using 
$h_\text{ref}=1/M\sum_{s=1}^M t_s$ as the reference model. This proposition is applicable since we compute the
Rashomon Set on test data that is independent of the hypothesis $h_\text{ref}$ which was fitted on 
training data. Using a confidence $\delta=1\%$, the proposition led to an error tolerance 
$\epsilon=\sqrt{\minus2\log(1\%)/N}+\emploss{S}{h_\text{ref}}\approx 3\% +\emploss{S}{h_\text{ref}}$.
By computing the upper bound $\epsilon^+(m)$ on test samples, we set the minimum number of trees
$m(\epsilon)=815$, see Figure \ref{fig:adult_experiment}~(a). 
At this tolerance level, the sign of the gap is consistent for $90.8\%$ of the individuals. 
Therefore, under-specification prohibits us from explaining one-tenth of the data.
We refer to \textbf{Appendix \ref{app:gap}} for how we deal with those unexplainable instances.

\begin{figure*}[t]
     \centering
     \begin{subfigure}[b]{0.4975\textwidth}
        \hspace{1pt}
        \includegraphics[width=\linewidth]
       {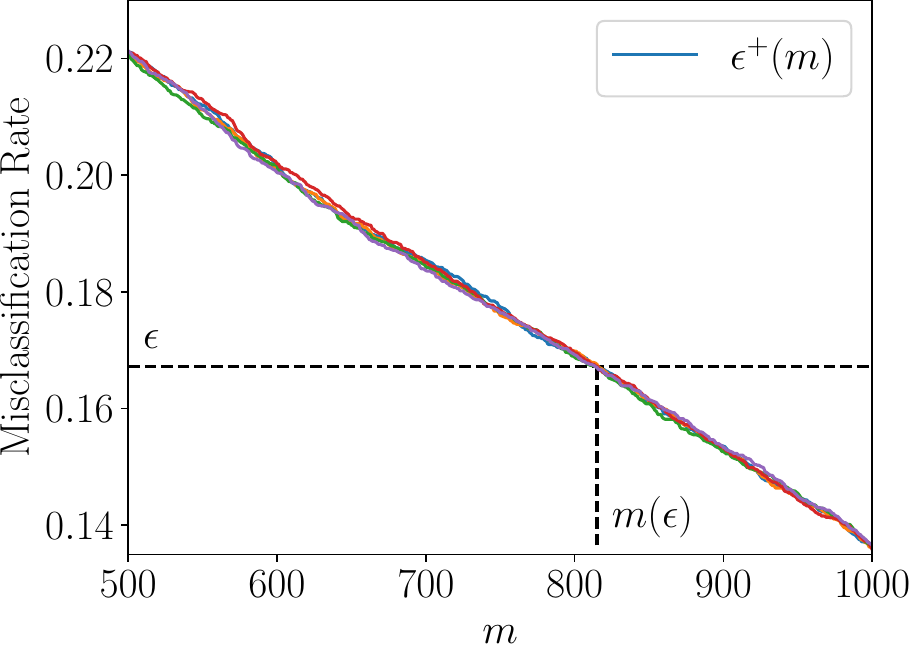}
       \caption{}
    \end{subfigure}
    \hfill
    \begin{subfigure}[b]{0.4825\textwidth}
    \hspace{5pt}
       \includegraphics[width=\linewidth]
        {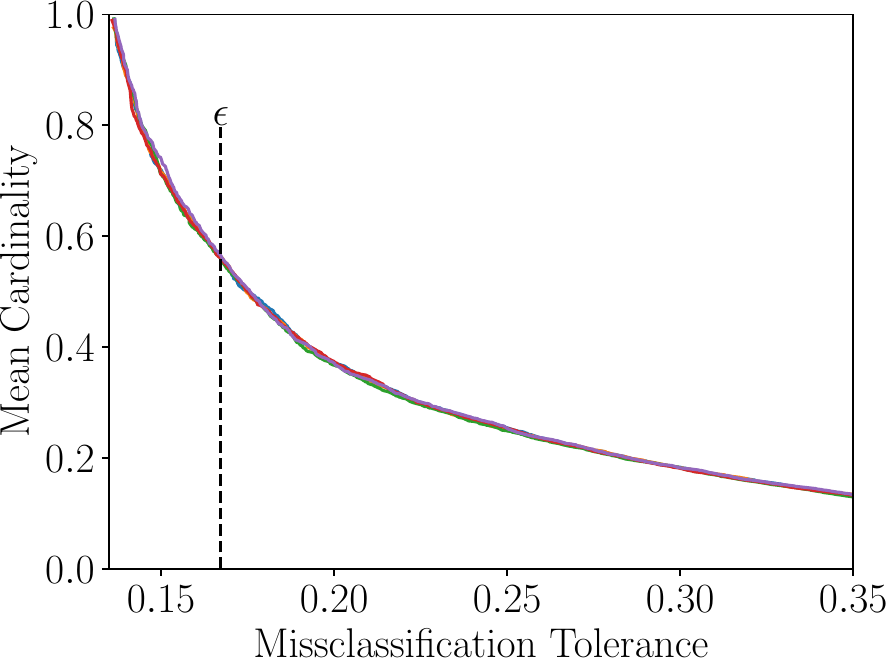}
        \caption{}
    \end{subfigure}
    \caption{Estimating the Rashomon Set of RFs on Adult-Income. 
    Each curve is associated with a different tree collection $\mathcal{T}_i$.}
    \label{fig:adult_experiment}
\end{figure*}

\subsubsection{Local Feature Attribution}

\begin{figure*}[t]
    \centering
    \begin{subfigure}[c]{0.46\textwidth}
    \includegraphics[width=\linewidth]
        {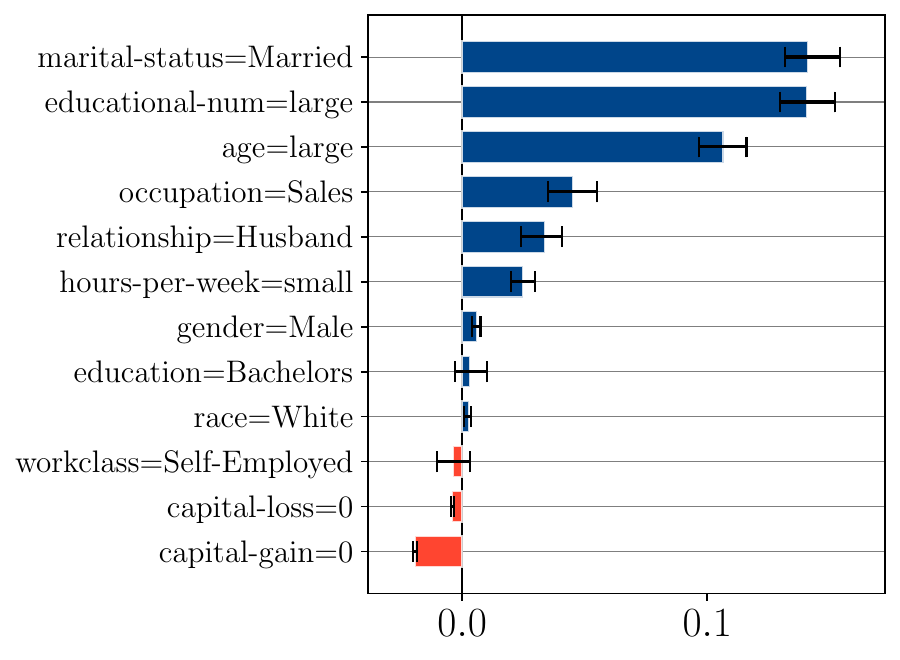}
    \end{subfigure}
    \hfill
    \begin{subfigure}[c]{0.53\textwidth}
    \hspace{1cm}
    \resizebox{0.65\textwidth}{!}{
        \includegraphics[width=\linewidth]
        {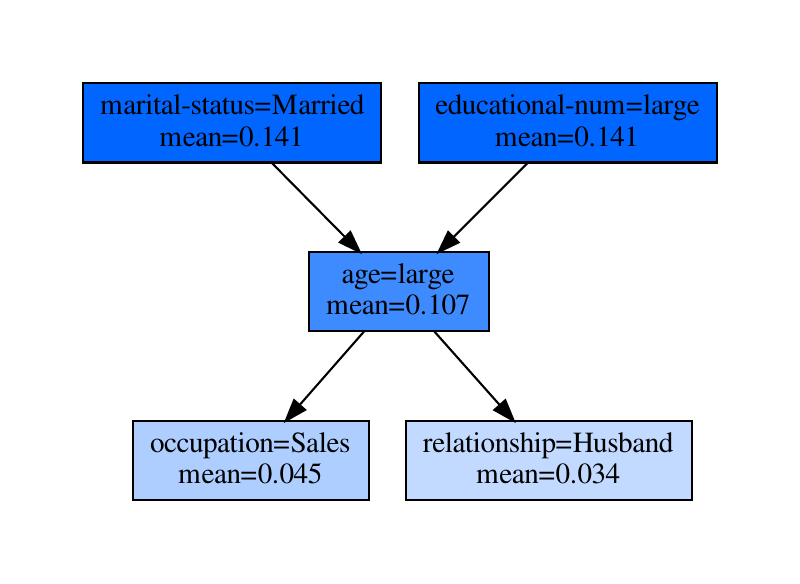}
    }
    \end{subfigure}
    \begin{subfigure}[c]{0.46\textwidth}
        \includegraphics[width=\linewidth]
        {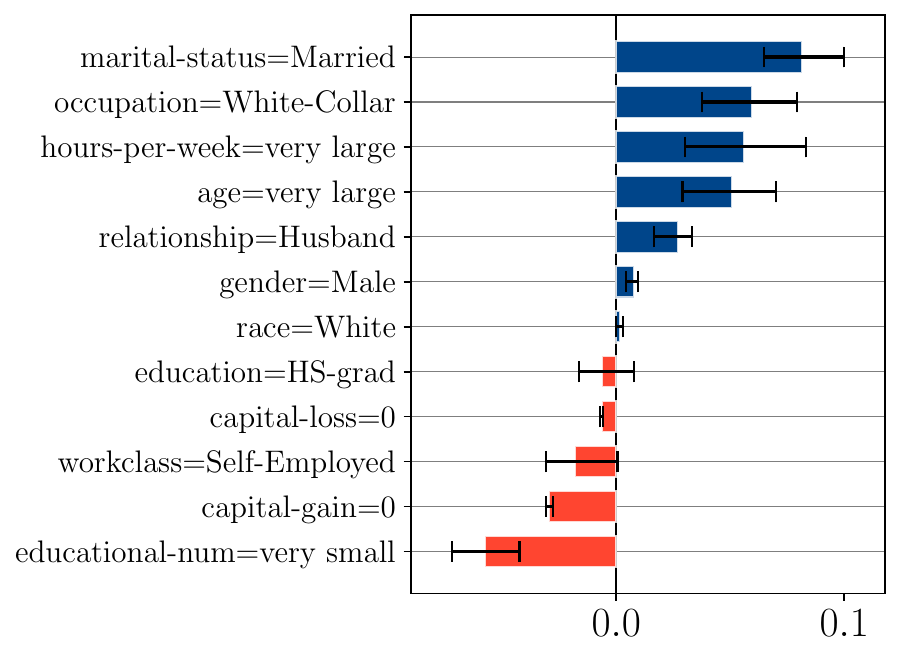}
    \end{subfigure}
    \hfill
    \begin{subfigure}[c]{0.53\textwidth}
        \includegraphics[width=\linewidth]
        {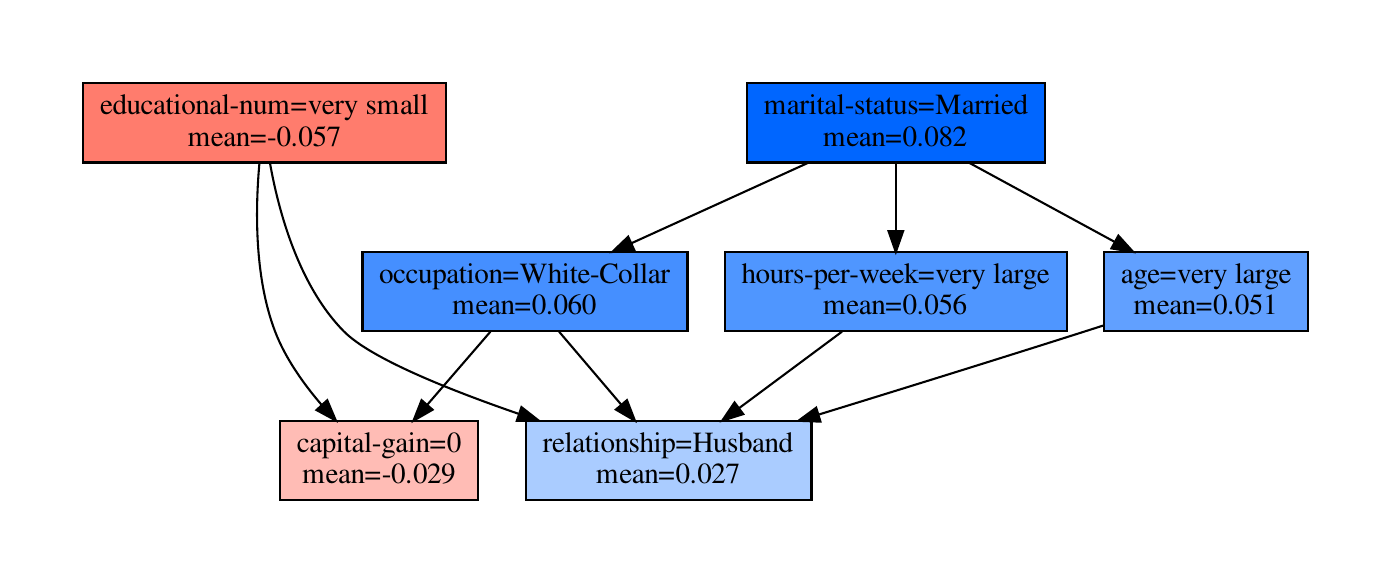}
    \end{subfigure}
    \caption{
    Local feature attributions on two individuals 
    (Top) A person with a high prediction, (Bottom) Individual near the decision 
    boundary. The Hasse Diagrams only show the first three ranks. 
    }
    \label{fig:adult_explain_instances}
\end{figure*}

The model outputs $h(\myvec{x})\in[0,1]$ must be interpreted as estimates of the conditional probabilities of 
$y$ given $\bm{x}$ and not as hard $0/1$ predictions. Therefore, local feature attributions should sum up 
to a difference in conditional probabilities. We computed local feature attributions with the efficient algorithm 
TreeSHAP. In fact, seeing that categorical features were one-hot-encoded, 
which is not supported in the TreeSHAP implementation of the \texttt{SHAP} library, we used the Partition-TreeSHAP 
algorithm described in \citep{laberge2022understanding}. The feature attribution requires a background distribution 
$\mathcal{B}$ to serve as a reference and we used the empirical distribution of the whole training set.
Still, given the considerable size of the Adult dataset, we had to subsample $B$ instances from the training set and 
use them to estimate Shapley values. So, we ended up explaining the models with estimates $\widehat{\bm{\phi}}$ rather 
than ground-truths $\bm{\phi}$. A proxy of the error made by subsampling is the Gap Error presented 
in Equation \ref{eq:gap_error}. We found that the Gap Errors would stabilize to around $0.2\%$ at $B=500$ and so
we employed $500$ background samples. This led to a ten-minute runtime for explaining $M=1000$ decision trees on $2000$ test instances. 

Figure \ref{fig:adult_experiment}~(b) presents the mean partial-order cardinality as a function
of tolerance on test error. We observe that the five curves are very similar which suggests that fitting $M=1000$ 
trees can be representative of all trees possibly generated for RFs. For error tolerances smaller than the 
$\epsilon$ employed, the mean cardinality decreases very rapidly. This means that our partial orders abstain
from making many statements supported by $h_\text{ref}=1/M\sum_{s=1}^M t_s$, but which are contradicted by 
other RFs with slightly worst test performance. We now discuss two instances that were explained with our framework.

The first instance is an individual who makes more than 50k per year and whose predictions
range from $0.69$ to $0.74$ across $\mathcal{H}_{815:}$. The average prediction on the background for all trees 
is $0.23$ so this individual has a positive gap, which we aim to explain with TreeSHAP. 
Figure \ref{fig:adult_explain_instances}~(Top) illustrates this person's local feature attribution and the 
resulting partial order that encodes the statements on which there is a consensus in $\Hypo_{815:}$.
We observe that the features \texttt{educational-num=large} and \texttt{matiral-status=Married} have 
maximal positive importance for understanding why this individual has higher-than-average predictions. 
At the second rank is the feature \texttt{age=large}, which is also important but to a lesser extent.
Looking at the bar char on the top left, we note that the feature \texttt{gender=Male} is given a small
yet consistently positive attribution across all models. It appears that all RFs with at least 815
trees exhibit a small gender bias. We will come back to this in our analysis of global feature importance.

% These two rankings coincide with the
% ranking we would get by averaging all 1000 trees as indicated by the 
% bar chart. 
%Still, the fact that these rankings remain intact even when a set of competing models rather than a single one gives us more confidence in their robustness. The three next features \texttt{age},\texttt{occupation},\texttt{hours-per-week} all appear at the same vertical position and are all mutually incomparable. This is a type of information that cannot be conveyed when simply averaging all models or averaging the ranks of their feature importance, as the work of \citep{shaikhina2021effects,schulz2021uncertainty} would recommend.

The second instance is a person who makes more than 50k and whose predictions range from 0.30 to 0.50. 
The prediction gap is still positive in that case but it is smaller than the previous example. 
Figure \ref{fig:adult_explain_instances}~(Bottom) shows how our framework would explain the positive gaps. 
We focus on the two features \texttt{capital-gain=0} and \texttt{workclass=Self-Employed} which both have a 
negative attribution according to the average model. Looking at the error bars on the bar chart, we observe 
that the model uncertainty is higher for \texttt{workclass} than with 
\texttt{capital-gain}. This means that there is more agreement among RFs that \texttt{capital-gain=0} reduced the model 
output. For \texttt{workclass=Self-Employed}, the model uncertainty is so high that the min-max interval crosses the 
origin, which implies the existence of RFs with satisfactory performance that yield a positive attribution to this 
feature. Our framework identified this ambiguity and hence removed the feature \texttt{workclass} from the 
partial order despite it having a negative attribution according to the average model.

\subsubsection{Global Feature Importance}

\begin{wrapfigure}{r}{0.43\textwidth}
    \centering
    \vspace{-1cm}
    \includegraphics[width=0.43\textwidth]
    {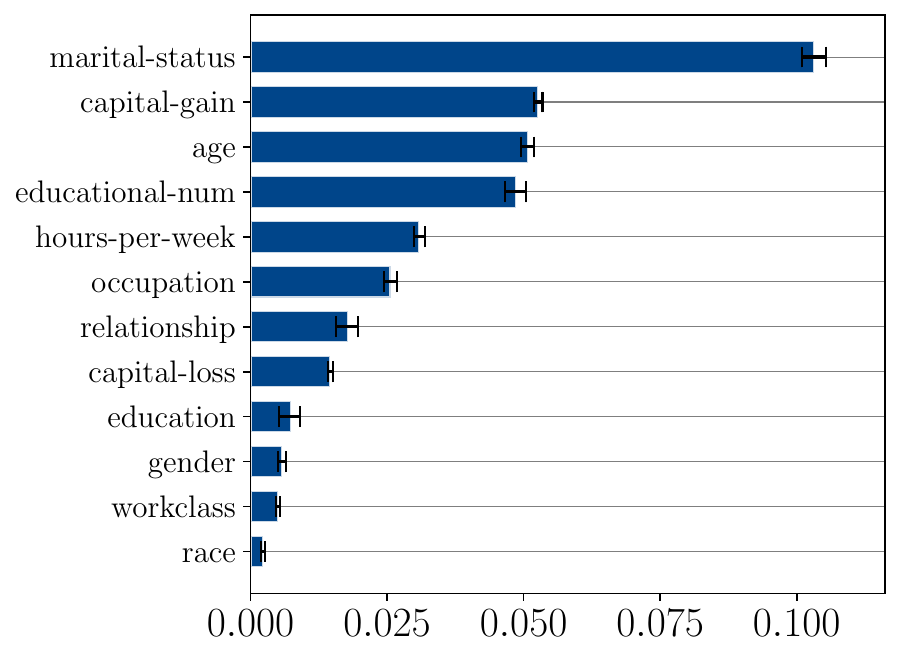}
    \caption{Global Feature Importance on Adult-Income.}
    \label{fig:adult_explain_global}
\end{wrapfigure}

Figure \ref{fig:adult_explain_global} presents the global feature importance. The associated Hasse diagram
is not shown because the feature ordering is a total order. Indeed, the rankings are consistent across 
all RFs with at least $815$ trees. Interestingly, there were more disagreements when looking at local 
feature attributions. This highlights that combining local attributions $\bm{\phi}$ into global ones 
$\bm{\Phi}$ can result in information loss. Hence, it is primordial to investigate explanations 
under-specification both globally and locally.\newline

Notice that all features have non-null importance across the Rashomon Set. This was not true
for the hypothesis class of Additive Models, see Figure \ref{fig:kaggle_explain_global}. 
We suspect that this is due to the training procedure of RFs.
Indeed, when growing trees, a random subset of candidate features is chosen at each internal node. The
optimal split is then chosen among these features. Hence, even if a feature is irrelevant for predicting $y$, 
there is a non-zero probability it will be used by some of the trees in the forest. 
This is unfortunate in the context of biases because any of our good RFs uses the 
\texttt{gender} features for prediction.
\newpage
\section{Discussion}\label{sec:discussion}

As suggested by our experiments, model under-specification has an important impact on feature 
attributions on real data, and taking into account this uncertainty seems necessary to derive reliable
insight from machine learning models. Our conservative approach only retains the information on features 
attributions on which all models agree and still succeeds in finding partial order in this chaos. This in itself 
is an important observation because one could have expected the
partial orders to be trivial and contain no interesting structure (no arrows).

The principal limitation of our approach is that we are currently restricted to Additive Regression, 
Kernel Ridge Regression, and Random Forests. It is therefore primordial to extend our work to other 
models, especially to more Classifiers. We envision using techniques from previous work to sample 
Logistic Regression models and Decision Trees \citep{dong2019variable,kissel2021forward,semenova2022existence}. 
Once an ensemble of models is available, we could apply Model Set Selection 
to choose $\epsilon$ (cf. \textbf{Section \ref{sec:method:epsilon:capture}}) and assert consensus of 
the selected models.

Still, there may also exist hypothesis spaces whose Rashomon Set is too large to be realistically estimated, for instance, Neural Networks. Moreover, the cost of training/explaining 
multiple models may be too high for practitioners to see any benefit. A potential solution to derive careful conclusions from these large models would be to employ only a few models, 
but train them in a way that ensures they are as diverse as possible. This application of our framework is left for future work as it involves unique and novel challenges regarding 
the training of Neural Networks.

The main characteristic of our approach is that we require a perfect consensus among all good models. However, when employing our methodology with a finite ensemble of models, 
one may wonder why not also consider statements on which a majority of models agree (or at least $90\%$ of the models agree). As a more extreme example, a practitioner may have 
1000 models and 999 of these models state something while a single one states the opposite. Our approach would abstain from making any statement in that case, which may seem unnecessarily 
strict. An important argument for requiring a perfect consensus is that it ensures the transitivity of the order relations. This property is crucial for the interpretability of the feature 
orderings. We note that some prior work has produced partial orders from the consensus of at least $\alpha\%$ of the models via the transitive closure and fine-tuning of $\alpha$ to avoid 
cycles \citep{cheng2010predicting}. Nonetheless, in our context of local explainability, this method has two issues. First, it would require fine-tuning $\alpha$ for each instance 
$\bm{x}^{(i)}$ and therefore the interpretation of order relations would change on an instance-by-instance basis. Second, because they rely on transitive closure, the resulting Hasse 
diagrams could be misinterpreted seeing as the existence of a directed path between two features would not imply a consensus among at least $\alpha\%$ of the models 
that one feature is more important. Our diagrams, on the other hand, remain simple to interpret: for any instance $\bm{x}^{(i)}$, a directed path between two features means that all models agree 
on the relative importance statement and the absence of such path means that at least one model disagrees on that statement. Still, we think that imperfect consensus is a pertinent 
future work direction, especially for extending our framework to Bayesian methods.

On a more philosophical level, a justification for perfect consensus is that, given that the error threshold $\epsilon$ was fixed at a value that represents a satisfactory performance, 
any single model that disagrees with the rest is still a good model, and its mere existence is enough to put into question the claim supported by the others. Going back to the 
extreme scenario of 999 models disagreeing with a single one, if this solitary model had the worst performance of the whole ensemble, slightly reducing the error tolerance would 
remove this model from the Rashomon Set and we would reach a consensus.

Speaking of tuning the error tolerance $\epsilon$, similar to prior work \citep{fisher2019all,marx2020predictive,hsu2022rashomon}, we explore a range of tolerance 
values and inspect the effect of under-specification on conclusions drawn from models. Nonetheless, it is not clear what is the right value for $\epsilon$,
however, we argue that this is a limitation shared by multiple studies on the Rashomon Set \citep{d2020underspecification,dong2019variable,semenova2022existence,coker2021theory}.
It is well understood that the $\epsilon$ parameter should be ``small enough'' to represent negligible performance differences. But, there is still no agreement on what ``small enough'' 
means depending on the ML task and hypothesis space. We think the most promising directions in tackling this limitation are Proposition 7 from \citet{fisher2019all}, Profile Likelihoods 
\citep[Appendix C.1]{coker2021theory}, Model Set Selection \citep{kissel2021forward}, and our \textbf{Propositions \ref{prop:capture_bound_linear} \& \ref{prop:capture_bound_classif}}. 
All these statistical guarantees suggest to define the ``set of all good models'' as a set that contains the best-in-class $h^\star$ with high probability. Future work should investigate 
these theoretical results jointly.

\section{Conclusion}\label{sec:conclusion}

In this work, we propose a new approach to explanations in the context of model uncertainty. Rather than considering the mean attributions or the mean rank, we identify
properties and relations of feature attributions that are consistent across a set of models with good performance. These logical statements about local/global feature 
attribution naturally lead to a partial order of feature importance, which we show can provide more nuanced explanations than the more common total orders based on mean attributions. 
As such, we believe that our work opens a new perspective on post-hoc explanations in the context of model uncertainty.

In future work, we intend to study more Classifiers (Logistic Regression, Decision Trees, Neural Networks) and other local/global post-hoc explanations 
(LIME, Permutation Importance, SAGE). Moreover, we shall apply our methodology to more practical settings, especially those where there are clear \textit{actionable} 
features on which a human subject is able to act upon. We hope that in these scenarios, the nuance introduced by partial orders will prove most beneficial.

\section*{Acknowledgements}

This work is supported by the DEEL Project CRDPJ 537462-18 funded by the National Science and Engineering Research Council of Canada (NSERC) and the Consortium for Research 
and Innovation in Aerospace in Québec (CRIAQ), together with its industrial partners Thales Canada inc, Bell Textron Canada Limited, CAE inc and Bombardier inc.\footnote{\url{https://deel.quebec}}

\newpage
\appendix
\section{Proofs}\label{app:proofs}

\subsection{Statistical Bounds}\label{app:proofs:statistical}

\begin{proposition}[\textbf{Proposition \ref{prop:capture_bound_linear}}]
    Under the assumption that the data was generated by the optimal model 
    $h^\star$ plus zero-mean Gaussian noise
    \begin{equation}
        y = h^\star(\myvec{x}) + \Delta,\quad\text{where }\,\,\,
        \Delta\sim\mathcal{N}(0, \sigma^2),
        \label{eq:hypo_model_noise}
    \end{equation}
    and using the squared loss $\ell(y',y)=(y' - y)^2$, we have that
    \begin{equation}
        \probdata[\emploss{S}{h^\star}> \epsilon_{\text{max}}] = 
        1-F_{\chi^2_N}\bigg(\frac{N}{\sigma^2}\epsilon_{\text{max}}\bigg),
    \end{equation}
    where $F_{\chi^2_N}$ is the CDF of a chi-2 random variable
    with $N$ degrees of freedom.
\end{proposition}
\begin{proof}
Under the assumption that Equation \ref{eq:hypo_model_noise} is valid, we have that 
$$\emploss{S}{h^\star}=\frac{1}{N}\sum_{i=1}^N(h^\star(\myvec{x}) - y^{(i)})^2=\frac{1}{N}\sum_{i=1}^N(\Delta^{(i)})^2,$$
where each $\Delta^{(i)}$ is sampled iid from a $\mathcal{N}(0, \sigma^2)$ Gaussian. Now we have
\begin{equation}
    \begin{aligned}
        \probdata[\emploss{S}{h^\star}>\epsilon_{\text{max}}]&=
        \Prob_{\bm{\Delta}\sim \mathcal{N}(0, \sigma^2)^N}\bigg[\frac{1}{N}\sum_{i=1}^N(\Delta^{(i)})^2>\epsilon_{\text{max}}\bigg]\\
        &=\Prob_{\bm{\Delta}\sim \mathcal{N}(0, \sigma^2)^N}\bigg[\sum_{i=1}^N\bigg(\frac{\Delta^{(i)}}{\sigma}\bigg)^2>\frac{N}{\sigma^2}\epsilon_{\text{max}}\bigg]\\
        &=\Prob_{\bm{\Delta}\sim \mathcal{N}(0, 1)^N}\bigg[\sum_{i=1}^N(\Delta^{(i)})^2>\frac{N}{\sigma^2}\epsilon_{\text{max}}\bigg]\\
        &= \Prob_{c \sim \chi^2_N}\bigg[c>\frac{N}{\sigma^2}\epsilon_{\text{max}}\bigg]\\
        &= 1-F_{\chi^2_N}\bigg(\frac{N}{\sigma^2}\epsilon_{\text{max}}\bigg).
    \end{aligned}
\end{equation}
\end{proof}

\newpage
\begin{proposition}[\textbf{Proposition \ref{prop:capture_bound_classif}}]
Let  $\ell(\widehat{y}, y)=\mathbbm{1}(\widehat{y} \neq y)$ be the $0\minus1$ loss, $S\sim \mathcal{D}^N$ 
    be a dataset, $h_\text{ref}\in\Hypo$ be a reference model that is independent of $S$, and $h^\star$ be a 
    best in-class hypothesis, for any $\epsilon'\in \R^+$, we have
    \begin{equation}
        \probdata[\emploss{S}{h^\star}\geq \epsilon'+\emploss{S}{h_\text{ref}}] \leq \exp\bigg\{ -\frac{N \epsilon'^2}{2}\bigg\}.
    \end{equation}
\end{proposition}

\begin{proof}
We assume that $\emploss{S}{h^\star}\geq \epsilon'+\emploss{S}{h_\text{ref}}$ and show that this implies the occurrence of an unlikely event. We first have
\begin{align}
    \emploss{S}{h^\star} - \emploss{S}{h_\text{ref}} &= \frac{1}{N}\sum_{i=1}^N \mathbbm{1}[h^\star(\myvec{x}^{(i)})\neq y^{(i)}] - \mathbbm{1}[h_\text{ref}(\myvec{x}^{(i)})\neq y^{(i)}]\\
    &= \frac{1}{N}\sum_{i=1}^N \Delta^{(i)},
\end{align}
where the $N$ random variables $\Delta^{(i)}:=\mathbbm{1}[h^\star(\myvec{x}^{(i)})\neq y^{(i)}] - \mathbbm{1}[h_\text{ref}(\myvec{x}^{(i)})\neq y^{(i)}]$ are iid, take values between $\minus1$ and $1$, and have the expectancy 
\begin{equation}
    \mu =\E_{S\sim \mathcal{D}^N}[\Delta^{(i)}] = \E_{(\myvec{x}^{(i)}, y^{(i)})\sim \mathcal{D}}[\Delta^{(i)}] = \poploss{h^\star} - \poploss{h_\text{ref}}.
    \label{eq:mu_delta}
\end{equation}
We accentuate that Equation \ref{eq:mu_delta} only holds if the reference 
model $h_\text{ref}$ is independent on the dataset $S$ used to assess model performance.
Now by definition of $h^\star$, we have $\mu \leq 0$. However, under our assumption that $\emploss{S}{h^\star}\geq \epsilon'+\emploss{S}{h_\text{ref}}$, 
we have that $\frac{1}{N}\sum_{i=1}^N \Delta^{(i)}\geq \epsilon'\geq 0$. Hence we have a bounded random variable $\Delta$ whose true mean is negative but 
whose empirical mean is large and positive. This event becomes highly improbable as $\epsilon'$ increases or the sample size $N$ increases, see the following Figure.

\begin{figure}[h]
    \centering
    \begin{tikzpicture}
\tikzset{>=latex}
\draw (-4,0) -- (4,0);
\draw (0,-0.2) node[anchor=north] {0} -- (0,0.2);
\draw (-4,-0.2) node[anchor=north] {-1} -- (-4,0.2);
\draw (4,-0.2) node[anchor=north] {1} -- (4,0.2);

% Left side
\draw[color=blue!50!white,line width=0.15cm] (-4,0) -- (0, 0);
\node at (-2,0.5) {$\lightblue{\mu}$};

% Right side
\draw[->,line width=0.05cm] (0, 0) -- node[anchor=south] {$\epsilon'$} (1.75, 0);
\draw[fill=black] (3,0) node[anchor=south,scale=0.85] {$\frac{1}{N}\sum_{i=1}^N \Delta^{(i)}$} ellipse (0.05 and 0.05);
\end{tikzpicture}
\end{figure}

Formally, using Hoeffding's inequality yields
\begin{align}
    \probdata[\emploss{S}{h^\star}\geq \epsilon' + \emploss{S}{h_\text{ref}}]&= \probdata\bigg[\frac{1}{N}\sum_{i=1}^N \Delta^{(i)}\geq \epsilon'\bigg]\\
    &\leq\probdata\bigg[\frac{1}{N}\sum_{i=1}^N \Delta^{(i)} - \mu\geq \epsilon'\bigg]\tag{Since $\mu \leq 0$}\\
    &\leq \exp\bigg\{ -\frac{N\epsilon'^2}{2}\bigg\}
    \tag{With Hoeffding's inequality},
\end{align}
concluding the proof.

\end{proof}

\newpage
\subsection{Relation to Prior Work}\label{app:proofs:prior_work}

\begin{proposition}[\textbf{Proposition \ref{prop:relation_to_prior}}]
    Let $\bm{\phi}(\cdot, \bm{x})$ be a linear feature attribution functional, and $E=\{h_k\}_{k=1}^M$ be an ensemble of $M$ models from $\Hypo$ trained with the same stochastic learning algorithm $h_k\sim \mathcal{A}(S)$. Said feature attribution and ensemble will be employed in the methods of \citep{shaikhina2021effects, schulz2021uncertainty}. Moreover let $\epsilon \geq \max \{\emploss{S}{h_k}\}_{k=1}^M$ be an error tolerance, and let $\preceq_{\epsilon, \myvec{x}}$ be the consensus order relation on $\text{SA}(\epsilon, \myvec{x})$ (cf. Equation \ref{eq:po}). If the relation $i\preceq_{\epsilon, \myvec{x}}j$ holds, 
    we have that $i$ is less important than $j$ in the two total orders of prior work
\citep{shaikhina2021effects, schulz2021uncertainty}.
\end{proposition}

\begin{proof}
We first note that, since $i,j\in \text{SA}(\epsilon, \myvec{x})$, there is a consensus across the Rashomon Set that these features attributions have sign $s_i$ and $s_j$ respectively. As a reminder, this simplifies the expression of the feature importance :  
$\forall\,h \in \Rashomon \,\,\,\,\,|\phi_i(h, \myvec{x})| = s_i\phi_i(h, \myvec{x})$.
Additionally, our assumption that
$\epsilon \geq \max \{\emploss{S}{h_k}\}_{k=1}^M$,
guarantees that $E\subseteq \Rashomon$.
We now prove that the order relation 
$i \preceq_{\epsilon, \myvec{x}} j$ is present in the two rankings from the literature.

\paragraph{\citet{shaikhina2021effects}}
compute the average model $h_E=\frac{1}{M}\sum_{k=1}^M h_k$ and rank features according to their importance for this model $|\bm{\phi}(h_E, \bm{x})|$.
For any $i,j\in \text{SA}(\epsilon,\bm{x})$, we deduce
\begin{align*}
    i \preceq_{\epsilon, \myvec{x}} j &\Rightarrow
    \forall\,h \in \Rashomon \quad|\phi_i(h, \myvec{x})| \leq |\phi_j(h, \myvec{x})|\\
    &\Rightarrow
    \forall\,h \in \Rashomon \,\,\,\,\,s_i\phi_i(h, \myvec{x}) \leq s_j\phi_j(h, \myvec{x})\\
    &\Rightarrow
    \forall\,h \in E \,\,\quad\qquad s_i\phi_i(h, \myvec{x}) \leq s_j\phi_j(h, \myvec{x})\\
    &\Rightarrow
    \frac{1}{M}\sum_{k=1}^M s_i\phi_i(h_k, \myvec{x}) \leq \frac{1}{M}\sum_{k=1}^M s_j\phi_j(h_k, \myvec{x})\\
    &\Rightarrow
    s_i\phi_i(h_E, \myvec{x}) \leq s_j\phi_j(h_E, \myvec{x})\tag{By Linearity of $\bm{\phi}$}\\
    &\Rightarrow
    |\phi_i(h_E, \myvec{x})| \leq |\phi_j(h_E, \myvec{x})|\tag{By Linearity of $\bm{\phi}$, $s_i=\texttt{sign}[\,\phi_i(h_E,\bm{x})\,]$},
\end{align*}
thus proving that the order relation is also present when explaining the average model.

\paragraph{\citet{schulz2021uncertainty}} 
order features using the mean rank $\frac{1}{M}\sum_{k=1}^M \mathbf{r}[\,|\bm{\phi}(h_k, \myvec{x})|\,]$,
where $\textbf{r}:\R^d_+ \rightarrow [d]$ is the rank function. By the definition, for any model $h$, we have
$|\phi_i(h,\bm{x})|\leq |\phi_j(h,\bm{x})|\iff r_i[\,|\bm{\phi}(h, \myvec{x})|\,] \leq r_j [\,|\bm{\phi}(h, \myvec{x})|\,]$. Therefore,
\begin{align*}
    i \preceq_{\epsilon, \myvec{x}} j &\Rightarrow
    \forall\,h \in \Rashomon \,\,\,\,\,|\phi_i(h, \myvec{x})| \leq |\phi_j(h, \myvec{x})|\\
    &\Rightarrow
    \forall\,h \in E \,\,\quad\qquad|\phi_i(h, \myvec{x})| \leq |\phi_j(h, \myvec{x})|\\
    &\Rightarrow
    \forall\,h \in E \quad\,\,\,\,\, 
    r_i[\,|\bm{\phi}(h, \myvec{x})|\,] \leq 
    r_j[\,|\bm{\phi}(h, \myvec{x})|\,]\\
    &\Rightarrow
    \frac{1}{M}\sum_{k=1}^M r_i[\,|\bm{\phi}(h_k, \myvec{x})|\,] \leq 
    \frac{1}{M}\sum_{k=1}^M r_j[\,|\bm{\phi}(h_k, \myvec{x})|\,],
\end{align*}
which implies that the order relation is also supported by the mean ranks.
\end{proof}

\newpage
\subsection{Random Forests}\label{app:proofs:random_forest}

\begin{proposition}[\textbf{Proposition \ref{prop:rf_optim}}]
Let $\mathcal{T}:= \{\tree_s\}_{s=1}^M$ be a set of $M$ trees, $\Hypo_{m:}$ be the set of all
subsets of at least $m$ trees from $\mathcal{T}$, and
$\phi:\Hypo_{m:}\rightarrow \R$ be a linear functional, then $\min_{h\in \Hypo_{m:}} \phi(h)$
amounts to averaging the $m$ smallest values of 
$\phi(\tree_s)$ for $s=1,2,\ldots M$.
\end{proposition}
\begin{proof}
    We can compute the linear functional
    on every tree $\{\phi(\tree_s)\}_{s=1}^M$ and store the
    indices of the $m$ smallest ones in a set $C_m$ s.t.
    $|C_m|=m$ and 
    \begin{equation}
        s\in C_m\,\, \text{and}\,\, s'\notin C_m \Rightarrow \phi(\tree_s) \leq \phi(\tree_{s'}).
        \label{eq:partition}
    \end{equation}
    Now, to prove to proposition, we must show that
    $\phi(\frac{1}{m}\sum_{s\in C_m} \tree_s) \leq \phi(h)\,\,\forall h\in \Hypo_{m:}$. Since
    $
    \min_{h\in \Hypo_{m:}} \phi(h) = \min_{k=m,\ldots,M}\min_{h\in \Hypo_k} \phi(h),
    $
    the proof can be done in two parts:
    first for a fixed $k$ we prove that 
    $\phi(\frac{1}{k}\sum_{s\in C_k} \tree_s) \leq \phi(h)\,\,\forall h\in \Hypo_k$ and secondly prove that 
    $\argmin_{k=m,\ldots,M} \phi(\frac{1}{k}\sum_{s\in C_k} \tree_s) =m$.
    \paragraph{Part 1}
    By linearity
    $\phi(\frac{1}{k}\sum_{s\in C_k} \tree_r) =
    \frac{1}{k}\sum_{s\in C_k} \phi(\tree_r)$. Also, remember that any model $h\in\mathcal{H}_k$ is associated to a subset $C'_k$ of $k$ seeds \textit{i.e.} $h = \frac{1}{k}\sum_{s\in C'_k}\tree_r$. Importantly, since $C_k$ and $C'_k$ have the same size, 
    the two sets $C_k\setminus C'_k$ and $C'_k\setminus C_k$
    have a one-to-one correspondence. We get
    \begin{align*}
        \frac{1}{k}\sum_{s\in C_k} \phi(\tree_s)&=
        \frac{1}{k}\bigg(\sum_{s\in C_k\cap C'_k}\phi(\tree_s)+
        \sum_{s\in C_k\setminus C'_k}\phi(\tree_s)\bigg)\\
        &\leq \frac{1}{k}\bigg(\sum_{s\in C_k\cap C'_k}\phi(\tree_s)+
        \sum_{s'\in C'_k\setminus C_k}\phi(\tree_{s'})\bigg) \tag{cf. Equation \ref{eq:partition}}\\
        &= \frac{1}{k}\sum_{s\in C'_k}\phi(\tree_s)=
        \phi\bigg( \frac{1}{k}\sum_{s\in C'_k} \tree_s\bigg)=\phi(h).
    \end{align*}
    
    \paragraph{Part 2}
    We now prove that $\argmin_{k=m,\ldots,M} \phi(\frac{1}{k}\sum_{s\in C_k} \tree_s) =m$. The key insight is that given
    $m'>m$, the set $C_m$ contains the $m$ smallest
    elements of $C_{m'}$. We get
    \begin{align*}
        \frac{1}{m'}\sum_{s\in C_{m'}} \phi(\tree_s)&=
        \frac{1}{m'}\bigg(\sum_{s\in C_m} \phi(\tree_s)
        +\sum_{s'\in C_{m'}\setminus C_m} \phi(\tree_{s'})\bigg)\\
        &\geq
        \frac{1}{m'}\bigg(\sum_{s\in C_{m}} \phi(\tree_s)
        +\sum_{s'\in C_{m'}\setminus C_m} \bigg[\frac{1}{m}\sum_{s\in C_m} \phi(\tree_s)\bigg]\,\bigg)\\
        &=
        \frac{1}{m'}\bigg(\sum_{s\in C_{m}} \phi(\tree_s)
        +\frac{m'-m}{m}\sum_{s\in C_m} \phi(\tree_s)\bigg)\\
        &=
        \frac{1}{m'}\frac{m'}{m}\sum_{s\in C_m} \phi(\tree_s)
        = \frac{1}{m}\sum_{s\in C_m} \phi(\tree_s),
    \end{align*}
    which ends the proof.
\end{proof}

\section{Optimization}\label{app:optim}
\subsection{Optimization over a Ellipsoid}\label{app:optim:ellipsoid}
\subsubsection{Linear Objective}\label{app:optim:ellipsoid:linear}

We study the optimization of a linear
function over an ellipsoid
\begin{equation}
    \begin{aligned}
        \max_{\bm{\omega}}\quad & \bm{a}^T \bm{\omega} \\
        \textrm{s.t.} \quad & (\bm{\omega}-\leastsq)^T\bm{A}(\bm{\omega}-\leastsq)
        \leq  \epsilon - \emploss{S}{\leastsq},
    \end{aligned}
\label{eq:linear_opt_ellipsoid}
\end{equation}
which is necessary to compute the local feature attribution consensus on the Rashomon
Set of Additive Regression and Kernel Ridge Regression. To lighten the notation, we will introduce the
variable $\epsilon':=\epsilon - \emploss{S}{\leastsq}$. Solving Equation \ref{eq:linear_opt_ellipsoid}
can be done efficiently with a Cholesky decomposition of $\bm{A}=\bm{C} \bm{C}^T$,
which we know exists since $\bm{A}$ is symmetric positive definite. We also have
$\bm{A}^{-1} = (\bm{C}^{-1})^T \bm{C}^{-1}$. Now, it is always possible to map an ellipsoid back to 
a sphere by defining a new variable
\begin{equation}
    \bm{z} :=  \bm{C}^T(\bm{\omega}-\leastsq),
    \label{eq:change_of_variable}
\end{equation}
see Figure \ref{fig:ellipsoid_to_sphere}. Applying the inverse change of
variable to $\bm{\omega}$ in Equation \ref{eq:linear_opt_ellipsoid}, we get

\begin{figure}[t]
    \centering
    \begin{tikzpicture}
    \tikzset{>=latex}
    % Unit circle
    \draw[fill=white!95!black]  (-7.5,0) ellipse (1 and 1);
    \node at (-7.5,-3) {$\bm{z}^T\bm{z}\leq \epsilon'$};
    
    % Ellipsoid
    \draw[rotate around={30:(1.5,1)},fill=white!95!black]  
    (1.5,1) ellipse (1.25 and 0.5);
    \node at (1.5,1) {$\leastsq$};
    \node at (2.5,-3) {$(\bm{\omega} - \leastsq)^T\bm{A}(\bm{\omega} - \leastsq)\leq \epsilon'$};

    % Axes
    \draw[->]  (-10,0) -- (-5, 0) node[anchor=north] {$z_1$};
    \draw[->]  (-7.5,-1.5) -- (-7.5, 2.5) node[anchor=east] {$z_2$};
    
    \draw[->]  (0,0) -- (5, 0) node[anchor=north] {$\omega_1$};
    \draw[->]  (2.5,-1.5) -- (2.5, 2.5) node[anchor=east] {$\omega_2$};
    
    % Transformations
    \draw[<-]  plot[smooth, tension=.7] coordinates {(-4.5,0.5) (-2.5,1) (-0.5,0.5)};
    \draw[->]  plot[smooth, tension=.7] coordinates {(-4.5,-0.5) (-2.5,-1) (-0.5,-0.5)};
    \node[->] at (-2.5,1.5) {$\bm{z} = \bm{C}^T(\bm{\omega} - \leastsq)$};
    \node[->] at (-2.5,-1.5) {$\bm{\omega} = (\bm{C}^{-1})^T\bm{z} +\leastsq$};

\end{tikzpicture}
    \caption{Mapping an ellipsoid to the unit sphere.}
    \label{fig:ellipsoid_to_sphere}
\end{figure}

% \begin{equation}
%     \begin{aligned}
%         \bm{z}^T\bm{z} &= (\bm{\omega}-\leastsq)^T \bm{A}^{\frac{1}{2}}(\epsilon - \emploss{S}{\leastsq})^{-
%     \frac{1}{2}}(\epsilon - \emploss{S}{\leastsq})^{-
%     \frac{1}{2}}\bm{A}^{\frac{1}{2}^T}(\bm{\omega}-\leastsq)\\
%         &=(\bm{\omega}-\leastsq)^T\frac{\bm{A}}{\epsilon - \emploss{S}{\leastsq}}(\bm{\omega}-\leastsq)\\
%         &\leq 1.
%     \end{aligned}
% \end{equation}

\begin{equation}
    \begin{aligned}
        \bm{a}^T \bm{\omega} &= 
        \bm{a}^T\big(\,(\bm{C}^{-1})^T\bm{z} + \leastsq\,\big)\\
        &= \underbrace{\bm{a}^T(\bm{C}^{-1})^T}_{\bm{a}'^T}\bm{z} + 
        \bm{a}^T\leastsq,
    \end{aligned}
\end{equation}
leading to the optimization problem
\begin{equation}
    \begin{aligned}
        \max_{\bm{z}}\quad & \bm{a}'^T \bm{z} + \bm{a}^T\leastsq\\
        \textrm{s.t.} \quad & \bm{z}^T\bm{z}
        \leq \epsilon.
    \end{aligned}
\label{eq:linear_opt_ellipsoid_v2}
\end{equation}
Importantly, the optimization problems of Equations \ref{eq:linear_opt_ellipsoid} and
\ref{eq:linear_opt_ellipsoid_v2} both reach the same optimal values. Since the objective $\bm{a}'^T \bm{z}$ is a 
scalar product, it reaches its maximum objective value $\sqrt{\epsilon'}\|\bm{a}'\|$ when the vector $\bm{z}$ points 
in the same direction as $\bm{a}'$. The minimum and maximum values of the objective are therefore
$\pm\sqrt{\epsilon-\emploss{S}{\leastsq}}\|\bm{a}'\| + \bm{a}^T\leastsq$. 
% The argmax $\bm{z}^\star$ can be mapped back into $\bm{\omega}^\star = (\epsilon - \emploss{S}{\leastsq})^{-\frac{1}{2}}\bm{A}^{-\frac{1}{2}^T}\bm{z}^\star + \leastsq$. 
% Very similar arguments apply if we are minimizing the linear objective instead of maximizing it. In that case,
% the minimum value $-\sqrt{\epsilon - \emploss{S}{\leastsq}}\|\bm{a}'\|+ \bm{a}^T\leastsq$.
% The full procedure is presented in Algorithm \ref{alg:optim_linear}.
% We note that it can be vectorized so that one function call for can
% optimize for multiple values of $\bm{a}$.

% \begin{algorithm}
% \caption{Compute the min/max of a linear function over an ellipsoid.}
% \begin{algorithmic}[1]
% \Procedure{Linear\_on\_Ellipsoid}{$\bm{a},\epsilon, \bm{A}^{-\frac{1}{2}}, \leastsq$}
%     \State $\bm{a}'=\bm{A}^{-\frac{1}{2}}\bm{a}$;
%     \State \Return $\pm \sqrt{\epsilon - \emploss{S}{\leastsq}}\|\bm{a}'\|+\bm{a}^T\leastsq $;
% \EndProcedure
% \end{algorithmic}
% \label{alg:optim_linear}
% \end{algorithm}

\subsubsection{Quadratic Objective}\label{app:optim:ellipsoid:quadratic}

We now investigate the optimization of a quadratic form over an ellipsoid
\begin{equation}
    \begin{aligned}
        \min_{\bm{\omega}}\quad & 
        \bm{\omega}_i^T \bm{B}_i \bm{\omega}_i-\bm{\omega}_j^T \bm{B}_j \bm{\omega}_j\\
        \textrm{s.t.} \quad & 
        (\bm{\omega}-\leastsq)^T\bm{A}(\bm{\omega}-\leastsq)
        \leq \epsilon'.
    \end{aligned}
\end{equation}
Letting $\bm{\omega}_{ij}\in \R^{M_i+M_j}$ be the concatenation of $\bm{\omega}_i$ and $\bm{\omega}_j$, 
and relabelling the least-square $\widehat{\bm{\omega}}:=\leastsq$, we express the optimization problem as
\begin{equation}
    \begin{aligned}
        \min_{\bm{\omega}_{ij}}\quad & 
        \bm{\omega}_{ij}^T \,\bm{B}_{ij}\, \bm{\omega}_{ij}\\
        \textrm{s.t.} \quad & 
        (\bm{\omega}_{ij}-\widehat{\bm{\omega}}_{ij})^T\bm{A}_{ij}(\bm{\omega}_{ij}-\widehat{\bm{\omega}}_{ij})
        \leq \epsilon',
    \end{aligned}
\end{equation}
where $\bm{B}_{ij}$ is a block-diagonal matrix containing $\bm{B}_i$ and $-\bm{B}_j$, 
and $\bm{A}_{ij}$ is the Schur complement of $\bm{A}$. The Schur complement is computed because we must 
project the Rashomon Set (which is an ellipsoid in $\R^{1+\sum_j M_j}$) onto the subspace $\R^{M_i+M_j}$ in which 
$\bm{\omega}_{ij}$ resides. Importantly, the projection of an ellipsoid on a subspace is still an 
ellipsoid whose covariance matrix is the Schur complement. Taking the Cholesky decomposition
$\bm{A}_{ij}=\bm{C}\bm{C}^T$ and using the change of variable in Equation \ref{eq:change_of_variable}, we get
\begin{equation}
    \bm{\omega}_{ij}^T \,\bm{B}_{ij}\, \bm{\omega}_{ij} = 
        (\bm{z}_{ij} - \widehat{\bm{z}}_{ij})^T \bm{B}_{ij}'
        (\bm{z}_{ij} - \widehat{\bm{z}}_{ij}),
\end{equation}
with $\bm{B}_{ij}'= \bm{C}^{-1} \bm{B}_{ij} (\bm{C}^{-1})^T$ and 
$\widehat{\bm{z}}_{ij}:=-\bm{C}^T\widehat{\bm{\omega}}_{ij}$. Thus, we can express the optimization in standard
TRS form
\begin{equation}
    \begin{aligned}
        \min_{\bm{z}_{ij}}\quad & 
         (\bm{z}_{ij} - \widehat{\bm{z}}_{ij})^T \bm{B}_{ij}'
        (\bm{z}_{ij} - \widehat{\bm{z}}_{ij})\\
        \textrm{s.t.} \quad & 
        \bm{z}_{ij}^T\bm{z}_{ij}
        \leq \epsilon'
    \end{aligned}
    \label{eq:TRS_standard}
\end{equation}
and solve the following necessary optimality conditions adapted from 
Corollary 7.2.2 in \cite[Section 7.2]{conn2000trust}.
\begin{corollary}[TRS Necessary Optimality Condition]
    Letting $\{\sigma_k\}_k$ be the eigenvalues of the matrix $\bm{B}_{ij}'$, any global
    minimizer $\bm{z}_{ij}$ of the TRS (Equation \ref{eq:TRS_standard}) must satisfy
    \begin{align}
        \bm{B}_{ij}'(\bm{z}_{ij} - \widehat{\bm{z}}_{ij}) &= \lambda \bm{z}_{ij}\label{eq:KKT_1}\\
        \lambda(\bm{z}_{ij}^T\bm{z}_{ij}-\epsilon')&=0\label{eq:KKT_2},
    \end{align}
    for some $\lambda \geq \max\{0\}\!\cup \!\{\minus\sigma_k\}_k$. 
    If $\lambda > \max\{\minus\sigma_k\}_k$ then $\bm{z}_{ij}$ is the \textbf{unique} global minimizer.
\end{corollary}
To solve these conditions, we diagonalize $\bm{B}_{ij}'=\bm{V}\bm{D}\bm{V}^T$,
define $\bm{\alpha}=\bm{V}^T \bm{z}_{ij}$ and $\widehat{\bm{\alpha}}=\bm{\bm{V}}^T \widehat{\bm{z}}_{ij}$.
Then, assuming $\lambda > \max\{-\sigma_k\}_k$, we rewrite Equation \ref{eq:KKT_1} as
\begin{equation}
    \bm{\alpha} = (\bm{D}+\lambda \bm{I})^{-1}\bm{D}\widehat{\bm{\alpha}}.
    \label{eq:solve_alpha}
\end{equation}
Also assuming $\lambda > 0$, Equation \ref{eq:KKT_2} becomes
$\bm{\alpha}^T\bm{\alpha}=\epsilon'$, which combined with Equation \ref{eq:solve_alpha} yields
\begin{equation}
    q(\lambda):=\sum_k \frac{\sigma_k^2}{(\sigma_k+\lambda)^2}\widehat{\alpha}_k = \epsilon'.
\label{eq:non_linear_lambda}
\end{equation}
We finally solve the non-linear Equation $q(\lambda)=\epsilon'$ for 
$\lambda > \max\{0\}\!\cup \!\{\minus\sigma_k\}_k$ with the bissection algorithm. From the resulting $\lambda$ we can determine the TRS solution $\bm{z}_{ij}$.
\newline

If we do not assume $\lambda > \max\{0\}\!\cup \!\{\minus\sigma_k\}_k$, 
there are two additional cases to consider:
\begin{enumerate}
    \item The solution is inside the ball ($\lambda=0$).
    \item The so-called ``Hard Case'' where $\lambda=\max\{\minus\sigma_k\}_k$ 
    and $(\bm{D}+\lambda \bm{I})$ becomes singular.
\end{enumerate}
For simplicity, we do not address them in this Appendix.
We instead refer to \citep[Section 7.3]{conn2000trust} for discussion on these technicalities.

\subsection{Combinatorial Optimization and Relaxations}\label{app:optim:combinatorial}

\subsubsection{Min/Max of Global Importance}

In this section we discuss the combinatorial optimization problems
that occur when computing the global feature importance over the Rashomon
Set of Random Forests. As a reminder, we have defined
\begin{equation}
    \Hypo_m := \bigg\{\frac{1}{m}\sum_{\tree \in T} \tree \,\,:\,\, 
    T \subseteq \mathcal{T} \,\,\,\,\text{and}\,\,\,\, |T|=m\bigg\},
\end{equation}
as the set of RFs containing $m$ trees. An alternative way to represent such a set is to introduce
binary variables $\bm{z}\in \{0, 1\}^M$ with $\sum_{s=1}^Mz_s =m$ 
and view all RFs from $\Hypo_m$ as $\frac{1}{m}\,\sum_{s=1}^M z_s t_s$ for some $\bm{z}$.

Now letting $\phi_j$ be the SHAP local feature attribution of feature $j$, we wish to find the minimum and 
maximum values of the global feature importance 
$\Phi^{[1]}_j(h):=\frac{1}{N}\sum_{i=1}^N |\,\phi_j(h, \myvec{x}^{(i)})\,|$ across all RFs with $m$ trees
\begin{equation}
    \begin{aligned}
        \minormax_{\bm{z}}\quad & 
        \frac{1}{Nm}\sum_{i=1}^N\bigg| \,\sum_{s=1}^M z_s\,\phi_j(t_s, \bm{x}^{(i)})\,\bigg|\\
        \textrm{s.t.} \quad & 
        \bm{z}\in \{0, 1\}^M \text{ and } \sum_{s=1}^Mz_s =m.
    \end{aligned}
    \label{eq:combinatoric_optim}
\end{equation}
These are non-linear combinatorial problems that are extremely hard to solve.
For that reason, we will provide quick approximate solutions based on a Linear relaxation of 
Equation \ref{eq:combinatoric_optim}. The first step of the relaxation is to enlarge the 
domain of $\bm{z}$ to allow fractional values. 
\begin{equation}
    \begin{aligned}
        \minormax_{\bm{z}}\quad & 
        \sum_{i=1}^N\bigg| \,\sum_{s=1}^M z_s\,\phi_j(t_s, \bm{x}^{(i)})\,\bigg|\\
        \textrm{s.t.} \quad & 
        \bm{z}\in [0, 1]^M \text{ and } \sum_{s=1}^M z_s =m.
    \end{aligned}
    \label{eq:combinatoric_optim_relax}
\end{equation}
The corresponding domain is a polytope and so it is compatible with Linear Programs. 
The second step of the Linear relaxation is to rephrase the absolute value function $|\cdot|$ as a 
Linear Program\newline

\noindent\begin{minipage}{.5\linewidth}
    \begin{equation}
        \begin{aligned}
            |\alpha| = \min_{\beta} \quad & \beta\\
            \textrm{s.t.} 
            \quad & \alpha \leq \beta\\
            \quad & \!\!\!\minus\alpha \leq \beta\\
        \end{aligned}
        \label{eq:dual_1}
    \end{equation}
    \end{minipage}%
\begin{minipage}{.5\linewidth}
    \begin{equation}
        \begin{aligned}
            |\alpha| = \max_{\beta} \quad & \beta \alpha\\
            \textrm{s.t.} 
            \quad & \!\!\!\raisebox{1pt}{\minus} 1 \leq \beta\\
            \quad & \beta \leq 1\\
        \end{aligned}
        \label{eq:dual_2}
    \end{equation}
\end{minipage}
\vspace{0.5cm}

After we get a solution to the relaxation of Equation \ref{eq:combinatoric_optim_relax}, 
we project $\bm{z}$ back on $\{0, 1\}^M$ using the following 
heuristic: if there are $o$ components with $z=1$, we select $M-o$ fractional values in decreasing order 
and set them to one. The other fractional values are set to zero. 
For example, if we have $M=3$ and find a solution $\bm{z} = [0, 1, 1, 0.75, 0.25]$ to the relaxation, 
we would discretize the solution to get $\bm{z} = [0, 1, 1, 1, 0]$. This heuristic may be 
sub-optimal but our goal is to provide quick approximate solutions.

\paragraph{Maximize}
By leveraging Equation \ref{eq:dual_2}, we can reformulate Equation 
\ref{eq:combinatoric_optim_relax} as
\begin{equation}
    \begin{aligned}
        \max_{\bm{z}}\, \sum_{i=1}^N\bigg| \,\sum_{s=1}^M z_s\,\phi_j(t_s, \bm{x}^{(i)})\,\bigg|
        &=\max_{\bm{z}}\, \sum_{i=1}^N\max_{\beta_i\in [-1, 1]} 
        \,\beta_i\sum_{s=1}^M \,z_s\phi_j(t_s, \bm{x}^{(i)})\\
        &=\max_{\bm{z}, \bm{\beta}}\,\sum_{i=1}^N\sum_{s=1}^M z_s\,\beta_i\phi_j(t_s, \bm{x}^{(i)})\\
        &= \max_{\bm{z}, \bm{\beta}} \bm{\beta}^T \bm{B} \bm{z},
    \end{aligned}
    \label{eq:Bilinear}
\end{equation}
where $\bm{z}$ and $\bm{\beta}$ are each restricted to a separate polytope and 
$B_{is} \equiv \phi_j(t_s, \myvec{x}^{(i)})$. Equation \ref{eq:Bilinear} is known as a 
Bilinear Program which is a non-convex optimization problem that 
can be solved to local optima via the coordinate ascent algorithm \citep{nahapetyan2009bilinear}. 
In our setting, the output of the coordinate ascent algorithm will already respect $\bm{z}\in \{0, 1\}^M$ since 
$\max_{\bm{z}} \bm{\beta}^T \bm{B} \bm{z}$ under the constraints on $\bm{z}$ yields $z_s=1$ for the 
$m$ smallest values of $\sum_{i=1}^N \beta_i \phi_j(t_s, \myvec{x}^{(i)})$ and $z_s=0$ for the others.

\paragraph{Minimize}

By leveraging Equation \ref{eq:dual_1}, we can reformulate Equation \ref{eq:combinatoric_optim_relax}
as
\begin{equation}
    \begin{aligned}
        \min_{\bm{z}, \bm{\beta}}& \quad \sum_{i=1}^N\beta_i\\
        \text{s.t.}&\quad  \bm{z}\in [0, 1]^M \text{ and } \sum_{s=1}^M z_s =m\\
        &\quad  \sum_{s=1}^M z_s\,\phi_j(t_s, \bm{x}^{(i)}) \leq \beta_i\\
        &\quad  \!\!\!\minus\sum_{s=1}^M z_s\,\phi_j(t_s, \bm{x}^{(i)}) \leq \beta_i,\\
    \end{aligned}
    \label{eq:Linear}
\end{equation}
which is a Linear Program with $N+M$ variables and $2(N+M)+1$ constraints that we can solve efficiently
if $N$ and $M$ are not too large. However, the solution of this LP can have fractional values so we must 
use the discretization heuristic to get the final solution $\bm{z}\in\{0,1\}^M$. In our experiments on 
Adult-Income, about $0.25\%$ of the non-null components of $\bm{z}$ would be fractional so we suspect our
discretization heuristic provided good solutions in that setting.

Now that we have discussed approximate schemes to get the min/max global feature importance across $\Hypo_m$,
we are left with addressing \emph{relative} importance relations between features.

\subsubsection{Global Relative Importance}

To assert a consensus on global relative importance (cf. \textbf{Definitions 
\ref{def:partial_order_global} \& \ref{def:computing_global_consensus_optim}})
we must solve $\minormax_h \Phi^{[1]}_j(h) - \Phi^{[1]}_k(h)$.
However, as previously discussed, we cannot guarantee to minimize/maximize $\Phi^{[1]}_j(h)$ 
to optimality for Random Forests. Consequently, we cannot guarantee to solve 
$\minormax_h\Phi^{[1]}_j(h) - \Phi^{[1]}_k(h)$ to optimality either. This is a critical 
issue because the resulting partial order may not be transitive.
Our solution is to create an ensemble $E$ containing
\begin{enumerate}
    \item Approximates of $\argminormax_h\Phi_j^{[1]}(h)$ for $1\leq j\leq d$.
    \item Approximates of $\argminormax_h\Phi_j^{[1]}(h)-\Phi_k^{[1]}(h)$ for $1\leq j<k\leq d$.
\end{enumerate}
After, we assert a consensus among all models in $E\subset\Hypo_{m(\epsilon):}$
leading to the partial order
\begin{equation}
 j \,\,\widehat{\preceq_\epsilon}\,\, k \iff 
    \forall\,h \in E\  \,\,\,\Phi_j(h) \leq \Phi_k(h).
\end{equation}
We underestimate the diversity of our models but the resulting partial order 
of global importance is guaranteed to be transitive. To approximate
$\argminormax_h\Phi^{[1]}_j(h) - \Phi^{[1]}_k(h)$, we propose to define the set
\begin{equation}
S_{jk} := \{i\in [N] : \forall h\in \Hypo_m\,\,\texttt{sign}[\phi_j(h, \bm{x}^{(i)})]=s_{ij} \text{ and }
\texttt{sign}[\phi_k(h, \bm{x}^{(i)})]=s_{ik}\}
\end{equation}
representing all data instances whose local attributions for features $j$ and $k$ has a 
consistent sign across the $\Hypo_m$. Then we solve
\begin{equation}
    \begin{aligned}
        &\argminormax_{\bm{z}} \sum_{i\in S_{jk}} 
        \bigg| \,\sum_{s=1}^M z_s\,\phi_j(t_s, \bm{x}^{(i)})\,\bigg|-
        \bigg| \,\sum_{s=1}^M z_s\,\phi_k(t_s, \bm{x}^{(i)})\,\bigg|\\
        &=
        \argminormax_{\bm{z}} \sum_{i\in S_{jk}} 
        s_{ij} \,\sum_{s=1}^M z_s\,\phi_j(t_s, \bm{x}^{(i)})-
        s_{ik}\,\sum_{s=1}^M z_s\,\phi_k(t_s, \bm{x}^{(i)})\\
        &=
        \argminormax_{\bm{z}} \sum_{s=1}^M z_s 
        \bigg(\sum_{i\in S_{jk}} 
        s_{ij} \phi_j(t_s, \bm{x}^{(i)})-
        s_{ik}\,\phi_k(t_s, \bm{x}^{(i)}) \bigg)\\
        &=
        \argminormax_{\bm{z}} \sum_{s=1}^M z_s a_{jks}
    \end{aligned}
\end{equation}
which is a linear function of $\bm{z}$ thus we can leverage \textbf{Proposition \ref{prop:rf_optim}}.

\newpage
\section{Ill-Defined Gaps}\label{app:gap}

In this appendix, we investigate instances $\bm{x}^{(i)}$ whose gap is ill-defined given the underspecification 
of the ML task. That is, there exists two models $h_1,h_2\in \Rashomon$ which assign gaps 
$G(h_1, \myvec{x}^{(i)})<0$ and $G(h_2, \myvec{x}^{(i)})>0$. When this occurs, it does not make sense to compute 
local feature attributions at $\myvec{x}^{(i)}$ since the different models end up answering different 
contrastive questions. We now present instances with an ill-defined gap and show that redefining the 
background $\background$ can help make these points explainable.

\subsection{Kaggle-Houses}\label{app:gap:kaggle}

As a reminder, the background $\background$ employed on Kaggle-Houses was the empirical distribution over the
training data. Figure \ref{fig:kaggle_gaps} shows the distributions of predictions for instance whose gap 
is well-defined or ill-defined across the Rashomon Set. We note that instances whose gap does not have a consistent sign
tend to have predictions $h_S(\bm{x}^{(i)})$ near the baseline $\E_{\bm{z}\sim\background}[h_S(\bm{z})]$ so that the Gap 
$G(h_S, \myvec{x}^{(i)})$ is very small. This could explain why models with similar performance can assign 
different signs to the gap. Importantly, model underspecification warns us that the contrastive question is 
not well-posed on these houses and it would be better to use another background $\background'$ when explaining 
them. We redefined $\background'$ to be the empirical distribution over all houses with a predicted price below 
the first quartile. Consequently, the prediction gaps increased and 97\% of the houses that were previously unexplainable 
suddenly became explainable.

\begin{figure*}[t]
    \centering
    \includegraphics[width=0.5\linewidth]
    {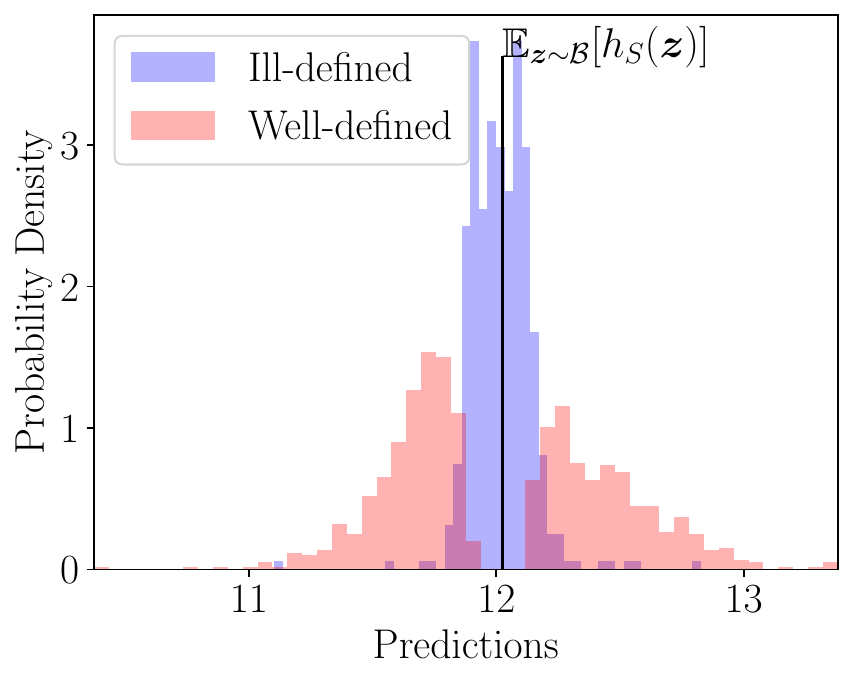}
    \caption{
    Distributions of predictions for houses with ill-defined and 
    well-defined gaps across the Rashomon Set of Kaggle-Houses. 
    The background $\background$ is the empirical distribution over the whole
    training data.}
    \label{fig:kaggle_gaps}
\end{figure*}

\newpage
\subsection{Adult-Income}\label{app:gap:adult}

As a reminder, the background $\background$ employed on Adult-Income was 500 instances sampled uniformly
at random from the training data. Figure \ref{fig:adult_gaps} shows the distributions of predictions for 
individuals whose gap is well-defined or ill-defined across the Rashomon Set. Again, individuals whose 
gap is ill-defined tend to have predictions $h_\text{ref}(\bm{x}^{(i)})$ near 
the baseline $\E_{\bm{z}\sim\background}[h_\text{ref}(\bm{z})]$ so that the Gap 
$G(h_\text{ref}, \myvec{x}^{(i)})$ is small. Once more, we replace the background and
re-explain those inputs. Letting $\background'$ be 500 uniformly-chosen adults who were predicted to make
more than 50K (\ie $h_\text{ref}(\bm{x}^{(i)})>0.5$), the gaps became highly negative 
and 100\% of the previously unexplainable individuals were now explainable.

\begin{figure*}[t]
    \centering
    \includegraphics[width=0.53\linewidth]
    {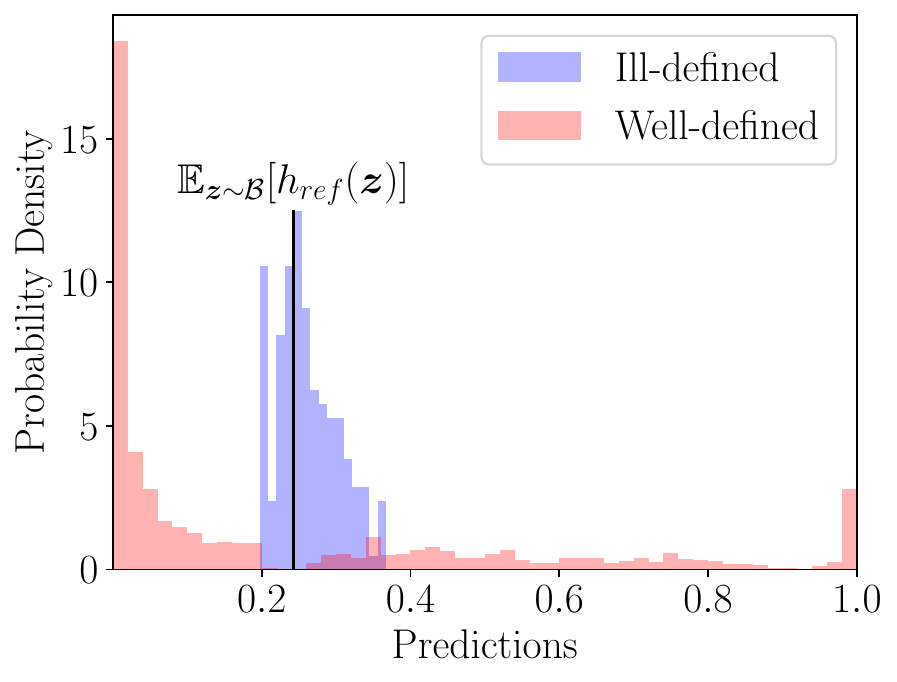}
    \caption{
    Distributions of predictions for instance with ill-defined and 
    well-defined gaps across the Rashomon Set for Adult-Income.
    The background $\background$ is the empirical distribution over 500 uniform samples
    from the training data.
    }
    \label{fig:adult_gaps}
\end{figure*}

\newpage
\vskip 0.2in
\bibliography{biblio}
\newpage
\end{document}